%% file: main.tex
\documentclass{article}

\usepackage{iclr2027_conference,times}
\usepackage{amsmath,amssymb,amsthm}
\usepackage{booktabs}
\usepackage{graphicx}
\usepackage{microtype}
\usepackage{multirow}
\usepackage{enumitem}
\usepackage{longtable}
\usepackage{xcolor}
\usepackage{tikz}
\usetikzlibrary{arrows.meta,positioning,shapes.geometric}
\usepackage{float}
\floatstyle{plain}
\newfloat{algorithm}{tbp}{loa}
\floatname{algorithm}{Algorithm}
\usepackage{hyperref}
\usepackage{url}
\usepackage{xurl}
\hypersetup{colorlinks=true,linkcolor=blue!45!black,citecolor=blue!45!black,
            urlcolor=blue!45!black}

\usepackage{thmtools,thm-restate}
\usepackage{makecell}

\newtheorem{theorem}{Theorem}

\newtheorem{corollary}{Corollary}
\theoremstyle{remark}
\newtheorem{remark}{Remark}
\theoremstyle{plain}

\newcommand{\E}{\mathbb{E}}
\newcommand{\R}{\mathbb{R}}
\newcommand{\tr}{\operatorname{tr}}
\newcommand{\diag}{\operatorname{diag}}
\newcommand{\cv}{\operatorname{CV}}

\newcommand{\Rch}{\mathcal{R}}
\newcommand{\Fr}{F^{\mathrm{R}}}

\newcommand{\Dab}{C_{ab}}

\title{Reducing the Adaptation Gap Through Reachable Fisher Geometry}

\author{
Wasif Jalal \\
Arizona State University \\
Tempe, Arizona, USA \\
\texttt{wjalal@asu.edu}
\And
Sachin Deb \\
Arizona State University \\
Tempe, Arizona, USA \\
\texttt{sachinde@asu.edu}
\And
Asif Salekin \\
Arizona State University \\
Tempe, Arizona, USA \\
\texttt{asif.salekin@asu.edu}
}

\iclrfinalcopy

\begin{document}
\maketitle

\begin{abstract} Parameter-efficient fine-tuning (PEFT) determines not only how many parameters are trained but which local directions a model can move, so similar adapters can affect subgroup losses differently. The relevant curvature matrices are infeasible to form at adapter scale, so practice uses scalar summaries, such as the Fisher trace. What can it reveal, and what does it discard? We answer through the \emph{reachable Fisher}: each subgroup's full-model Fisher pulled back through the adapter Jacobian, under likelihood losses the Gauss--Newton curvature the adapter can reach, whose trace is computable from score-gradient norms without forming the matrix. Under matched subgroup gradients, if the reachable-Fisher difference is positive definite on model-changing directions by a margin exceeding the Hessian--Fisher defect, every sufficiently small nonzero update along them increases the signed gap. The restricted operator norm of the Hessian difference instead determines worst-case quadratic change; a matrix-free Frobenius discrepancy upper-bounds its reachable-Fisher component, with defect and estimation errors kept separate. A trace cannot certify definiteness, nor do matched traces control magnitude: equality rules out a positive-definite difference yet permits the largest operator mismatch that trace allows. LoRA's factorization also precludes unrestricted definiteness in raw coordinates. Across 306 single-seed models, higher trace accompanies greater subgroup difficulty in 75.7\% of 1{,}218 eligible evaluations, and trace matching during training narrows the best--worst subgroup gap in all 30 dataset--encoder--adapter combinations. Yet a held-out audit exposes what it leaves uncontrolled: the observed-label operator discrepancy falls in 23 of these 30, the primary unbiased squared-Frobenius statistic in only 16. The trace is therefore a scalable diagnostic and training heuristic, not a certificate of local gap behavior or matrix alignment. \end{abstract}

\section{Introduction}

An adapter chooses more than how many parameters to train: it chooses the local directions in which the model can move. Two adapters with similar parameter counts and aggregate performance can therefore affect \emph{subgroups}, predefined groups of inputs such as skin-tone, age, or race groups, differently. This effect is inherently geometric. First-order differences depend on how subgroup losses slope along the directions available to the adapter; once those slopes match, their local behavior depends on how the losses curve along the same directions. Recent work accordingly uses curvature-related quantities in subgroup-aware training and auditing \citep{tran2022pruning, wang2023robust, sharma2026curvfed, dai2025fairsam, roy2026ethical}.

We study this geometry through the \emph{reachable Fisher}. Let an adapter be a smooth map $T$ from $p$ trainable parameters $\phi$ to full model parameters $\theta=T(\phi)$, with Jacobian $J=\partial T/\partial\phi$. For subgroup $a$ with full-model Fisher $F_a$, we define $\Fr_a = J^\top F_a J$. This pullback to the trainable coordinates retains the directions the adapter can express and excludes those it cannot reach~\citep{amari1998natural, martens2020new}. The exclusion matters: summaries computed from the full-model Fisher can create or hide an apparent reachable difference (Theorem~\ref{thm:onlyreachable}). Under our likelihood assumptions the Fisher is the positive-semidefinite Gauss--Newton component of the loss curvature; the Hessian in adapter coordinates differs from $\Fr_a$ by a \emph{Hessian--Fisher defect} collecting a model-curvature residual and a reparameterization term in the second derivatives of $T$, which we retain explicitly in every guarantee. The pullback does not remove coordinate dependence, so we compare traces only within a fixed model and parameterization.

At realistic adapter scale, forming the $p\times p$ reachable Fisher is infeasible, whereas its trace is computable from squared adapter-score norms without constructing the matrix. The trace measures total reachable curvature but discards its directional distribution. We therefore distinguish the normalized trace gap from the Frobenius discrepancy $\|\Fr_a-\Fr_b\|_F$, which measures full matrix mismatch. This distinction matters empirically: trace matching reduces the normalized trace gap in 20 of 30 audited cells, yet in 9 of those 20 the primary unbiased estimator of the squared Frobenius discrepancy increases. The scalar moved as intended even when the matrix-level estimate moved oppositely. We therefore ask: \emph{what does reachable-Fisher geometry determine about subgroup-loss changes under small PEFT updates, and what can it not determine?}
The answer separates direction from magnitude: when must every sufficiently small admissible update move a signed subgroup gap in the same direction, and how large can the worst-case local change be? These questions require different spectral information. Definiteness determines whether a quadratic curvature contribution has the same sign in every admissible direction, whereas an operator norm determines its largest absolute magnitude. A trace cannot certify definiteness or generally control the operator norm.


Factorized adapters add a complication. For LoRA \citep{hu2021lora}, $T$ adds a scaled rank-$r$ update $BA$ to the frozen weights, but $(B,A)$ and $(BG,G^{-1}A)$ produce the same update for any invertible $r\times r$ matrix $G$. This symmetry guarantees nonzero directions that leave the model unchanged. More generally, all directions that are inactive to first order form the Jacobian kernel, $\ker J=\{\delta:J\delta=0\}$. Directional statements must therefore be made on its orthogonal complement, the \emph{horizontal subspace} $\mathcal{H}_\phi=(\ker J)^\perp$, where every nonzero direction changes the model to first order. For subgroups $a$ and $b$ with losses $L_a$ and $L_b$, we study the signed surrogate-loss gap $L_a-L_b$. Its local change contains a linear slope term and a quadratic curvature term. When subgroup slopes match on the reachable directions, the linear term vanishes. If $\Fr_a-\Fr_b$ is positive definite on $\mathcal{H}_\phi$ with a margin larger than the Hessian--Fisher defect, then every sufficiently small nonzero horizontal additive update increases the signed gap; when $L_a-L_b\ge0$, it also widens the absolute gap. By contrast, the restricted operator norm of the Hessian difference determines the worst-case quadratic change; the reachable-Fisher operator norm and its Frobenius upper bound control it only up to the defect. Trace equality rules out the sufficient positive-definite regime without generally controlling magnitude.

This geometric question complements methods that optimize worst-group risk or impose outcome constraints \citep{sagawa2019distributionally,hardt2016equality,agarwal2018reductions}. Our theory motivates auditing trace matching at the matrix level rather than treating it as a guarantee of improved subgroup outcomes; empirically, we evaluate the trace as a diagnostic and training heuristic. Our contributions:
\begin{itemize}[leftmargin=*,itemsep=0pt,topsep=0pt]
    \item \textbf{We identify the relevant space for PEFT geometry.} We characterize the LoRA Jacobian kernel exactly, prove that its dimension is at least $r^2$ when $r\le\min(d_{\mathrm{in}},d_{\mathrm{out}})$, extend the obstruction to DoRA \citep{liu2024dora}, and show that definiteness must be evaluated on $\mathcal H_\phi$ (Section~\ref{sec:theory-setup}).
    \item \textbf{We separate direction from magnitude.} We give a defect-aware
    sufficient condition under which every sufficiently small horizontal update
    increases a signed subgroup gap (Theorem~\ref{thm:impossible},
    Corollary~\ref{cor:weyl}) and characterize the worst-case quadratic change
    exactly through an operator norm (Theorem~\ref{thm:worstcase}).
    Appendix~\ref{app:exact} verifies these identities on a small nonlinear
    network and illustrates how a non-negligible defect can prevent the
    reachable-Fisher certificate from applying.
    \item \textbf{We establish what common summaries miss.} Full-model summaries and spectra computed separately for each group can both fail (Theorem~\ref{thm:onlyreachable}), and trace matching has precise limitations (Theorem~\ref{thm:trace}). We also derive a matrix-free Gram identity that computes the Frobenius discrepancy between the two groups' sampled reachable Fishers without forming either matrix (Proposition~\ref{prop:gram}).
    \item \textbf{We test the resulting diagnostic and intervention.} Across 306 single-seed models and 1{,}218 eligible model--split--subgroup-axis evaluations, higher reachable-Fisher trace accompanies greater subgroup difficulty in 75.7\% of cases. Trace matching during training narrows the best--worst subgroup gap in all 30 paired task--encoder--adapter cells (median relative reduction 17.2\%) and the surrogate-loss gap in 27. In the held-out audit, the observed-label operator discrepancy decreases in 23 of these 30, whereas the primary unbiased squared-Frobenius statistic decreases in only 16.

\end{itemize}
\textbf{Takeaways.}
Within a fixed model and parameterization, the reachable-Fisher trace is a
scalable diagnostic and matrix-free training target, but not a certificate of
matrix alignment, local gap direction, or improved subgroup performance. The Gram
identity makes what it misses auditable, suggesting a practical workflow: use the
trace as a heuristic, then verify subgroup outcomes and matrix discrepancies. Our
guarantees are local, concern signed surrogate-loss gaps under horizontal
additive updates, and do not extend to rate-parity criteria such as equalized
odds (Appendix~\ref{app:eod}).

\section{Related work}
\label{sec:related}

\textbf{PEFT and subgroup performance.} Adapters differ in rank, scaling,
magnitude--direction split, initialization and basis
\citep{hu2021lora, kalajdzievski2023rank, liu2024dora, meng2024pissa, kopiczko2024vera},
and adapter-space flatness need not imply full-space flatness
\citep{li2024flat}. Benchmarks find large method--task variation in
utility and fairness \citep{dutt2023parameter, jin2024fairmedfm}, no consistent
ordering between LoRA and full fine-tuning in subgroup utility
\citep{ding2024fairness}, and fairness-aware PEFT methods that select components,
reweight groups or constrain representations
\citep{dutt2024fairtune, sukumaran2024fairlora, kamalaruban2026fairness, zhou2026fairnet}.

\textbf{Curvature objects and their estimators.} Four matrices are often called
curvature: the Hessian of the loss; the generalized Gauss--Newton matrix, which
drops the Hessian's term involving second derivatives of the network; the
\emph{expected} Fisher, which averages score outer products over labels drawn
from the model; and the \emph{observed-label (empirical)} Fisher, which uses the
dataset labels instead. They coincide only under specific conditions
\citep{martens2020new}, and the
empirical Fisher is generally neither the Fisher nor an unbiased estimate of it
\citep{kunstner2019limitations}. Matrix-free products, Lanczos methods and
structured approximations make curvature tractable at scale
\citep{pearlmutter1994fast, ghorbani2019investigation, yao2020pyhessian, martens2015optimizing, dangel2019backpack};
sharpness is not reparameterization invariant \citep{dinh2017sharp}, which
matters when one predictor has many adapter coordinates. Fisher-guided PEFT uses
Fisher information to choose \emph{where} to adapt
\citep{sung2021training, feng2026learning}; we ask how subgroups act on a
subspace once it is chosen.

\textbf{Curvature and subgroups.} Pruning enlarges both accuracy gaps and gaps in
gradient norm or top Hessian eigenvalue \citep{tran2022pruning}; CUMA matches
input-space curvature across groups \citep{wang2023robust}; CurvFed, FairSAM and
FairMerging use curvature statistics inside training or merging
\citep{sharma2026curvfed, dai2025fairsam, liufairmerging}; FLARE uses a
label-dependent Fisher penalty to discover latent groups \citep{roy2026ethical}.
These methods use curvature \emph{inside} a procedure and validate it
empirically. We instead ask what reachable curvature \emph{provably} controls,
distinguish the empirical training proxy from the expected Fisher, and show that
the trace these methods rely on mixes a mean-gradient term with a dispersion term.
Appendix~\ref{app:related} gives an extended review.

\input{sec_theory}

\section{From theory to computable statistics}
\label{sec:experiments}

Section~\ref{sec:theory} identifies matrix quantities governing local subgroup-gap changes, but a dense $p\times p$ reachable Fisher is infeasible at scale: with $p=1.6$M, one requires about 10~TB. We compute three matrix-free statistics from per-example gradients: a model-expected trace for diagnosis, an observed-label trace for training, and an observed-label Frobenius discrepancy for the held-out audit; the workflow is to match cheaply, then verify at the matrix level. The experiments evaluate these statistics, not the theorems' conditions: they do not establish slope matching, horizontal updates, or a small Hessian--Fisher defect. Theory addresses surrogate-loss gaps while applications prioritize predictive performance, so Sections~\ref{sec:results} and~\ref{sec:mitigation-results} report both.

\subsection{Post-training diagnostic: does trace rank subgroup difficulty?}
\label{sec:diagnostic}

The diagnostic tests whether subgroups to which the fitted model is more locally
sensitive are also harder for that model. For each input we compute the
model-expected reachable-Fisher trace contribution with respect to $\phi$,
\begin{equation}
t_i=\begin{cases}
\sum_{c=1}^{C}q_\theta(c\mid x_i)
\lVert\nabla_\phi\log q_\theta(c\mid x_i)\rVert_2^2,
& \text{categorical},\\[1mm]
\sigma^{-2}\lVert\nabla_\phi\mu_\theta(x_i)\rVert_2^2,
& \text{Gaussian regression},
\end{cases}
\label{eq:trace-estimator}
\end{equation}
where $q_\theta(c\mid x_i)$ is the predicted probability of class $c$ among $C$ classes and $\mu_\theta(x_i)$ the predicted mean, with $\sigma^2$ estimated from held-out residuals; we average within subgroups
to obtain $\widehat t_a=\tr\widehat F^{\mathrm{R},\mathrm{exp}}_a$, the trace of a sample estimate of $\Fr_a$ in~\eqref{eq:reachable-fisher}.
For each fixed model, data split and subgroup axis (the attribute defining the groups, e.g.\ skin tone), we correlate these traces with subgroup
difficulties across the groups on that axis:
\begin{equation}
r=\operatorname{corr}_a(\widehat t_a,d_a),
\qquad
d_a=1-\operatorname{BAcc}_a\ \text{(classif.)},\quad
d_a=\operatorname{MAE}_a\ \text{(regr.)},
\label{eq:correlation}
\end{equation}
with balanced accuracy (BAcc) and mean absolute error (MAE) computed per group on
the same evaluation split as $\widehat t_a$. This is an \emph{aligned} Pearson correlation: $r>0$ means higher
trace accompanies greater difficulty. Because a correlation across two groups is
necessarily $\pm1$, it is reported only for axes with at least three groups.
Correlations are averaged after Fisher's $z$ transformation. Correlation measures
alignment; the coefficient of variation (CV) of $\{\widehat t_a\}$ separately
measures how unevenly trace is spread across groups; it is comparable only
within a fixed model and parameterization (Lemma~\ref{lem:projection}).

\subsection{Training intervention: can trace matching narrow subgroup gaps?} \label{sec:objective}  
The intervention tests whether equalizing the subgroup traces during training narrows their performance gap. A group-balanced sampler draws equally from the constrained groups, reducing sampling noise for rare groups. Let $u_i=\nabla_\phi\ell_i$ be the per-example observed-label gradient, via vectorized automatic differentiation, and let $\mathcal A_B$ denote the groups in batch $B$. For each group, $c_a=\frac{1}{n_a}\sum_{i\in a}\|u_i\|_2^2=\tr\widehat F^{\mathrm{R},\mathrm{obs}}_a$. We compare $c_a$ with the detached mean $c_{\mathrm{ref}}=\operatorname{sg}\!\big[|\mathcal A_B|^{-1}\sum_{a\in\mathcal A_B}c_a\big]$ through $e_a=(c_a-c_{\mathrm{ref}})/\max\{c_{\mathrm{ref}},\epsilon_{\mathrm{den}}\}$, which normalizes traces that vary by orders of magnitude and avoids division by zero. The penalty $R(\phi,\mu)=\sum_a\mu_a\rho_\kappa(e_a)$ uses the Huber loss with $\kappa=1$, quadratic near matching and linear for outlying estimates. We optimize
\begin{equation} \mathcal L_{\mathrm{task}}(\phi)+\omega R(\phi,\mu). 
\label{eq:objective-main} 
\end{equation} 
The weights $\mu$ follow a bounded exponentiated-gradient update on moving averages of $e_a$, giving more weight to groups that remain above the reference. With $g_{\mathrm{task}}=\nabla_\phi\mathcal L_{\mathrm{task}}$ and $g_{\mathrm{pen}}=\nabla_\phi R$, we set $\omega=\min\{1,\beta\|g_{\mathrm{task}}\|/\|g_{\mathrm{pen}}\|\}$ with $\beta=\tfrac12$. This caps the penalty-gradient norm at half the task-gradient norm and guarantees $g_{\mathrm{task}}^\top(g_{\mathrm{task}}+\omega g_{\mathrm{pen}})\ge \tfrac12\|g_{\mathrm{task}}\|_2^2$, so before AdamW preconditioning the penalty may rotate but cannot reverse the task gradient (Appendix~\ref{sec:method} reports uncapped behavior). 
Retaining the task loss prevents the solution in which $R$ is reduced
by making the predictor insensitive.

\paragraph{What the intervention does not guarantee.} The penalty acts on the trace, which leaves three gaps between it and the theory. It does not enforce slope matching, so the linear term of~\eqref{eq:gap-expansion} is untouched, and $c_a=\|\hat g_a\|^2+\tr\widehat\Sigma_a$ mixes a mean-gradient term with within-group dispersion (Lemma~\ref{lem:trace-decomp}). It acts on the observed-label empirical Fisher rather than the model-expected Fisher of the theory. And a small trace gap bounds a matrix discrepancy only under Loewner ordering (Theorem~\ref{thm:trace}(iii)), which we do not measure. Its effect therefore has to be checked at the matrix level: after training we measure, for the baseline and trace-matched models on the same held-out group-balanced batches, the trace gap and the Frobenius discrepancy of Proposition~\ref{prop:gram}, which bounds the curvature channel up to defect and estimator error (Corollary~\ref{cor:closure}); the audit protocol is in Appendix~\ref{sec:frobenius}.

\subsection{Datasets and model grid}
\label{sec:data}

Table~\ref{tab:data-main} lists the four public datasets and the model grid.
The 17 pretrained image encoders span supervised, self-supervised,
convolutional, hierarchical and vision--language pretraining, each trained with
full fine-tuning, LoRA, DoRA, rsLoRA, PiSSA and VeRA (PiSSA and VeRA cannot
attach to ResNet-50), giving $306$ models; training details
are in Appendix~\ref{app:repro}.

\begin{table}[!h]
\vspace{-0.5em}
\caption{Diagnostic datasets and model grid. A subgroup axis is the attribute
defining the groups; group counts in parentheses. Full details in
Table~\ref{tab:data}.}
\label{tab:data-main}
\centering
\footnotesize
\resizebox{\linewidth}{!}{%
\begin{tabular}{@{}llll@{}}
\toprule
Dataset & Task & Subgroup axes (groups) & Models \\
\midrule
Fitzpatrick17k \citep{groh2021evaluating} & 3-class diagnosis & skin tone (7), binary tone (3) & 17 encoders \\
UTKFace \citep{zhang2017age} & age regression & race (5) & 17 encoders \\
FHIBE face \citep{xiang2025fair} & age regression & ancestry (6) & 17 encoders \\
Diabetes, 130 hospitals \citep{strack2014impact} & 30-day readmission & race (6) &
\shortstack[l]{FT-Transformer \\ \citep{gorishniy2021revisiting}} \\
\bottomrule
\end{tabular}
}
\vspace{-0.5em}
\end{table}

\subsection{Statistical units and uncertainty}
\label{sec:stats}

This subsection defines the 306-model diagnostic sweep; Section~\ref{sec:mitigation-results}
describes the 30 paired intervention cells separately. A diagnostic \emph{cell}
is one trained model evaluated on one split and one subgroup axis (the ``train''
split is an internal slice of the training partition never used for gradient
updates, so all splits are out-of-fit), for example a
ViT--LoRA model on the UTKFace test split across race groups; each eligible cell
contributes one correlation~\eqref{eq:correlation}. Excluding two-group axes leaves $1{,}218$ cells.
 The
\emph{positive-sign rate} is the fraction of cells with $r>0$. Cells are not
independent (they share encoders, datasets and splits), so cell shares are
descriptive and the 17 encoders are the repeated unit for image-task uncertainty
(2,000-replicate bootstrap). The performance gap of a model on an axis
is $\max_ad_a-\min_ad_a$. The sweep uses one training seed.

\section{Results}
\label{sec:results}

\textbf{Reachable Fisher tracks subgroup difficulty, heterogeneously.} Across 1{,}218 cells, $r$ is positive in 75.7\% and has a descriptive Fisher-$z$ mean of $0.637$. Averaging encoders equally (Table~\ref{tab:alignment}), alignment is strongest for UTKFace ($0.930$), intermediate for FHIBE face ($0.601$), positive but architecture-sensitive for Fitzpatrick17k ($0.320$), where one encoder is strongly negative (Figure~\ref{fig:results}, Appendix~\ref{app:additional}), and weak for the single tabular architecture ($0.117$), which has no encoder-level interval. Replacing task difficulty with the surrogate loss, the quantity the theorems concern, strengthens the pooled association to 85.5\% positive over 406 test-split cells, though not uniformly: Fitzpatrick17k nearly doubles to $0.611$ while FHIBE face falls to $0.411$. Method-specific and normalization checks are in Appendix~\ref{app:additional}. The sweep therefore supports the ranking premise behind the penalty but does not test a theorem; by Lemma~\ref{lem:projection}, traces are comparable only within a fixed model and parameterization.

\begin{table}[!h]
\vspace{-0.5em}
\caption{Subgroup Fisher--difficulty alignment. Image-task means weight encoders equally, with encoder-bootstrap 95\% intervals; sign rates are descriptive. Surrogate-loss results use the test split only, which is why their cell counts are smaller. The pooled row is a cell-level Fisher-$z$ mean and is not comparable with the encoder-cluster means.}
\label{tab:alignment}
\centering
\small
\resizebox{\linewidth}{!}{%
\begin{tabular}{lrrrrrr}
\toprule
 & \multicolumn{3}{c}{task-metric difficulty} & \multicolumn{3}{c}{surrogate-loss difficulty} \\
\cmidrule(lr){2-4}\cmidrule(lr){5-7}
Setting & cells & positive sign & encoder-cluster mean $r$ [95\% CI] & cells & positive sign & mean $r$ \\
\midrule
Diabetes FT-Transformer & 18 & 61.1\% & $0.117$ [not estimable] & 6 & 66.7\% & $+0.284$ \\
Fitzpatrick17k & 600 & 59.3\% & $0.320\ [0.184,\,0.424]$ & 200 & 81.5\% & $+0.611$ \\
UTKFace & 300 & 98.3\% & $0.930\ [0.901,\,0.950]$ & 100 & 98.0\% & $+0.886$ \\
FHIBE face & 300 & 86.7\% & $0.601\ [0.511,\,0.679]$ & 100 & 82.0\% & $+0.411$ \\
\midrule
All cells (pooled) & 1{,}218 & 75.7\% & $0.637$ & 406 & 85.5\% & $+0.668$ \\

\bottomrule
\end{tabular}
}
\vspace{-0.5em}
\end{table}

\input{sec_mitigation}


\section{Discussion and conclusion}
\label{sec:discussion}
\label{sec:conclusion}

Subgroup behaviour under PEFT depends on the curvature visible through the adapter, and the sign and magnitude of a gap change need different matrix information. The theory is sharp but conditional by construction: it needs slope matching, horizontal updates, and a small defect. On a small LoRA network where every matrix is formed exactly, the expansion and the worst-case term behave as proved while the reachable-Fisher certificate fails, the defect being large (Appendix~\ref{app:exact}); the distance between the theory's objects and the computable ones is measured here, not assumed (Appendix~\ref{sec:theory-scope}). Empirically, the trace orders subgroup difficulty within a model and, as a training signal, narrows gaps modestly and unevenly; we recommend it within a fixed parameterization, alongside direct outcome measurements. \textbf{Scope and outlook.} Closing that distance is the natural next step: estimating the defect directly, moving beyond the in-distribution regime, where Theorem~\ref{thm:bayesfloor} leaves loss geometry least room, and training against the audited discrepancy rather than its trace.

\subsection*{AI Use Statement}
We used generative AI tools for language editing, \LaTeX{} debugging, and consistency checks against experimental outputs. All AI-assisted work was reviewed and verified by the authors, who take full responsibility for the final content.

\subsection*{Ethics statement}

This is a secondary analysis of existing benchmark datasets and stored model outputs; it collects no new human-subject data. The attributes labeled skin tone, race, sex, and gender are dataset-specific measurements or annotations and should not be treated as interchangeable biological categories. Subgroup estimates for small cells can be unstable, and a diagnostic association must not be interpreted as evidence that a model is safe or equitable in deployment. Any release must follow the original datasets' licenses and governance requirements. Models evaluated here are research artifacts and are not validated for clinical use.

\subsection*{Reproducibility statement}

\textbf{All the code used in this study is made available through} \href{https://anonymous.4open.science/r/Optimization-Fisher-With-Low-Ranked-Adaptation-In-Deep-Learning-Models-461E}{this Anonymous GitHub repository}.\footnote{\url{https://anonymous.4open.science/r/Optimization-Fisher-With-Low-Ranked-Adaptation-In-Deep-Learning-Models-461E}}

Section~\ref{sec:theory} states the main formal results and their assumptions;
Appendix~\ref{app:theory} gives the extended statements and scope, and
Appendix~\ref{app:proofs} provides complete proofs and machine-precision checks
of the constructed examples and algebraic identities.
Section~\ref{sec:experiments} specifies the model grid, metrics and statistical
unit; Appendix~\ref{app:mitigation} lists every paired run and documents how the
reported results are regenerated; Appendix~\ref{sec:method} specifies the
training objective and matrix-audit procedure; and
Appendix~\ref{app:exact} presents the exact small-network verification.
Appendix~\ref{app:repro} records checkpoint identifiers, splits, estimator
details, hyperparameters, seeds, and exclusions. The supplemental repository
contains the experiment runner, saved outputs, analysis materials, and the
materials needed to regenerate the reported tables and figures.

\bibliography{references}
\bibliographystyle{iclr2027_conference}

\appendix

\input{app_theory}

\input{theory_proofs}

\input{methodology}

\section{Extended related work}
\label{app:related}

\subsection{Parameter-efficient adaptation and fair fine-tuning}

PEFT changes more than the number of trainable parameters: it changes the coordinates and directions through which a pretrained model can move. LoRA restricts each weight update to a low-rank product \citep{hu2021lora}; rsLoRA changes its rank-dependent scaling, DoRA separates magnitude from direction, PiSSA initializes from leading singular components, and VeRA learns scaling vectors over shared random bases \citep{kalajdzievski2023rank, liu2024dora, meng2024pissa, kopiczko2024vera}. Expressivity results make the role of rank explicit \citep{zeng2024expressive}, while Flat-LoRA shows that a solution that is flat in adapter coordinates can remain sharp in the full parameter space \citep{li2024flat}.

Broad evaluations show why this question should not be answered from one adapter or backbone. A medical-imaging benchmark spanning 17 PEFT methods and six datasets found strong performance in low-data regimes but substantial method--task variation \citep{dutt2023parameter}. FairMedFM similarly reports heterogeneous utility--fairness trade-offs across 20 foundation models, 17 medical datasets, and several adaptation protocols \citep{jin2024fairmedfm}. Direct comparisons of LoRA and full fine-tuning find no consistent ordering in subgroup utility or calibration \citep{ding2024fairness}, and full-model fine-tuning studies show that worst-group accuracy can depend sharply on balancing strategy and backbone scale \citep{labonte2024group}. Fairness-aware PEFT methods respond by selecting trainable components, regularizing group-conditioned losses, activating conditional low-rank updates, or removing sensitive information from the learned representation \citep{dutt2024fairtune, sukumaran2024fairlora, kamalaruban2026fairness, zhou2026fairnet, bhosale2026fairllava}. Our method is complementary: it regularizes subgroup sensitivity in the adapter space and evaluates outcome-level fairness separately.

\subsection{The Hessian, generalized Gauss--Newton matrix, and Fisher information}

Several matrices called ``curvature'' appear in deep learning, but they answer different questions. The Hessian $H=\nabla_\theta^2 L$ is the exact local second derivative of the chosen training objective. It can be indefinite in a nonconvex network and therefore records both positive and negative curvature. If the loss is a convex function of the model output, the generalized Gauss--Newton (GGN) matrix pulls the output-space loss Hessian back through the network Jacobian. The GGN is positive semidefinite and omits the residual term containing second derivatives of the network itself \citep{botev2017practical, martens2020new}. Thus, the Hessian and GGN coincide for a model linear in its parameters, and can be close when the omitted model-curvature term is small, but they are not generally identical.

For a conditional probabilistic model, the expected Fisher is the covariance of score gradients when the label is sampled from the model. It is positive semidefinite and defines the Riemannian metric used by natural-gradient descent \citep{amari1998natural, martens2020new}. Under standard regularity conditions, the model-expected negative log-likelihood Hessian equals the Fisher; for common exponential-family output models, the Fisher also equals the corresponding GGN \citep{martens2020new}. These identities explain why the Fisher can serve as a stable Hessian surrogate in likelihood models, but only for the relevant expectation and loss. The equality need not hold for a finite dataset, a misspecified model, an arbitrary loss, or a restricted parameterization.

A further distinction is crucial here. The empirical Fisher is the average outer product of gradients evaluated at the \emph{observed} labels. Despite its name, it is generally neither the Fisher under the model distribution nor an unbiased Monte Carlo estimate of that Fisher; it also need not approximate the Hessian well \citep{kunstner2019limitations}. Its trace is simply the mean squared norm of per-example observed-label gradients. That makes it inexpensive, nonnegative, and directly tied to the examples driving an optimizer, but it should be interpreted as an optimization-sensitivity statistic rather than as exact second-order curvature. Our training penalty uses this empirical trace. Our post-training diagnostic instead samples or sums over model-predicted labels to estimate the expected Fisher. We keep the two quantities separate throughout.

\subsection{Estimating curvature at neural-network scale}

Forming a dense Hessian or Fisher for $p$ parameters costs $O(p^2)$ storage, so most methods estimate only the structure needed by a downstream task. Pearlmutter's reverse-over-forward construction computes exact Hessian--vector products at roughly gradient cost without materializing the matrix \citep{pearlmutter1994fast}. Hessian-free optimization combines such products with iterative linear solvers \citep{martens2010deep}, while stochastic Lanczos quadrature uses products to estimate the eigenvalue density and leading spectral structure \citep{ghorbani2019investigation}. PyHessian packages trace, top-eigenvalue, and spectral-density estimation for modern networks \citep{yao2020pyhessian}. These matrix-free methods preserve more global structure than a diagonal estimate, but require repeated products.

Structured approximations trade some fidelity for cheaper storage and inversion. K-FAC approximates layerwise Fisher blocks by Kronecker products \citep{martens2015optimizing}; related block-diagonal Gauss--Newton methods propagate curvature through the network \citep{botev2017practical}. BackPACK extends backpropagation to recover per-example gradients, diagonal Hessian or GGN terms, and matrix-factor representations \citep{dangel2019backpack}. At the simplest end, diagonal Fisher, diagonal Hessian, and their traces retain only coordinatewise or aggregate sensitivity. Hutchinson's randomized estimator recovers a trace or diagonal from matrix--vector products \citep{hutchinson1989stochastic}; AdaHessian combines such diagonal estimation with spatial and temporal averaging inside an optimizer \citep{yao2021adahessian}; and low-rank or Lanczos approximations retain dominant eigendirections. The choice is task dependent: natural-gradient updates require an inverse or linear solve, sharpness studies emphasize extreme eigenvalues or neighborhoods, and pruning or importance scoring often needs only diagonal entries.

Our estimator occupies the inexpensive end of this spectrum. For each example, it obtains the gradient only with respect to the trainable DoRA and head parameters, then sums its squared entries. Averaging these scalars within a subgroup gives the trace of a subgroup empirical-Fisher matrix in the reachable adaptation coordinates. We do not form, invert, or eigendecompose a matrix. Differentiating the trace penalty during training does require second-order automatic differentiation, which motivates the chunked accumulation described in Section~\ref{sec:method}.

\subsection{What Fisher and Hessian estimates have been used for}

The oldest use is optimization. Newton and Hessian-free methods use loss curvature, whereas natural gradient preconditions the gradient by the Fisher metric and is approximately invariant to smooth reparameterization \citep{amari1998natural, martens2010deep, martens2020new}. K-FAC makes a structured natural-gradient approximation practical for neural networks \citep{martens2015optimizing}. Fisher SAM similarly defines the adversarial neighborhood of sharpness-aware minimization with information geometry rather than a Euclidean norm \citep{foret2020sharpness, kim2022fisher}.

Curvature also supports diagnosis and generalization studies. Work on flat minima relates local sensitivity to generalization \citep{hochreiter1997flat, keskar2016large}, and visualization and spectral studies examine Hessian slices, outliers, and bulk eigenvalue structure \citep{li2018visualizing, ghorbani2019investigation}. However, Euclidean sharpness can be changed by a function-preserving reparameterization \citep{dinh2017sharp}; information-geometric measures have therefore been proposed to obtain reparameterization-invariant notions of sensitivity \citep{liang2019fisher, jang2022reparametrization}. This caveat is especially relevant to PEFT, where the same predictor can be represented in different adapter coordinates.

A third line of work treats curvature as parameter importance. Optimal Brain Damage used diagonal second derivatives for pruning \citep{lecun1989optimal}; modern methods such as WoodFisher approximate an inverse empirical Fisher to predict the effect of deleting weights \citep{singh2020woodfisher}. Fisher-weighted penalties protect parameters important to old tasks in elastic weight consolidation \citep{kirkpatrick2017overcoming}, and Fisher-weighted averaging supports model merging \citep{matena2022merging}. Hessian inverses also underlie influence functions that approximate how training points affect predictions \citep{koh2017understanding}. In Bayesian deep learning, the Hessian, GGN, or Fisher supplies the local precision in Laplace approximations \citep{ritter2018scalable, immer2021improving}. These uses are related by a common local-sensitivity calculation, but they do not make the underlying matrices interchangeable.

Fisher-guided PEFT brings these ideas into restricted trainable subspaces. FISH Mask and FISH-Tuning select trainable parameters or modules \citep{sung2021training, xue2025fish}; Learning in the Fisher Subspace and FiLoRA select or initialize low-rank directions \citep{feng2026learning, han2026filora}; FI-LoRA allocates rank using importance, and FoRA selects adapted layers \citep{zeng2026fi, park2026fora}. These methods use Fisher information to decide \emph{where} adaptation should occur. Our question is different: once a trainable subspace has been chosen, do observed subgroups act on that subspace with comparable strength?

\subsection{Subgroup robustness with incomplete demographics}

When group annotations are available, group DRO optimizes worst-group risk, while equal opportunity and equalized odds compare conditional error rates; constrained reductions turn such criteria into trainable objectives \citep{sagawa2019distributionally, hardt2016equality, agarwal2018reductions}. Subgroup fairness, multicalibration, and multiaccuracy extend evaluation beyond a small set of coarse, predeclared groups \citep{kearns2018preventing, hebert2018multicalibration, kim2019multiaccuracy}. The choice of coarse, fine-grained, intersectional, or noisy group definitions can itself change which mitigation method appears effective \citep{alloula2025subgroups}. This is particularly important in medical imaging, where conclusions also vary with the dataset, backbone, fairness metric, and model-selection rule \citep{zong2022medfair, jin2024fairmedfm}.

Without complete demographic annotations, prior work replaces named groups with distributional, error-based, or representation-based surrogates. Distributionally robust and adversarially reweighted objectives protect high-loss subsets without observing their identities \citep{hashimoto2018fairness, lahoti2020fairness}. GEORGE and EIIL infer latent subclasses or environments, whereas Just Train Twice uses errors from an initial model to upweight likely minority examples \citep{sohoni2020no, creager2021environment, liu2021just}. More recent methods minimize training-loss variance without demographic priors or construct soft neighborhoods from observed-label gradients \citep{wang2024towards, luo2025fairness}. In medical imaging, feature-based subgroup discovery can expose performance gaps that are larger than those found from available metadata \citep{bissoto2025subgroup}. Our method does not infer latent groups: it assumes recorded subgroup labels for both aggregation and training. The expected-Fisher contribution is computed per input without a subgroup label, but forming a subgroup statistic still requires the labels at the aggregation step.

\subsection{Curvature and subgroup behavior}

Several fairness studies nevertheless use Hessian or Fisher-based quantities as subgroup signals. Pruning can enlarge both accuracy gaps and gaps in gradient norm or the top Hessian eigenvalue \citep{tran2022pruning}. CUMA matches distributions of input-space loss curvature across observed groups under shift \citep{wang2023robust}. CurvFed regularizes the top eigenvalue of a parameter-space Fisher surrogate without demographics in federated training \citep{sharma2026curvfed}; FairSAM modifies sharpness-aware training to balance robustness across groups \citep{dai2025fairsam}; and FairMerging uses subgroup gradient and curvature statistics to control disparities introduced by weight-space model merging \citep{liufairmerging}. Among fairness-without-demographics methods, the recent FLARE preprint is closest in motivation: it combines embeddings, cross-entropy loss, and a label-dependent Fisher penalty to discover latent behavioral clusters, then trains cluster-specific models \citep{roy2026ethical}. All of these methods use geometry inside an optimization, aggregation, or subgroup-discovery procedure.

Our contribution differs from all of these in what it establishes rather than in
which statistic it computes. Prior work uses curvature \emph{inside} an
optimization, aggregation, or subgroup-discovery procedure and validates it
empirically. We instead ask what curvature in the reachable subspace
\emph{provably} controls. Section~\ref{sec:theory} shows that,
once the slope channel is controlled, a reachable-Fisher difference that is
positive definite on the horizontal subspace, above the Hessian--Fisher defect,
forces the signed gap upward under every small horizontal update; that
full-model statistics that ignore the adapter Jacobian can err in both
directions; and that trace equality cannot control the operator norm that sets
the worst-case curvature channel. This offers one
possible explanation for why adapter-robustness findings conflict
\citep{ding2024fairness, sukumaran2024fairlora, li2024flat}, though we do not
claim it accounts for any particular disagreement in that literature. Unlike CUMA, our perturbation is
in parameter rather than input space. Unlike CurvFed and FLARE, we distinguish
the empirical training proxy from the expected-Fisher geometry used for
evaluation, and we prove that the trace both methods rely on conflates a
mean-gradient term with a dispersion term. Unlike work that uses curvature to
accelerate optimization, prune parameters, or approximate a posterior, our
target is a matrix discrepancy across known subgroups, with a matrix-free
estimator and a guarantee attached to it. Finally, we state a limit: the
performance gap converges to a Bayes-risk difference rather than being bounded
below by it. This lets the premise be tested without treating
geometric alignment as a fairness guarantee.

\section{Reproducibility details}
\label{app:repro}

\begin{table}[ht]
\caption{Evaluation settings. Split sizes are the stored fine-tuning train/validation/test partitions; diabetes additionally uses 42,741 earlier encounters for pretraining before applying FFT or PEFT. ``Axes'' lists the criteria of subgroups used in curvature-difficulty correlations.}
\label{tab:data}
\centering
\small
\resizebox{\linewidth}{!}{%
\begin{tabular}{lllrll}
\toprule
Dataset & task & output & train/val/test & axes (groups) & encoders \\
\midrule
Fitzpatrick17k & image classification & 3 classes & 11,592/2,484/2,484 & tone (7), binary tone (3) & 17 \\
UTKFace & image regression & age & 13,796/2,956/2,957 & race (5) & 17 \\
FHIBE face & image regression & age & 5,556/1,192/1,210 & ancestry (6) & 17 \\

Diabetes 130H& tabular classification & binary & 28,495/20,353/10,177 & race (6) & FT-T \\
\bottomrule
\end{tabular}
}
\end{table}

\subsection{Training}

Every model receives a fresh prediction head. Full fine-tuning updates all backbone and head parameters; PEFT updates the head and adapter only. LoRA-family rank is 8 with scale 16 and no adapter dropout; VeRA rank is 256. We use AdamW for at most 20 epochs, two warm-up epochs followed by cosine decay, gradient clipping at 1, and early stopping with patience 5. Image learning rates are $10^{-4}$ for full fine-tuning, $5\times10^{-4}$ for LoRA-family methods, and $5\times10^{-3}$ for VeRA; the tabular full-fine-tuning rate is $5\times10^{-4}$. Each training partition contains an internal 15\% held-out slice that is never used for gradient updates, so its reported ``train'' metrics are out-of-fit.

\subsection{Image roster}

The image roster comprises ViT, DeiT, CrossViT, RoPE-ViT, MaxViT, Swin, ResNet-50, ConvNeXt, ConvNeXt V2, SigLIP~2, MetaCLIP~2, DINOv2, DINOv3, I-JEPA, MAE, BEiT, and EVA-02. It spans supervised, self-supervised, convolutional, hierarchical, and vision--language pretraining families \citep{dosovitskiy2020image, touvron2021training, chen2021crossvit, heo2024rotary, tu2022maxvit, liu2021swin, he2016deep, liu2022convnet, woo2023convnext, tschannen2025siglip, chuang2026meta, oquab2023dinov2, simeoni2025dinov3, assran2023self, he2022masked, bao2021beit, fang2024eva}. Each compatible block is trained with full fine-tuning, LoRA, DoRA, rsLoRA, PiSSA, and VeRA. PiSSA and VeRA cannot attach to the convolution-only ResNet implementation, producing 100 models per image dataset rather than 102. With six models for the tabular architecture, the completed grid contains $3\times100+6=306$ models.

\subsection{Data preparation and leakage controls}

Fitzpatrick17k removes 17 rows explicitly marked as wrongly labeled, deduplicates byte-identical files by hash, and stratifies a 70/15/15 split on apparent skin tone crossed with the three-way target. The source does not provide lesion or patient identifiers; near-duplicate crops may therefore cross splits. FHIBE face keeps primary-subject rows only (7,958 crops from 1,651 subjects; the secondary rows carry another person's demographics) and splits 70/15/15 grouped by \texttt{subject\_id}, stratified on ancestry crossed with the subject's median age binned at $<$25/25--34/35--49/50+. No subject appears in more than one split, which matters because subjects contribute several photographs; the 15\% out-of-fit slice within Train is itself subject-grouped.
 UTKFace uses a 70/15/15 split stratified by the race--gender cross-product. The diabetes cohort is ordered by encounter identifier: 42\% pretraining, 28\% fine-tuning train, 20\% validation, and 10\% test. Both tabular models use identical rows. Image training partitions reserve 15\% of the nominal training split for out-of-fit reporting; the same principle is used by the tabular pipeline. Split seed, run seed, and bootstrap seed are all 12,345 but are represented separately in code.

\subsection{Exact model roster}

The image checkpoint identifiers are: \nolinkurl{google/vit-base-patch16-224}, \nolinkurl{facebook/deit-base-patch16-224}, \nolinkurl{crossvit_18_dagger_408.in1k}, \nolinkurl{vit_base_patch16_rope_reg1_gap_256.sbb_in1k}, \nolinkurl{maxvit_base_tf_384.in1k}, \nolinkurl{microsoft/swin-base-patch4-window7-224}, \nolinkurl{microsoft/resnet-50}, \nolinkurl{facebook/convnext-base-224}, \nolinkurl{facebook/convnextv2-base-22k-224}, \nolinkurl{google/siglip2-base-patch16-224}, \nolinkurl{facebook/metaclip-2-worldwide-b16}, \nolinkurl{facebook/dinov2-base}, \nolinkurl{facebook/dinov3-vitb16-pretrain-lvd1689m}, \nolinkurl{facebook/ijepa_vith14_1k}, \nolinkurl{facebook/vit-mae-base}, \nolinkurl{microsoft/beit-base-patch16-224-pt22k}, and \nolinkurl{eva02_base_patch14_448.mim_in22k_ft_in22k_in1k}. FT-Transformer uses numeric feature tokens, categorical embeddings, and Transformer blocks. Tabular preprocessing is fitted on the pretraining partition and then frozen to prevent leakage.

\subsection{Fisher computation}

Per-sample gradients are computed with functional automatic differentiation and vectorization, with a per-sample loop fallback for unsupported operators. For multiclass outputs, the implementation visits all classes and weights squared score gradients by the predicted class probabilities. For binary outputs it evaluates the Bernoulli expectation. For regression, $\sigma^2$ is estimated from held-out residuals and is never replaced with the observed target inside the score. A single unshuffled pass stores per-example traces; subgroup means and bootstrap intervals are obtained by masking this vector. Thus no subgroup is given a different Fisher pass.

\section{Additional empirical results}
\label{app:additional}

\paragraph{Cross-method dispersion does not track gaps.} Within each backbone
and axis we form the ratio of PEFT to full-fine-tuning Fisher CV, and likewise
for the gap, over the 16 encoders that support every method. These ratios
compare different parameterizations and are descriptive
(Lemma~\ref{lem:projection}). Ten of 15 method--task CV ratios are below one and 11 of 15 gap ratios are, with Spearman $\rho=-0.30$ across 15 non-independent pairs (Table~\ref{tab:peft}): a
lower cross-method CV does not accompany a smaller gap.

\begin{figure}[ht]
  \centering
  \includegraphics[width=0.58\linewidth]{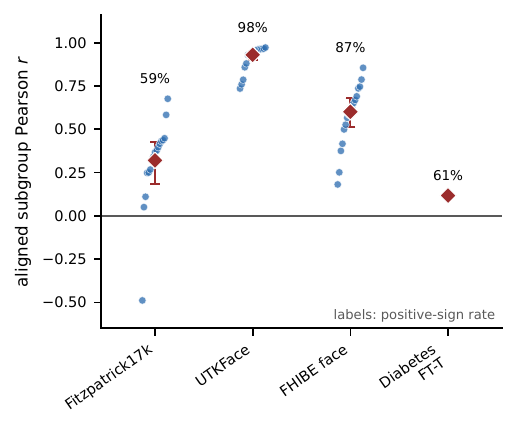}
  \caption{Encoder-level alignment by task. Blue points are encoder-level
  Fisher-$z$ means, red diamonds are equal-encoder task means, and error bars
  bootstrap encoders; percentages are descriptive positive-sign rates over
  cells. Summary values are in Table~\ref{tab:alignment} and
  method-conditioned means in Table~\ref{tab:method-alignment}.}
  \label{fig:results}
\end{figure}

\begin{table}[ht]
\caption{Balanced 16-encoder comparison to full fine-tuning. Fitzpatrick17k entries are balanced-accuracy changes in percentage points (higher is better); UTKFace and FHIBE face are MAE changes in years (lower is better). Ratios are geometric means across the three image tasks; values below one denote a reduction. Both ratios compare different parameterizations and are descriptive.}
\label{tab:peft}
\centering
\small
\begin{tabular}{lrrrrrr}
\toprule
MMethod & params (\%) & $\Delta$Fitz BAcc & $\Delta$UTK MAE & $\Delta$FHIBE MAE & Fisher-CV ratio & gap ratio \\
\midrule
LoRA   & 1.57 & $+2.81$ & $-0.34$ & $-1.33$ & 1.077 & 0.823 \\
DoRA   & 1.66 & $+3.21$ & $-0.31$ & $-1.38$ & 1.095 & 0.902 \\
rsLoRA & 1.57 & $+1.46$ & $-0.28$ & $-0.79$ & 1.123 & 0.901 \\
PiSSA  & 1.55 & $+1.01$ & $+0.00$ & $-0.88$ & 0.956 & 0.946 \\
VeRA   & 0.10 & $-0.19$ & $-0.04$ & $-1.23$ & 0.603 & 0.745 \\
\bottomrule
\end{tabular}
\end{table}

\subsection{Robustness checks}

Three design checks delimit the result. First, method-stratified alignment remains high on UTKFace and positive for every method on both dermatology tasks, so the headline is not produced by full fine-tuning alone. Second, replacing $\widehat t_a$ with $\widehat t_a/\widehat t_{\mathrm{pooled}}$ leaves every within-split correlation exactly unchanged; this verifies that normalization cannot manufacture the sign. Third, removing two-level sex and gender axes prevents the algebraic $r=\pm1$ artifact. The remaining axes contain three to fourteen groups. These checks do not address run-to-run training variation; multi-seed replication remains necessary.

\subsection{Per-task PEFT ratios}

Table~\ref{tab:full-ratios} expands the geometric means in Table~\ref{tab:peft}. The roster has 16 encoders because PiSSA and VeRA cannot attach to the convolution-only ResNet-50, which is therefore excluded from every cross-method comparison. For method $m$, task $t$ and encoder $e$, let $V_{m,t,e}$ be the Fisher CV and $G_{m,t,e}=\max_ad_a-\min_ad_a$ the test-split performance gap. The entries are $\bar V_{m,t}=\operatorname{median}_e\,V_{m,t,e}/V_{\mathrm{FT},t,e}$ and $\bar G_{m,t}=\operatorname{median}_e\,G_{m,t,e}/G_{\mathrm{FT},t,e}$, where a ratio is defined only when its full-fine-tuning denominator is positive, and Table~\ref{tab:peft} reports the geometric mean over the three tasks. 
The entries are medians over the 16 encoders; encoders are weighted equally, and none is dropped.
Fitzpatrick17k has two subgroup axes (skin tone and binary tone); for that task we average the two values of $G_{m,t,e}$, and likewise of $V_{m,t,e}$, before forming the ratio. The other tasks have a single axis.
Every full-fine-tuning denominator $G_{\mathrm{FT},t,e}$ was strictly positive, so no encoder-level ratio was dropped for a zero denominator.
These ratios compare different parameterizations and are descriptive (Lemma~\ref{lem:projection}). Fisher CV is reduced on Fitzpatrick17k and UTKFace but increased on FHIBE face for LoRA, DoRA and rsLoRA. Direct subgroup performance gaps do not follow the same pattern: DoRA increases the Fitzpatrick gap by 19.5\%, and PiSSA increases the UTKFace gap by 1.9\% despite reducing Fisher dispersion.

\begin{table}[ht]
\caption{Per-task ratios relative to full fine-tuning on the balanced 16-encoder roster.}
\label{tab:full-ratios}
\centering
\small
\begin{tabular}{llrrrrr}
\toprule
Quantity & task & LoRA & DoRA & rsLoRA & PiSSA & VeRA \\
\midrule
\multirow{3}{*}{Fisher CV}
& Fitzpatrick17k & 0.953 & 0.845 & 1.009 & 0.907 & 0.913 \\
& UTKFace        & 0.933 & 0.927 & 0.915 & 0.950 & 0.537 \\
& FHIBE face     & 1.405 & 1.676 & 1.536 & 1.016 & 0.446 \\
\midrule
\multirow{3}{*}{Performance gap}
& Fitzpatrick17k & 0.943 & 1.195 & 1.031 & 0.977 & 1.013 \\
& UTKFace        & 0.856 & 0.938 & 0.915 & 1.019 & 0.929 \\
& FHIBE face     & 0.690 & 0.656 & 0.777 & 0.851 & 0.440 \\

\bottomrule
\end{tabular}
\end{table}

\subsection{Method-conditioned difficulty alignment}

Table~\ref{tab:method-alignment} reports method-conditioned alignment. The tabular setting has only three correlations per entry, one per split, so its values are descriptive and should not be used to rank methods.

\begin{table}[ht]
\caption{Fisher-$z$ mean aligned Pearson $r$ by setting and method.}
\label{tab:method-alignment}
\centering
\small
\begin{tabular}{lrrrrrr}
\toprule
Setting & full FT & LoRA & DoRA & rsLoRA & PiSSA & VeRA \\
\midrule
Fitzpatrick17k          & 0.337 & 0.217 & 0.422 & 0.321 & 0.381 & 0.330 \\
UTKFace                 & 0.897 & 0.929 & 0.942 & 0.932 & 0.934 & 0.945 \\
FHIBE face              & 0.400 & 0.700 & 0.626 & 0.587 & 0.552 & 0.655 \\
Diabetes FT-T & 0.149 & 0.130 & $-0.039$ & 0.163 & 0.470 & $-0.206$ \\
\bottomrule
\end{tabular}
\end{table}

\input{app_mitigation}

\section{Equalized-rate criteria}
\label{app:eod}

Our claims concern the per-group performance gap. Rate-parity criteria condition on the label and are a different object. We define them here because they appear as a secondary outcome in Section~\ref{sec:mitigation-results}, and we give a construction showing why the curvature analysis should not be expected to bear on them.

\paragraph{Definitions.} Equal-opportunity difference (EOD) is the between-group difference in true-positive rates; average-odds difference (AOD) averages the true-positive and false-positive rate differences \citep{hardt2016equality}. Both condition on the label, so they apply only to the classification settings, Fitzpatrick17k and Diabetes 130H. We do not threshold the regression targets, and no rate metric is computed for UTKFace or FHIBE face. Two conventions make the classification case concrete. Fitzpatrick17k's three-way target is binarized one-vs-rest, each class in turn positive, and the three values are macro-averaged; diabetes readmission is already binary. Every axis has more than two groups (skin tone seven, binary tone three, diabetes race six), so we compute per-group rates and report the largest pairwise difference $\max_{g,g'}\lvert r_g-r_{g'}\rvert$, which is nonnegative and smaller is better.

\paragraph{Why no association should be expected.} Per-group task difficulty and
between-group rate parity are related but non-equivalent, and the
non-equivalence is exact.

\begin{theorem}[Loss geometry does not determine EOD]
\label{thm:eod}
There exist two subgroup distributions and a one-parameter model
$q_\phi(x_1)=\sigma(\phi)$, $q_\phi(x_2)=1-\sigma(\phi)$ such that for
\emph{every} $\phi$,
\[
L_a(\phi)=L_b(\phi),\quad g_a(\phi)=g_b(\phi),\quad H_a(\phi)=H_b(\phi),
\quad \Fr_a(\phi)=\Fr_b(\phi),\quad \eta_{ab}(\phi)=0,
\]
so the entire loss geometry is identical and the gap vanishes identically, yet
at the Bayes-optimal $\phi$ the rates satisfy $\mathrm{EOD}=0.4$ and
$\mathrm{AOD}\approx0.244$.
\end{theorem}

\begin{proof}
Binary labels; two input atoms $x_1,x_2$ with $\eta(x_1)=\Pr(Y{=}1\mid x_1)=0.9$
and $\eta(x_2)=0.1$. Group $a$ places mass $(\tfrac12,\tfrac12)$ on
$(x_1,x_2)$; group $b$ places mass $(\tfrac1{10},\tfrac9{10})$. Take the
one-parameter model $q_\phi(x_1)=\sigma(\phi)$, $q_\phi(x_2)=1-\sigma(\phi)$,
where $q_\phi(x)$ is the predicted probability of $Y=1$ and $\sigma$ is the
logistic function.

\emph{The two atoms contribute the identical function of $\phi$.} Writing
$s=\sigma(\phi)$, the expected negative log-likelihood at $x_1$ is
$-[0.9\log s+0.1\log(1-s)]$, and at $x_2$, where the model predicts $1-s$ and
the truth is $\eta=0.1$, it is $-[0.1\log(1-s)+0.9\log s]$, which is the same
expression. Hence for \emph{any} mixture weights the group loss equals that
common function, so $L_a(\phi)=L_b(\phi)$ for all $\phi$, and consequently
$g_a=g_b$, $H_a=H_b$ and $\Delta_{ab}\equiv0$ with all derivatives zero.

\emph{The Fisher also coincides.} At $x_1$ the score is
$\partial_\phi\log s=1-s$ for $y=1$ and $\partial_\phi\log(1-s)=-s$ for $y=0$,
giving Fisher contribution $s(1-s)^2+(1-s)s^2=s(1-s)$. At $x_2$ the predicted
probability is $1-s$, the scores are $-s$ and $1-s$ respectively, and the
contribution is $(1-s)s^2+s(1-s)^2=s(1-s)$. Both atoms contribute $s(1-s)$, so
$F_a=F_b=s(1-s)$ for any mixture; with $p=1$ and $J=1$, $\Fr_a=\Fr_b$ and
$\eta_{ab}=|H_a-\Fr_a-(H_b-\Fr_b)|=0$.

\emph{Rates differ at the Bayes point.} The model is well specified: at
$\sigma(\phi)=0.9$ it reproduces $\eta$ exactly, so this is the global minimizer
of both group losses. Thresholding at $\tfrac12$ predicts $\hat y=1$ exactly on
$x_1$. Group $a$: $\Pr(Y{=}1)=\tfrac12(0.9)+\tfrac12(0.1)=0.5$ and
$\Pr(\hat y{=}1,Y{=}1)=\tfrac12(0.9)=0.45$, so $\mathrm{TPR}_a=0.9$;
$\Pr(Y{=}0)=0.5$ and $\Pr(\hat y{=}1,Y{=}0)=0.05$, so $\mathrm{FPR}_a=0.1$.
Group $b$: $\Pr(Y{=}1)=\tfrac1{10}(0.9)+\tfrac9{10}(0.1)=0.18$ and
$\Pr(\hat y{=}1,Y{=}1)=0.09$, so $\mathrm{TPR}_b=0.5$; $\Pr(Y{=}0)=0.82$ and
$\Pr(\hat y{=}1,Y{=}0)=0.01$, so $\mathrm{FPR}_b=0.01/0.82\approx0.0122$.
Therefore $\mathrm{EOD}=|0.9-0.5|=0.4$ and
$\mathrm{AOD}=\tfrac12(0.4+|0.1-0.0122|)\approx0.244$.
\end{proof}

Unlike the results of Section~\ref{sec:theory}, Theorem~\ref{thm:eod} needs no slope, defect or
smallness hypothesis: the loss geometry is identical in every order and for
every $\phi$, and the rate disparity is nevertheless $0.4$. The mechanism is
that $\mathrm{TPR}$ conditions on $Y=1$ and so depends on the group base rate,
and in this construction changing the mixture weights changes the
conditional-rate denominators but not the loss geometry, because both input
atoms contribute the identical loss and Fisher for every $\phi$. In particular,
the equal-mixture pair (both groups with mass $(\tfrac12,\tfrac12)$) has
exactly the same loss geometry as the pair above but $\mathrm{EOD}=0$: the
geometry does not determine EOD.

\section{Interpretive boundaries and failure modes}

Each boundary below is now a consequence of a stated result rather than a
caution, and we give the governing result in each case.

\paragraph{Coordinate dependence.}
Raw gradient norms depend on the chosen trainable coordinates, so changing the
adapter family, rank, or parameter scaling can change the trace without a
corresponding change in predictive behavior. Lemma~\ref{lem:projection} makes
this exact: under $\phi=S(\psi)$ with Jacobian $K$, $\Fr_a\mapsto K^\top\Fr_aK$,
so taking $K=cI$ rescales every trace by $c^2$ while leaving the predictor
untouched. The across-group Fisher CV is unaffected by that isotropic case,
since all group traces move together, but it is not invariant in general: with
$\Fr_a=\diag(1,0)$, $\Fr_b=\diag(0,1)$ and $K=\diag(1,2)$ the CV moves from $0$
to $0.6$. We therefore compare subgroup traces only within a fixed
model and parameterization, and note that the inertia of $\Fr_a-\Fr_b$ (its
numbers of positive, negative and zero eigenvalues) and the
generalized eigenvalues of the pencil $(\Fr_a,\Fr_b)$ are the coordinate-free
alternatives.

\paragraph{The trace is a summary, not the governing quantity.}
Theorem~\ref{thm:worstcase} shows the second-order term is governed by
$\|\Dab\|^{\mathcal{H}}_2$, which agrees with $\|\Fr_a-\Fr_b\|^{\mathcal{H}}_2$
only up to the defect $\eta^{\mathcal{H}}_{ab}$; and Theorem~\ref{thm:trace}
shows that matching traces leaves it as large as the common trace:
among positive semidefinite matrices of equal trace $\tau$, the spectral
discrepancy can be exactly $\tau$. Lemma~\ref{lem:trace-decomp} adds that the
empirical trace mixes $\|\hat g_a\|_2^2$ with $\tr\widehat\Sigma_a$, so a trace
penalty can succeed by trading one term against the other; equalizing
$\|\hat g_a\|^2$ does not imply $g_a=g_b$. Any conclusion drawn from trace or CV agreement alone is therefore weaker
than it appears; Section~\ref{sec:frobenius} gives a matrix discrepancy that does bound the
curvature channel's magnitude, though, like any scalar norm, not its sign.

\paragraph{Expected Fisher is not observed-label empirical Fisher.}
The training penalty uses $\nabla\ell_i\nabla\ell_i^\top$ at the observed label.
The post-training diagnostic in Equation~\ref{eq:trace-estimator} instead
averages score gradients over the model's predictive distribution. These
matrices can behave differently \citep{kunstner2019limitations}; we report them
separately and do not use the expected-Fisher interpretation to describe the
quantity optimized during training. Note that Assumption~\textbf{A2} of
Section~\ref{sec:theory} identifies the Gauss--Newton matrix with the
\emph{expected} Fisher only, which is why the defect $\Xi_a$ is carried
explicitly rather than discarded.

\paragraph{Difficulty is not fairness.}
Equation~\ref{eq:correlation} tests whether subgroup ordering agrees with one
task metric; EOD and AOD condition on the target and compare error rates, and
Theorem~\ref{thm:eod} shows that the loss-geometric quantities studied here do
not determine them.
Calibration, demographic parity, worst-group risk and intersectional harms
impose further, separate requirements that no scalar local sensitivity can
certify.

\paragraph{A matched geometry is not a small gap.}
As excess surrogate risks vanish the performance gap tends to
$\gamma_{ab}=|R_a^\star-R_b^\star|$ (Theorem~\ref{thm:bayesfloor}), so exact
geometric matching does not drive it to zero. Note that $\gamma_{ab}$ is
\emph{not} a lower bound: smaller gaps are attainable by degrading the
better-served group (Remark~\ref{rem:not-a-floor}). A reader should therefore not
treat a reduced discrepancy as evidence of improved outcomes, nor an unchanged
gap as evidence that the geometry did not move; the two are only loosely
coupled.

\paragraph{The curvature results are restricted to horizontal additive updates.}
By Proposition~\ref{prop:gauge} and Corollary~\ref{cor:gauge-G},
$\ker J(\phi)$ contains the tangent space of the $\mathrm{GL}(r)$ orbit
$BA=(Be^{tX})(e^{-tX}A)$ and, for $r\le\min(d_{\mathrm{in}},d_{\mathrm{out}})$,
has dimension at least $r^2$ everywhere, so
$\Fr_a-\Fr_b$ always has a null direction and is never positive definite on raw
$(A,B)$ coordinates. Every definiteness hypothesis in Section~\ref{sec:theory}
is therefore stated for $\delta\in\mathcal{H}_\phi$, and
$\mathcal{S}_\varepsilon=\{0\}$ means ``no nonzero horizontal additive update''.
This is a first-order restriction and nothing more:
Remark~\ref{rem:not-a-quotient} gives a LoRA point at the standard $B=0$
initialization where two updates differing by a kernel element produce
second-order gap changes of opposite sign. Nothing in Section~\ref{sec:method} projects onto $\mathcal{H}_\phi$. At a fixed point the reachable Fishers annihilate the gauge directions, but the trace and Frobenius statistics are not constant along a gauge orbit (Appendix~\ref{sec:frobenius}), so the penalty can also move the factors along function-preserving directions; this is not the same as verifying the hypotheses, and the optimizer's raw steps are not covered by the theorems at all.

\paragraph{The curvature results assume slope matching.}
Theorems~\ref{thm:impossible} and \ref{thm:feasibility}(ii), the forced-increase
reading of Theorem~\ref{thm:onlyreachable}(ii), and the gap consequence of trace
equality (Remark~\ref{rem:trace-magnitude}) all require $g_a=g_b$. Without it a
first-order term dominates and the conclusions fail outright
(Remark~\ref{rem:thm1-hyps}). Nothing in Section~\ref{sec:method} enforces slope
matching exactly, so the theory describes a channel the objective addresses, not
a property the trained models are shown to have.

\input{app_exact}

\end{document}

%% file: sec_theory.tex
\section{Reachable Fisher geometry of the subgroup gap}
\label{sec:theory}

We now make precise, and prove, the introduction's claim that different spectral properties of the reachable Fisher answer different questions about a subgroup gap, while a trace answers neither. The argument has four steps: the adapter fixes the reachable directions (Section~\ref{sec:theory-setup}); curvature along them raises two questions, direction and magnitude (Section~\ref{sec:theory-two}); common summaries lose information both need (Section~\ref{sec:theory-onlyreachable}); and all of it concerns surrogate losses (Section~\ref{sec:theory-bayes}). 
Section~\ref{sec:experiments} then derives their computable counterparts, which Sections~\ref{sec:results} and~\ref{sec:mitigation-results} evaluate. Proofs are in Appendix~\ref{app:proofs}.


\subsection{Setup and the master decomposition}
\label{sec:theory-setup}

\textbf{Adapter geometry.} Frozen weights $\theta_0\in\R^P$ are adapted
through a smooth map $T:\R^p\to\R^P$, $p\le P$, so $\theta=T(\phi)$; full
fine-tuning has $T(\phi)=\phi$ (so $J=I$), and for LoRA each adapted weight is
$W=W_0+s\,BA$ with $B\in\R^{d_{\mathrm{out}}\times r}$,
$A\in\R^{r\times d_{\mathrm{in}}}$, a fixed nonzero scale $s$ ($s=2$ in our runs; Appendix~\ref{app:repro}), and $\phi=(A,B)$. 
The Jacobian $J(\phi)=\partial T/\partial\phi$ maps adapter perturbations to
first-order changes of the full model, so locally the adapter moves $\theta$
only within $\operatorname{range}J$.

\textbf{Subgroup losses.} Let $\mathcal{A}$ be a finite set of subgroups with
distributions $\mathcal{D}_a$ over $z=(x,y)$, with input marginal $\mathcal{D}_a^X$, and write
$L_a(\phi)=\E_{\mathcal{D}_a}[\ell(T(\phi);z)]$, $g_a=\nabla_\phi L_a$ and
$H_a=\nabla^2_\phi L_a$. The \emph{signed gap} is $\Delta_{ab}=L_a-L_b$;
exchanging $a$ and $b$ reverses every directional conclusion below.

\textbf{Reachable Fisher.} For $\ell=-\log p_\theta(y\mid x)$ let $F_a$ be the
full-model Fisher of group $a$. Pulling it back through $J$ gives the
\emph{reachable Fisher}:
\begin{equation}
F_a=\E_{x\sim\mathcal{D}_a^X}\E_{y\sim p_\theta(\cdot\mid x)}
\!\big[\nabla_\theta\log p_\theta\,\nabla_\theta\log p_\theta^\top\big],
\qquad
\Fr_a(\phi)=J(\phi)^\top F_a(T(\phi))\,J(\phi).
\label{eq:reachable-fisher}
\end{equation}
$F_a\in\R^{P\times P}$ is never formed; $\Fr_a\in\R^{p\times p}$ is the object
of interest.


\textbf{Assumption A2 (likelihood form).} $p_\theta(y\mid x)$ is an exponential family whose natural parameter is the network output, with $\ell$ its negative log-likelihood, so the generalized Gauss--Newton matrix equals $F_a$ \citep{martens2020new}. A2 covers softmax cross-entropy and fixed-variance Gaussian regression but does \emph{not} give $H_a=\Fr_a$; the smoothness assumption \textbf{A1} needed for the expansion follows below.

\textbf{The horizontal subspace.} The sign certificate below requires positive
definiteness, which raw factorized coordinates cannot provide: they contain
directions that do not change the model to first order (\emph{gauge} directions). For rank-$r$ LoRA with
$s\ne0$ and $r\le\min(d_{\mathrm{in}},d_{\mathrm{out}})$,
Proposition~\ref{prop:gauge} gives, with $r_A=\operatorname{rank}A$ and
$r_B=\operatorname{rank}B$,
\[
\dim\ker J=d_{\mathrm{out}}(r-r_A)+d_{\mathrm{in}}(r-r_B)+r_Ar_B\ \ge\ r^2
\quad\text{at every }\phi,
\]
and the lower bound extends to DoRA (Corollary~\ref{cor:gauge-G}). Hence
$\lambda_{\min}(\Fr_a-\Fr_b)\le0$ everywhere. We therefore restrict every
directional statement to the \emph{horizontal subspace}
$\mathcal{H}_\phi=(\ker J(\phi))^\perp$, with
$\lambda^{\mathcal{H}}_{\min}(S)=\min\{v^\top Sv:v\in\mathcal{H}_\phi,\|v\|=1\}$
and $\|S\|^{\mathcal{H}}_2=\max\{|v^\top Sv|:v\in\mathcal{H}_\phi,\|v\|=1\}$, and
assume $d_{\mathcal{H}}:=\dim\mathcal{H}_\phi\ge1$. Unless noted, all
eigenvalues and norms in this section are these restricted ones.

\textbf{Caveat: a first-order restriction.} Every result concerns straight additive updates $\delta\in\mathcal{H}_\phi$. At the LoRA origin $A=B=0$ the subspace is $\{0\}$ and the statements are vacuous, and at the standard $B=0$ initialization two updates differing by an element of $\ker J$ can move the gap with opposite signs at second order (Remark~\ref{rem:not-a-quotient}); the results do not transfer to an unprojected optimizer.

\textbf{Assumption A1 (smoothness).} Each $L_a$ is $C^3$ (three times continuously differentiable) on the ball $B(\phi,\varepsilon_0)$ of adapter settings within distance $\varepsilon_0$ of $\phi$, with third derivatives bounded by $M$ (Appendix~\ref{app:a1} explains the role of each). It is what bounds the remainder below.

\begin{restatable}[Exact gap expansion]{proposition}{propExpansion}
\label{prop:expansion}
Under \textbf{A1}, for $\|\delta\|\le\varepsilon\le\varepsilon_0$,
\begin{equation}
\Delta_{ab}(\phi+\delta)-\Delta_{ab}(\phi)
=\underbrace{(g_a-g_b)^\top\delta}_{\textnormal{slope}}
+\underbrace{\tfrac12\delta^\top \Dab\,\delta}_{\textnormal{curvature}}
+R_{ab}(\delta),
\quad \Dab:=H_a-H_b,\ \ |R_{ab}|\le\tfrac{M}{3}\|\delta\|^3 .
\label{eq:gap-expansion}
\end{equation}
\end{restatable}

Under \textbf{A2}, $H_a=\Fr_a+\Xi_a$, where the \emph{Hessian--Fisher defect}
$\Xi_a$ collects a reparameterization term (second derivatives of $T$) and a
model-curvature residual, so the curvature matrix splits as
\begin{equation}
\Dab=\underbrace{(\Fr_a-\Fr_b)}_{\text{reachable Fisher}}
+\underbrace{(\Xi_a-\Xi_b)}_{\text{defect}},
\qquad
\eta^{\mathcal{H}}_{ab}=\|\Xi_a-\Xi_b\|^{\mathcal{H}}_2\le\eta_{ab}=\|\Xi_a-\Xi_b\|_2 .
\label{eq:master}
\end{equation}
For LoRA-type adapters, the reparameterization part is bounded by $|s|\,\|\nabla_W\ell_a-\nabla_W\ell_b\|_F$ (Proposition~\ref{prop:defect}), where $\nabla_W\ell_a$ is the gradient at the deployed weight $W=W_0+sBA$ and $|s|$ is the adapter scale fixed above, so this half of the defect is small exactly when the two subgroup gradients agree there. The model-curvature residual is not bounded, so $\eta_{ab}$ must be assumed small or measured. This expansion leads to two different questions: when must the
curvature term have the same sign for every admissible direction, and how large
can it be in the worst case?

\subsection{Two questions: direction and magnitude}
\label{sec:theory-two}
At a point where the subgroup slopes match ($g_a=g_b$), the problem becomes spectral. If $\Dab$ is positive in every horizontal direction, every sufficiently small update increases the signed gap; if it is indefinite, the chosen direction decides the sign, so the restricted minimum eigenvalue answers the directional question. If instead we ask only how large the quadratic change can be, the relevant quantity is the largest absolute eigenvalue, the restricted operator norm.
Table~\ref{tab:two-questions} summarizes the answers. Both describe the curvature term, which governs the gap only under slope
matching: $g_a=g_b$ makes the linear term vanish for every update, not just
one. This is stronger than the one-step constraint $(g_a-g_b)^\top\delta=0$, which removes the linear term only along the chosen $\delta$: a slope can be removed by a single constraint, whereas no nonzero update can eliminate a positive-definite quadratic form. We first state the directional result for the exact Hessian difference.

\begin{table}[!h]
\vspace{-0.5em}
\caption{What controls direction and magnitude after slope matching. All quantities are restricted to $\mathcal{H}_\phi$; without slope matching, the linear term generally dominates small updates.}
\label{tab:two-questions}
\centering
\small
\begin{tabular}{lll}
\toprule
Question & Exact governing quantity & Reachable-Fisher certificate \\
\midrule
Curvature term positive in every direction? &
$\lambda^{\mathcal{H}}_{\min}(\Dab)>0$ &
$\lambda^{\mathcal{H}}_{\min}(\Fr_a-\Fr_b)>\eta^{\mathcal{H}}_{ab}$ \\
Worst-case curvature change? &
$\tfrac12\|\Dab\|^{\mathcal{H}}_2\varepsilon^2$ &
at most $\tfrac12\big(\|\Fr_a-\Fr_b\|^{\mathcal{H}}_2+\eta^{\mathcal{H}}_{ab}\big)\varepsilon^2$ \\
\bottomrule
\end{tabular}
\vspace{-0.5em}
\end{table}

\begin{restatable}[Forced increase on the horizontal subspace]{theorem}{thmImpossible}
\label{thm:impossible}
Write $\rho_{\min}:=\lambda^{\mathcal{H}}_{\min}(\Dab)$ for the smallest eigenvalue of the curvature difference restricted to $\mathcal{H}_\phi$. Assume \textbf{A1}, $g_a=g_b$, and $\rho_{\min}>0$, that is, $\Dab$ is positive definite \emph{on $\mathcal{H}_\phi$}. Then for every nonzero
$\delta\in\mathcal{H}_\phi$ with
$\|\delta\|\le\min(\varepsilon_0,3\rho_{\min}/4M)$, where $\varepsilon_0,M$ are from \textbf{A1} and the second bound is
$+\infty$ when $M=0$,
\[
\Delta_{ab}(\phi+\delta)-\Delta_{ab}(\phi)\;\ge\;\tfrac14\rho_{\min}\|\delta\|^2\;>\;0 .
\]
The signed gap therefore strictly increases for every admissible horizontal
additive update. If in addition $\Delta_{ab}(\phi)\ge0$, the absolute gap
$|\Delta_{ab}|$ strictly widens.
\end{restatable}

The margin $\rho_{\min}$ controls both halves of the statement: it sets how far the conclusion extends, through the step bound $3\rho_{\min}/4M$, and how fast the gap grows, through $\tfrac14\rho_{\min}\|\delta\|^2$. A larger margin buys a larger admissible step and a faster guaranteed increase; the step bound is where the cubic remainder of~\eqref{eq:gap-expansion} still costs less than half the quadratic term (Appendix~\ref{app:exact} reports a measured value). Each assumption is necessary. Slope matching removes the first-order channel;
the horizontal restriction removes parameterization-null directions; and the
sign of the initial gap is needed only to turn signed-gap growth into
absolute-gap widening. Remark~\ref{rem:thm1-hyps} and
Proposition~\ref{prop:gauge} give counterexamples when any assumption is dropped. The exact Hessian difference is generally unavailable;
Corollary~\ref{cor:weyl} replaces it with the reachable-Fisher difference at the
price of the defect tolerance.

\begin{restatable}[Reachable-Fisher sufficient condition]{corollary}{corWeyl}
\label{cor:weyl}
Write $\eta^{\mathcal{H}}_{ab}=\|\Xi_a-\Xi_b\|^{\mathcal{H}}_2\le\eta_{ab}$. If
$\lambda^{\mathcal{H}}_{\min}(\Fr_a-\Fr_b)>\eta^{\mathcal{H}}_{ab}$ then
$\lambda^{\mathcal{H}}_{\min}(\Dab)>0$; if \emph{additionally} $g_a=g_b$, then
Theorem~\ref{thm:impossible} applies with margin $\rho_{\min}\ge\lambda^{\mathcal{H}}_{\min}(\Fr_a-\Fr_b)-\eta^{\mathcal{H}}_{ab}$.
By Proposition~\ref{prop:gauge} the restriction to $\mathcal{H}_\phi$ is
necessary: the unrestricted condition $\lambda_{\min}(\Fr_a-\Fr_b)>\eta_{ab}$
is unsatisfiable for factorized adapters.
\end{restatable}

\emph{Takeaway.} Reachable-Fisher definiteness certifies the sign only when its restricted minimum eigenvalue exceeds the Hessian--Fisher defect; it says nothing about magnitude.

\begin{restatable}[Worst-case growth of the two channels]{theorem}{thmWorstcase}
\label{thm:worstcase}
Over $\delta\in\mathcal{H}_\phi$ with $\|\delta\|\le\varepsilon$, let
$\mathcal{Q}_{ab}(\varepsilon)=\max\tfrac12|\delta^\top\Dab\delta|$ be the
worst-case second-order term and
$\mathcal{G}_{ab}(\varepsilon)=\max|\Delta_{ab}(\phi+\delta)-\Delta_{ab}(\phi)|$
the worst-case gap change. Then
\[
\mathcal{Q}_{ab}(\varepsilon)=\tfrac12\|\Dab\|^{\mathcal{H}}_2\varepsilon^2,
\qquad
\mathcal{G}_{ab}(\varepsilon)\le\|g_a-g_b\|\varepsilon
+\tfrac12\|\Dab\|^{\mathcal{H}}_2\varepsilon^2+\tfrac{M}{3}\varepsilon^3,
\]
the first attained at a dominant eigenvector of $\Dab$ restricted to
$\mathcal{H}_\phi$. If $g_a=g_b$ the bound is tight to third order:
$|\mathcal{G}_{ab}(\varepsilon)-\tfrac12\|\Dab\|^{\mathcal{H}}_2\varepsilon^2|
\le\tfrac{M}{3}\varepsilon^3$.
\end{restatable}

\emph{Takeaway.} By \eqref{eq:master}, $\|\Dab\|^{\mathcal{H}}_2$ and
$\|\Fr_a-\Fr_b\|^{\mathcal{H}}_2$ differ by at most $\eta^{\mathcal{H}}_{ab}$, so
matching reachable Fishers does less than matching Hessians: it leaves a
second-order gap change of up to $\tfrac12\eta_{ab}\varepsilon^2$
(Theorem~\ref{thm:sufficiency}; exact Hessian matching is treated in
Theorem~\ref{thm:feasibility}).

\subsection{What practical summaries can reveal}
\label{sec:theory-onlyreachable}
Two distinct failures separate a practical statistic from the quantities above: it may be computed in the wrong space, or it may scalarize away the directional information that survives. We take them in turn, then give a matrix-level remedy. We write $S\succeq0$ for positive semidefinite $S$. A \emph{Jacobian-blind} statistic is computed from the full-model Fishers $F_a,F_b$ without $J$, such as a norm of $F_a-F_b$ or a spectral summary computed separately for each group. Such a statistic cannot determine reachable geometry because it does not know how the adapter Jacobian is oriented: a large full-model difference may lie entirely outside $\operatorname{range}J$, while two full-model Fishers with identical spectra may interact differently with it.

\begin{restatable}[Full-model summaries do not determine reachable geometry]{theorem}{thmOnlyreachable}
\label{thm:onlyreachable}
\emph{(i)} For every $c>0$ there exist $F_a,F_b\succeq0$ and a Jacobian
$J$ with $\|F_a-F_b\|_2=c$ yet $\Fr_a=\Fr_b$.
\emph{(ii)} There exist $F_a,F_b\succeq0$ with identical spectra such that
$\Fr_a-\Fr_b$ is positive definite on $\mathcal{H}_\phi$.
\end{restatable}

A full-model audit can thus report a large disparity of which nothing survives
the projection, or miss a reachable difference that meets the premise of
Corollary~\ref{cor:weyl}; without $J$, no such statistic identifies the
reachable projection.

Moving to the correct space is not enough. A trace records the sum of eigenvalues, not how curvature is distributed across directions, so two reachable Fishers can have equal traces and very different geometry.

\begin{restatable}[Limits of trace matching]{theorem}{thmTrace}
\label{thm:trace}
For $\Fr_a,\Fr_b\succeq0$ of size $p\ge2$:
\emph{(i) Necessary condition.} $\Fr_a=\Fr_b$ implies $\tr\Fr_a=\tr\Fr_b$; the
converse fails.
\emph{(ii) General failure.} For any $\tau>0$,
$\sup\{\|\Fr_a-\Fr_b\|_2:\tr\Fr_a=\tr\Fr_b=\tau\}=\tau$, attained by an
indefinite difference: equal traces permit the largest spectral mismatch the
trace budget allows.
\emph{(iii) Structured positive result.} If $\Fr_a-\Fr_b\succeq0$ (Loewner
ordering), then $\|\Fr_a-\Fr_b\|_2\le\tr\Fr_a-\tr\Fr_b$, so equal traces force
equal matrices.
\end{restatable}

\begin{restatable}[A small normalized trace gap excludes the sufficient regime]{corollary}{corTraceExclude}
\label{cor:trace-exclude}
If $\lambda^{\mathcal{H}}_{\min}(\Fr_a-\Fr_b)>0$ then
$\tr\Fr_a-\tr\Fr_b\ge d_{\mathcal{H}}\,\lambda^{\mathcal{H}}_{\min}(\Fr_a-\Fr_b)$.
Hence $|\tr\Fr_a-\tr\Fr_b|/d_{\mathcal{H}}\le\eta^{\mathcal{H}}_{ab}$ excludes the
hypothesis of Corollary~\ref{cor:weyl} in both orientations.
\end{restatable}

Trace therefore has two asymmetric uses. A sufficiently small normalized trace
gap rules out the sufficient positive-definiteness regime; without Loewner
ordering it cannot certify matrix matching or bound an indefinite mismatch. The
observed-label trace (of the empirical reachable Fisher
$\widehat F^{\mathrm{R},\mathrm{obs}}_a$, the average of
$\nabla_\phi\ell_i\nabla_\phi\ell_i^\top$ over group $a$'s examples, with mean
gradient $\hat g_a$ and gradient covariance $\widehat\Sigma_a$) also mixes the
channels:
$\tr\widehat F^{\mathrm{R},\mathrm{obs}}_a=\|\hat g_a\|^2+\tr\widehat\Sigma_a$,
a mean-gradient term that is related to but does not determine $g_a-g_b$, plus
within-group dispersion (Lemma~\ref{lem:trace-decomp}). A matrix discrepancy
avoids the second failure, and can be evaluated without forming a $p\times p$
matrix.

\begin{restatable}[Matrix-free Gram identity]{proposition}{propGram}
\label{prop:gram}
For groups with $n_a$ and $n_b$ examples, any factors
$G_a\in\R^{m_a\times p}$, $G_b\in\R^{m_b\times p}$ (one or more rows per
example) and
$\widehat S_c=\frac1{n_c}G_c^\top G_c$,
\begin{equation}
\big\|\widehat S_a-\widehat S_b\big\|_F^2
=\frac{\|G_aG_a^\top\|_F^2}{n_a^2}
-\frac{2\|G_aG_b^\top\|_F^2}{n_an_b}
+\frac{\|G_bG_b^\top\|_F^2}{n_b^2},
\label{eq:gram}
\end{equation}
computable in $O\big((m_a^2+m_am_b+m_b^2)p\big)$ time with no $p\times p$ matrix
formed.
\end{restatable}

Per-example gradients supply such factors for both empirical reachable Fishers
(Appendix~\ref{sec:method}), which is what lets Section~\ref{sec:mitigation-results}
audit trained models at the matrix level. The Frobenius discrepancy upper-bounds the operator
norm, and so bounds the curvature channel up to defect and estimator error
(Corollary~\ref{cor:closure}); like any norm, it does not certify definiteness.

\subsection{Scope: surrogate-loss gaps versus predictive-performance gaps}
\label{sec:theory-bayes}

What do these loss-geometric results imply for predictive-performance gaps? Less than a direct translation: surrogate-loss differences do not determine balanced accuracy, mean absolute error, or rate-based criteria. The result below gives the bridge available under classification calibration and near-Bayes convergence.
For \emph{binary}
classification let $R_a$ be group-$a$ $0$--$1$ risk, $R_a^\star$ its Bayes risk,
$\ell$ classification-calibrated (driving its excess risk to zero drives the
excess $0$--$1$ risk to zero) with calibration function $\psi$ (which converts excess surrogate risk into
excess $0$--$1$ risk), and
$\epsilon_a=L_a-L_a^\star$, where $L_a^\star$ is the group-$a$ Bayes surrogate
risk over all measurable predictors, not the minimum over the PEFT class
\citep{bartlett2006convexity}.

\begin{restatable}[Near-Bayes convergence, not a floor]{theorem}{thmBayesfloor}
\label{thm:bayesfloor}
For every $\phi$, with $\gamma_{ab}=|R_a^\star-R_b^\star|$,
\[
\big|\,|R_a-R_b|-\gamma_{ab}\,\big|
\;\le\;\max\{\psi^{-1}(\epsilon_a),\psi^{-1}(\epsilon_b)\} ,
\]
and hence
\[
|R_a-R_b|\;\ge\;\max\big\{0,\;\gamma_{ab}-\max\nolimits_c\psi^{-1}(\epsilon_c)\big\}.
\]
Consequently $|R_a-R_b|\to\gamma_{ab}$ along any sequence of predictors whose
excess surrogate risks both tend to zero.
\end{restatable}

Near Bayes optimality the risk gap approaches the Bayes-risk difference, which is
a limit, not a floor (Remark~\ref{rem:not-a-floor}). The result extends to
balanced error and multi-group maxima (Remark~\ref{rem:balanced}), not to
multiclass balanced accuracy or MAE; Appendix~\ref{sec:theory-scope} gives the
full scope.

%% file: sec_mitigation.tex
\input{tables/mitigation_verified_e1}

\section{Does reachable-Fisher trace matching narrow subgroup gaps?}
\label{sec:mitigation-results}

\paragraph{Arms and baseline.} We compare the trace-matching objective of
Section~\ref{sec:objective} with \emph{group-balanced task-loss training}: the same objective without the penalty, using the same sampler, schedule, seed, and initialization.  The sampler draws equal counts from every constrained group and the task loss is the unweighted mean over the batch, so both arms minimize the group-balanced risk (uniform $p_a$); pooled utility is nevertheless reported over all test examples, so it is a population quantity. Because the two arms share the sampler and the loss, their difference isolates the curvature penalty rather than the effect of oversampling.

\paragraph{Protocol.} Three image tasks (Fitzpatrick17k, UTKFace, and the face
age-regression task of FHIBE) are crossed with ConvNeXt, ViT and SigLIP~2 encoders and
LoRA, DoRA and PiSSA adapters; a tabular task (diabetes readmission) is crossed
with the same adapters on an FT-Transformer. This gives \mvPn{} cells, each
trained once per arm. Batch size is tuned per cell as a hyperparameter (64--1024); both arms of a cell share it. The retained checkpoint is selected by the pooled validation metric, never by a subgroup gap (Appendix~\ref{sec:method}). The primary outcome is
the test-split best--worst subgroup gap $G$ in balanced accuracy or MAE,
summarized as $R_{\mathrm{gap}}=100\,(G_{\mathrm P}-G_{\mathrm H})/G_{\mathrm P}$
(P plain, H Huber-penalized trace matching; positive favours trace matching). The
secondary outcome is pooled utility over all test examples. EOD and AOD are in Appendix~\ref{app:mitigation}.

\paragraph{Results.} Cells share tasks, encoders and adapters and are
single-seed, so all statistics are descriptive.
\begin{itemize}[leftmargin=*,itemsep=0pt,topsep=0pt]
    \item \emph{Gap.} Trace matching narrows the gap in all \mvPn{} cells, with
    median $R_{\mathrm{gap}}=\mvPmed\%$ (task-blocked bootstrap
    $[\mvBlo,\,\mvBhi]\%$). The effect is present in every task and adapter:
    medians are \mvFitzmed\% (Fitzpatrick17k), \mvUTKmed\% (UTKFace),
    \mvFacemed\% (FHIBE face) and \mvDiabmed\% (diabetes), and \mvLoRAmed\%,
    \mvDoRAmed\% and \mvPiSSAmed\% for LoRA, DoRA and PiSSA.
    \item \emph{Utility.} \mvQboth{} cells narrow the gap without a loss of
    pooled utility and \mvQgaponly{} at a small cost
    (Figure~\ref{fig:mit-winrates}); median changes are
    $+\mvUclsmed$~pp balanced accuracy and $\mvUregmed$~years MAE. The largest
    classification loss is $\mvUclsworst$~pp (\mvUclsworstcell), where the gap
    falls from \mvUclsworstgp{} to \mvUclsworstgh{}.
    \item \emph{Surrogate-loss gap.} The gap in the training loss, the quantity
    the theorems concern, narrows in 27 of 30 cells (median $-22.8\%$,
    bootstrap $[-32.2,-16.1]\%$) and moves with the performance gap
    (Spearman $\rho=+0.57$).
    \item \emph{Matrix audit.} On identical held-out group-balanced batches (Appendix~\ref{sec:method}), trace matching lowers the normalized trace gap in 20 of 30 cells (median $-12.0\%$) and the observed-label empirical operator discrepancy, which Theorem~\ref{thm:trace}(ii) says the trace cannot bound, in 23 (median $-10.2\%$, bootstrap $[-18.5,-3.5]\%$). The primary unbiased squared-Frobenius statistic falls in only 16 (median $-4.3\%$, interval spanning zero): of the 20 cells whose trace gap fell, it rose in 9. Using the plug-in norms, the median operator-to-Frobenius ratio is $0.72$, so that bound is loose by about a third. Per-cell values are in Table~\ref{tab:mit-audit}.
    \item \emph{Exploratory association.} Cells whose expected-Fisher CV fell more also
    narrowed the gap more (Spearman $\rho=+0.23$ on relative changes; the CV
    fell in 25 of 30 cells).
\end{itemize}
\textbf{Reading the results against the theory.} Theorem~\ref{thm:impossible} motivates the penalty: under slope matching, a reachable-Fisher difference definite above the defect forces the signed gap up, so shrinking it is the lever measurable at scale. A trace penalty reaches that condition only under Loewner ordering (Theorem~\ref{thm:trace} (iii)), which we do not measure, so these outcomes are compatible with the theory without being predicted by it. Appendix~\ref{app:scope} compares adapter and full-model penalty scope.

%% file: tables/mitigation_verified_e1.tex
\newcommand{\mvPn}{30}
\newcommand{\mvPwins}{30}
\newcommand{\mvPpct}{100}
\newcommand{\mvPmed}{17.2}
\newcommand{\mvPmean}{20.2}
\newcommand{\mvPlo}{12}
\newcommand{\mvPhi}{21}
\newcommand{\mvBlo}{11}
\newcommand{\mvBhi}{29}
\newcommand{\mvPsignp}{1.86e-09}
\newcommand{\mvPwilp}{1.86e-09}
\newcommand{\mvQboth}{23}
\newcommand{\mvQgaponly}{7}
\newcommand{\mvQutilonly}{0}
\newcommand{\mvQneither}{0}
\newcommand{\mvPeodn}{12}
\newcommand{\mvPeodw}{7}
\newcommand{\mvPaodw}{6}
\newcommand{\mvPaodn}{12}
\newcommand{\mvUclsworst}{-2.4}
\newcommand{\mvUclsworstcell}{Fitzpatrick17k/ViT/PiSSA}
\newcommand{\mvUregworst}{+0.26}
\newcommand{\mvUregworstcell}{FHIBE face/SigLIP~2/LoRA}
\newcommand{\mvAn}{70}
\newcommand{\mvAwins}{44}
\newcommand{\mvApct}{63}
\newcommand{\mvAmed}{6.4}
\newcommand{\mvUclsmed}{0.5}
\newcommand{\mvUclsbig}{0}
\newcommand{\mvUregmed}{-0.09}
\newcommand{\mvUclsworstgp}{0.260}
\newcommand{\mvUclsworstgh}{0.139}
\newcommand{\mvLoRAw}{10 of 10}
\newcommand{\mvLoRAmed}{17.2}
\newcommand{\mvDoRAw}{10 of 10}
\newcommand{\mvDoRAmed}{17.3}
\newcommand{\mvPiSSAw}{10 of 10}
\newcommand{\mvPiSSAmed}{15.5}
\newcommand{\mvFitzw}{9 of 9}
\newcommand{\mvFitzmed}{25.7}
\newcommand{\mvUTKw}{9 of 9}
\newcommand{\mvUTKmed}{14.3}
\newcommand{\mvFacew}{9 of 9}
\newcommand{\mvFacemed}{11.1}
\newcommand{\mvDiabw}{3 of 3}
\newcommand{\mvDiabmed}{23.0}
\newcommand{\mvLOTOpctlo}{100}
\newcommand{\mvLOTOpcthi}{100}
\newcommand{\mvLOTOmedlo}{15.2}
\newcommand{\mvLOTOmedhi}{19.3}

%% file: app_theory.tex
\section{Extended theory}
\label{app:theory}

This appendix collects the discussion that accompanies Section~\ref{sec:theory}:
what the restriction to $\mathcal{H}_\phi$ means, why it is forced, the remarks
that follow each result, and a full statement of scope. Proofs follow in
Appendix~\ref{app:proofs}.

\subsection{Reading assumption A1}
\label{app:a1}

Assumption \textbf{A1} makes the second-order analysis of
Section~\ref{sec:theory} precise. Each of its three parts has one job.

\textbf{The ball $B(\phi,\varepsilon_0)$.} This is the set
$\{\phi'\in\R^p:\|\phi'-\phi\|\le\varepsilon_0\}$ of adapter settings within
Euclidean distance $\varepsilon_0$ of the current point $\phi$. All results are
local: they concern updates $\delta$ with $\|\delta\|\le\varepsilon\le\varepsilon_0$,
so that $\phi+\delta$ stays inside the ball where the smoothness and the bound
$M$ hold. The value of $\varepsilon_0$ need not be known; it only has to exist.

\textbf{$C^3$ smoothness.} Each subgroup loss $L_a$ is three times continuously
differentiable on the ball, so its gradient $g_a$, Hessian $H_a$ and third
derivative exist and vary continuously. This is what licenses a Taylor expansion
to second order with an explicit third-order remainder. It holds for smooth
activations and losses (for example $\tanh$, GELU and softmax cross-entropy);
for piecewise-linear activations such as ReLU it holds only away from their
kinks, so for such networks \textbf{A1} is an idealization.

\textbf{The constant $M$.} $M$ bounds the third derivative on the ball:
$|D^3L_a(\phi')[\delta,\delta,\delta]|\le M\|\delta\|^3$ for all $\phi'$ in the
ball and all $\delta$. It measures how fast the curvature itself changes. By
Taylor's theorem with Lagrange remainder, each $L_a$ differs from its quadratic
model by at most $\tfrac{M}{6}\|\delta\|^3$, so the gap $\Delta_{ab}=L_a-L_b$
differs from its slope-plus-curvature model by at most
$\tfrac{M}{3}\|\delta\|^3$ (Proposition~\ref{prop:expansion}). A small $M$ means
the quadratic model stays accurate over larger steps. $M$ also sets the step
size in Theorem~\ref{thm:impossible}: for $\|\delta\|\le3\rho_{\min}/(4M)$ the
cubic remainder is too small to overturn the quadratic increase. $M$ is not
estimated in our experiments; Appendix~\ref{app:exact} checks on a small network
that the remainder indeed scales as $\|\delta\|^3$.

\subsection{The horizontal restriction}

$\mathcal{H}_\phi=(\ker J(\phi))^\perp$ is a Euclidean choice of representative
for the first-order quotient by the \emph{entire} first-order kernel
$\ker J(\phi)$, which contains the $\mathrm{GL}(r)$ gauge tangent and can be
strictly larger at rank-deficient points (Proposition~\ref{prop:gauge}(ii)).

\textbf{What this restriction does and does not mean.} $\mathcal{H}_\phi$ is a
first-order object. Every curvature statement of Section~\ref{sec:theory} concerns \emph{straight additive
updates} $\delta\in\mathcal{H}_\phi$, and we do \emph{not} identify such a
$\delta$ with an equivalence class of updates: Remark~\ref{rem:not-a-quotient}
exhibits a LoRA point at the standard $B=0$ initialization where two raw updates
differing by an element of $\ker J$, and hence representing the same first-order
predictor change, produce gap changes of \emph{opposite sign} at second
order. A conclusion of the form ``$\mathcal{S}_\varepsilon=\{0\}$'' therefore
means \emph{no nonzero horizontal additive update}, not ``no nonzero quotient
class''. Building a genuine second-order quotient would require a regular
full-rank stratum and a smooth local horizontal section, machinery we do not
develop and which in any case does not cover the $B=0$ initialization our
experiments use.

Proposition~\ref{prop:gauge} is the reason every curvature statement in Section~\ref{sec:theory} is restricted
to $\mathcal{H}_\phi$. Read on the full coordinate space $\R^p$, the hypotheses
of Theorem~\ref{thm:impossible} and Corollary~\ref{cor:weyl} are \emph{empty}
for exactly the adapters this paper is about: the $\mathrm{GL}(r)$ symmetry
$BA=(Be^{tX})(e^{-tX}A)$ guarantees a null direction, so no reachable Fisher
difference is ever positive definite and no $\Dab$ is either once slopes match.
Restricting to horizontal additive updates removes those directions and restores
satisfiability: on $\mathcal{H}_\phi$ the definiteness conditions can hold, and
we verify this numerically in the supplement. The obstruction is not a
technicality about degenerate points: whenever $r\le\min(d_{\mathrm{in}},d_{\mathrm{out}})$,
which holds for every adapter we train, \eqref{eq:kerbound} gives it at every
$\phi$. The repair, however, is correspondingly modest: it buys a well-posed class
of updates to quantify over, not a quotient manifold on which to do second-order
geometry (Remark~\ref{rem:not-a-quotient}).

\subsection{Remarks on the main results}

Corollary~\ref{cor:weyl} is the only route from reachable Fisher geometry to the
forced-increase conclusion, and it is conditional on three things: slope
matching, the defect being smaller than the reachable gap \emph{on
$\mathcal{H}_\phi$}, and the restriction to $\mathcal{H}_\phi$ itself, without
which Proposition~\ref{prop:gauge} makes it vacuous. We do not claim $\Fr_a-\Fr_b$ is the \emph{only} lever
on $\Dab$: by Proposition~\ref{prop:defect} the defect $\Xi_a-\Xi_b$ is
another, and a method that changed it would change $\Dab$ too. We claim it is
the lever that is measurable at scale and that the estimator of
Section~\ref{sec:frobenius} acts on.

Minimizing the worst-case \emph{second-order} term is thus exactly minimizing
$\|\Dab\|^{\mathcal{H}}_2$, which by Proposition~\ref{prop:defect} agrees with
minimizing $\|\Fr_a-\Fr_b\|^{\mathcal{H}}_2$ to within $\eta^{\mathcal{H}}_{ab}$.
This is an additive approximation, not an equivalence, since the defect varies with
$\phi$ and can partly cancel the Fisher term.

Part (ii) offers one concrete explanation for why adapter-robustness audits
conflict \citep{ding2024fairness, sukumaran2024fairlora, li2024flat}: two
studies measuring full-model sharpness on models whose reachable geometries
differ as in (ii) record identical curvature. We do not claim this is the cause
of any particular disagreement in that literature.

Three cautions. The extremal difference $\diag(\tau,-\tau,0,\dots)$ is
\emph{indefinite}, so Theorem~\ref{thm:impossible} does not apply to it and no
forced-increase reading is available there. $\mathcal{Q}_{ab}$ is governed by
$\|\Dab\|^{\mathcal{H}}_2$, not by $\|\Fr_a-\Fr_b\|_2$, which is why the
consequence carries $\eta^{\mathcal{H}}_{ab}$. And $\mathcal{Q}_{ab}$ is only the
quadratic channel: the displayed bound is on the gap itself only because slope
matching removes the linear term. Combined with the
decomposition
$\tr\widehat F^{\mathrm{R},\mathrm{obs}}_a=\|\hat g_a\|_2^2+\tr\widehat\Sigma_a$
(Lemma~\ref{lem:trace-decomp}), which shows a trace penalty mixes a mean-gradient
term (related to, but not determining, the slope channel $g_a-g_b$) with a
within-group dispersion term, this is a concrete
deficiency of the objective implemented in Section~\ref{sec:method}.

Three terms separate the computed penalty from the gap, and none is cosmetic.
The \emph{slope} term cannot be dropped: for a fixed-variance Gaussian model at a
point where $\Fr_a=\Fr_b=H_a=H_b$ and
$\widehat{\mathcal{P}}=\zeta_c=\eta_{ab}=M=0$ but $g_a-g_b=1$, the exact
worst-case gap change is $\varepsilon$, while a bound omitting the slope term
would give $0$. The \emph{defect} $\eta^{\mathcal{H}}_{ab}$ is bounded only in
part (Proposition~\ref{prop:defect}). And the \emph{estimator} terms $\zeta_c$
are not sampling noise alone: the observed-label empirical Fisher is in general
neither the model-expected Fisher nor an unbiased estimate of it
\citep{kunstner2019limitations}, so $\zeta_c$ carries a systematic component
that does not vanish with batch size. This is why the Frobenius discrepancy of
Section~\ref{sec:frobenius} is used as a matrix-level audit and not presented as
a gap guarantee.

\subsection{Scope: what is and is not established}
\label{sec:theory-scope}

\textbf{Established, unconditionally.} The projection identity and invariance
classification (Lemma~\ref{lem:projection}); the gauge redundancy of factorized
adapters (Proposition~\ref{prop:gauge}); the trace extremum
and the ordered-case control by the trace gap (Theorem~\ref{thm:trace}) and the
decomposition (Lemma~\ref{lem:trace-decomp}); the
Gram identity (Proposition~\ref{prop:gram}); the matrix constructions of
Theorem~\ref{thm:onlyreachable}; and, for rate-parity criteria outside our
scope, Theorem~\ref{thm:eod} in Appendix~\ref{app:eod}.

\textbf{Established, conditionally.} For additive updates
$\delta\in\mathcal{H}_\phi$, under slope matching $g_a=g_b$ and a $\Dab$ positive
definite on $\mathcal{H}_\phi$, every such small nonzero update increases the signed
gap (Theorem~\ref{thm:impossible}) and the admissible set is $\{0\}$
(Theorem~\ref{thm:feasibility}(ii)); these become statements about the reachable
Fisher only through Corollary~\ref{cor:weyl}, which needs
$\lambda^{\mathcal{H}}_{\min}(\Fr_a-\Fr_b)>\eta^{\mathcal{H}}_{ab}$. Worst-case
second-order growth is $\tfrac12\|\Dab\|^{\mathcal{H}}_2\varepsilon^2$
(Theorem~\ref{thm:worstcase}), and the Frobenius discrepancy bounds that channel up
to $\zeta_a+\zeta_b+\eta^{\mathcal{H}}_{ab}$, not the gap
(Corollary~\ref{cor:closure}).

\textbf{Not established.} That the \emph{definiteness and forced-increase}
claims have content on raw LoRA/DoRA coordinates; by
Proposition~\ref{prop:gauge} and Corollary~\ref{cor:gauge-G} they do not.
(Theorems~\ref{thm:worstcase} and \ref{thm:sufficiency} remain meaningful there;
it is specifically positive definiteness that is unattainable.) That
$\mathcal{H}_\phi$ is a second-order quotient, or that a horizontal result
transfers to other representatives of the same first-order class, or to an
unprojected optimizer (Remark~\ref{rem:not-a-quotient}). That
$\Fr_a-\Fr_b$ is the \emph{only} lever on $\Dab$; the defect is another and we
bound only its reparameterization half (Proposition~\ref{prop:defect}). That
$\eta^{\mathcal{H}}_{ab}$ or $\zeta_c$ is small for deep networks, or that
either was measured in our runs. That reducing $\widehat{\mathcal{P}}$ reduces
the \emph{loss} gap, let alone a \emph{performance} gap. That $\gamma_{ab}$ is a
floor (Remark~\ref{rem:not-a-floor}). That the defect bound covers DoRA, one
of the three adapters in Section~\ref{sec:mitigation-results}. That the trace
penalty narrows gaps reliably: Section~\ref{sec:mitigation-results} finds a
modest, heterogeneous effect that is not statistically established. That any scalar norm of $\Fr_a-\Fr_b$ (trace,
Frobenius or spectral) certifies the definiteness condition of
Corollary~\ref{cor:weyl}. That Theorem~\ref{thm:bayesfloor} covers multiclass
balanced accuracy or MAE; it is stated for binary $0$--$1$ risk and, by
Remark~\ref{rem:balanced}, binary balanced error.

\textbf{The claim we defend.} \emph{For straight additive updates constrained to
a horizontal complement of the local Jacobian kernel, and once the slope channel
is controlled: if the reachable-Fisher difference is positive definite on that
complement with restricted minimum eigenvalue above the Hessian--Fisher defect,
every sufficiently small nonzero such update increases the signed subgroup loss
gap. Separately, the worst-case size of the curvature channel is an operator
norm; trace equality cannot control it, and the Frobenius discrepancy bounds it,
at a factor-$m$ cost and up to an estimator gap we do not bound. Neither scalar
certifies the definiteness condition. Whether the curvature channel matters for a
measured performance gap is not established here, and our own measurements do
not show that it does.}

\textbf{Open.} Finite-sample error in $\widehat F^{\mathrm{R}}_a$ for small groups;
whether a matching penalty run to convergence reaches a matched point;
non-smooth adapters; bounding $\mathcal{E}_a-\mathcal{E}_b$; and distribution
shift, where the relevant object is the reachable curvature under the target
rather than the source distribution.

%% file: theory_proofs.tex

\section{Additional results and proofs}
\label{app:proofs}

Statements that appear in Section~\ref{sec:theory} are restated here with their
original numbers; results that appear only in this appendix are stated in full
where they are proved. Where a hypothesis is load-bearing we give the
counterexample that shows it cannot be dropped. All constructions, the
positive ones and the counterexamples alike, are checked to machine precision by
\texttt{verify\_theory.py}, included in the supplementary material at the root of
the submission archive, which is written so that it fails if any statement is
strengthened back to a form we had to retract; running \texttt{python
verify\_theory.py} prints one line per check and a final verdict.

\subsection{The gap expansion and the Hessian--Fisher defect}

\propExpansion*

\begin{proof}
Apply Taylor's theorem with Lagrange remainder to $L_a$ and to $L_b$ separately
at $\phi$: for $c\in\{a,b\}$,
$L_c(\phi+\delta)=L_c(\phi)+g_c^\top\delta+\frac12\delta^\top H_c\delta+r_c$
with $|r_c|\le\frac{M}{6}\|\delta\|^3$ by \textbf{A1}. Subtracting the two
expansions and setting $R_{ab}=r_a-r_b$ gives \eqref{eq:gap-expansion}, with
$|R_{ab}|\le\frac{M}{6}\|\delta\|^3+\frac{M}{6}\|\delta\|^3=\frac{M}{3}\|\delta\|^3$.
\end{proof}

\begin{restatable}[Hessian--Fisher defect]{proposition}{propDefect}
\label{prop:defect}
Under \textbf{A2}, $H_a=\Fr_a+\Xi_a$ with
\[
\Xi_a=\underbrace{\sum\nolimits_{k}[\nabla_\theta\ell_a]_k\nabla^2_\phi T_k}_{\mathcal{T}_a\ \textnormal{(reparameterization)}}
\;+\;\underbrace{J^\top(\nabla^2_\theta\ell_a-F_a)J}_{\mathcal{E}_a\ \textnormal{(model curvature)}},
\qquad
\Dab=(\Fr_a-\Fr_b)+(\Xi_a-\Xi_b).
\]
Write $\eta_{ab}=\|\Xi_a-\Xi_b\|_2$. For LoRA, rsLoRA and PiSSA,
$T(\phi)=\theta_0+s\,BA$ with fixed scale $s$ ($s=\alpha/r$ for LoRA and PiSSA,
$s=\alpha/\sqrt r$ for rsLoRA); let $\nabla_W\ell_a$ denote the gradient with
respect to the \emph{deployed} weight $W=W_0+sBA$, i.e.\ before $s$ is absorbed.
Then, for a single adapted block,
\[
\|\mathcal{T}_a-\mathcal{T}_b\|_2\;\le\;|s|\,\|\nabla_W\ell_a-\nabla_W\ell_b\|_F ,
\]
and with several blocks the bound holds with the maximum of the right-hand side
over blocks.
\end{restatable}

We bound only the reparameterization half of the defect. The model-curvature
residual $\mathcal{E}_a-\mathcal{E}_b$ is \emph{not} bounded here: it vanishes
for models linear in $\theta$ and is small near interpolation, but for a deep
network we treat $\eta_{ab}$ as a quantity to be assumed small or measured, not
one we have shown to be small. For DoRA the magnitude--direction split
contributes a further term of the same type which we also do not bound. This is
the principal gap between the geometry we can compute, $\Fr_a-\Fr_b$, and the
geometry that governs the expansion, $\Dab$.

\begin{proof}
Differentiating $L_a(\phi)=\ell_a(T(\phi))$ twice by the chain rule gives
\[
H_a=J^\top\big(\nabla^2_\theta\ell_a\big)J
+\sum\nolimits_{k=1}^{P}\big[\nabla_\theta\ell_a\big]_k\nabla^2_\phi T_k .
\]
Adding and subtracting $J^\top F_aJ=\Fr_a$ inside the first term gives
$H_a=\Fr_a+\mathcal{T}_a+\mathcal{E}_a$ with $\mathcal{T}_a,\mathcal{E}_a$ as
stated, where \textbf{A2} is used only to identify $F_a$ with the generalized
Gauss--Newton matrix so that $\mathcal{E}_a=J^\top(\nabla^2_\theta\ell_a-F_a)J$
is exactly the model-curvature residual. Subtracting the identity for $b$ gives
the expression for $\Dab$.

For the LoRA-family bound, fix one adapted block with
$T(\phi)=\theta_0+s\,BA$, $W=W_0+sBA$. The map is bilinear in $\phi=(A,B)$, and
its second differential in direction $(\dot A,\dot B)$ is $2s\,\dot B\dot A$, so
\[
\mathcal{T}_a[(\dot A,\dot B),(\dot A,\dot B)]
=2s\,\langle\nabla_W\ell_a,\dot B\dot A\rangle ,
\]
with $\nabla_W\ell_a$ the gradient with respect to the deployed weight. For
$\|\dot A\|_F^2+\|\dot B\|_F^2=1$, submultiplicativity and
$2\|\dot B\|_F\|\dot A\|_F\le\|\dot A\|_F^2+\|\dot B\|_F^2=1$ give
$\|\dot B\dot A\|_F\le\tfrac12$, so Cauchy--Schwarz yields
$|(\mathcal{T}_a-\mathcal{T}_b)[(\dot A,\dot B),(\dot A,\dot B)]|
\le|s|\,\|\nabla_W\ell_a-\nabla_W\ell_b\|_F$; since
$\mathcal{T}_a-\mathcal{T}_b$ is symmetric this is the stated operator-norm
bound. The constant is $|s|$ for every member of the family; only the value of
$s$ differs ($\alpha/r$ for LoRA and PiSSA, which differ only in initialization,
$\alpha/\sqrt r$ for rsLoRA). If instead the gradient is taken with respect to
the unscaled product $\Delta W=BA$, then $\nabla_{\Delta W}\ell=s\nabla_W\ell$ and
the bound reads $\|\nabla_{\Delta W}\ell_a-\nabla_{\Delta W}\ell_b\|_F$ with
constant $1$. With several blocks, $\mathcal{T}_a-\mathcal{T}_b$ is block
diagonal across blocks (each block's $T$ depends only on its own factors), so
its operator norm is the maximum of the blockwise bounds.
\end{proof}

\begin{remark}[What the bound does and does not cover]
\label{rem:defect-scope}
The bound controls $\mathcal{T}_a-\mathcal{T}_b$ only. The residual difference
$\mathcal{E}_a-\mathcal{E}_b$ is not bounded: it vanishes when the network is
linear in $\theta$, and is small near interpolation because
$\nabla^2_\theta\ell_a-F_a$ carries the factor $\nabla_f\ell_a$, but for a deep
network at a general point we have no bound. Consequently $\eta_{ab}$ is a
quantity to be assumed small or measured, not one shown to be small, and every
statement converting between $\Dab$ and $\Fr_a-\Fr_b$ carries it explicitly.
For DoRA the magnitude--direction reparameterization contributes a further term
of the same type, which we also do not bound; this matters because DoRA is one of
the three adapters in Section~\ref{sec:mitigation-results}, whereas the bound does
cover the LoRA and PiSSA runs.
\end{remark}

\begin{restatable}[Gauge redundancy of factorized adapters]{proposition}{propGauge}
\label{prop:gauge}
Let $T(\phi)=\theta_0+s\,BA$ with $\phi=(A,B)$, $B\in\R^{d_{\mathrm{out}}\times r}$,
$A\in\R^{r\times d_{\mathrm{in}}}$ and fixed scale $s\ne0$ (the LoRA, rsLoRA and
PiSSA parameterization), and assume $1\le r\le\min(d_{\mathrm{in}},d_{\mathrm{out}})$.
Write $r_A=\operatorname{rank}A$, $r_B=\operatorname{rank}B$.
\begin{enumerate}[label=(\roman*),leftmargin=*,itemsep=1pt,topsep=2pt]
\item \emph{Gauge tangent.} For every $X\in\R^{r\times r}$ the curve
$A(t)=e^{-tX}A$, $B(t)=Be^{tX}$ satisfies $B(t)A(t)=BA$ for all $t$, and its
tangent $v=(-XA,BX)$ lies in $\ker J(\phi)$. The gauge tangent space
$\{(-XA,BX)\}$ has dimension exactly $r^2-(r-r_A)(r-r_B)$.
\item \emph{Kernel.} Exactly,
\begin{equation}
\dim\ker J(\phi)=d_{\mathrm{out}}(r-r_A)+d_{\mathrm{in}}(r-r_B)+r_Ar_B\;\ge\;r^2 ,
\label{eq:kerbound}
\end{equation}
and $\dim\ker J(\phi)$ exceeds the gauge tangent dimension by exactly
$(r-r_A)(d_{\mathrm{out}}-r_B)+(r-r_B)(d_{\mathrm{in}}-r_A)$. The kernel
\emph{can} be strictly larger than the gauge tangent at rank-deficient points,
but need not be: for $r=d_{\mathrm{in}}=1$, $d_{\mathrm{out}}=2$, $A=[1]$, $B=0$
both are one-dimensional.
\item \emph{Consequences.} For all groups $a,b$ and every $\phi$ there is a
nonzero $v\in\ker J(\phi)$, so
\[
(\Fr_a-\Fr_b)\,v=0
\qquad\text{and so}\qquad
\lambda_{\min}(\Fr_a-\Fr_b)\le0 ,
\]
and if in addition $g_a=g_b$ then some nonzero $v$ has $v^\top\Dab v=0$, so
$\Dab\not\succ0$. (At the origin $A=B=0$ the gauge tangent is zero; there any
nonzero pure-$A$ direction $v=(U,0)$ serves, since $B\,(A+tU)=0$ along the line.)
\end{enumerate}
Without $r\le\min(d_{\mathrm{in}},d_{\mathrm{out}})$ the lower bound $r^2$ in
\eqref{eq:kerbound} fails: $d_{\mathrm{in}}=d_{\mathrm{out}}=1$, $r=2$ gives
$\dim\ker J=3$ at generic points.
\end{restatable}

\begin{proof}
Since $\tfrac{d}{dt}e^{tX}=Xe^{tX}$, the curve
$A(t)=e^{-tX}A$, $B(t)=Be^{tX}$ gives
$B(t)A(t)=Be^{tX}e^{-tX}A=BA$ for every $t$, so $T(\phi(t))=T(\phi)$ identically.
Its tangent at $t=0$ is $v=(\dot A,\dot B)=(-XA,BX)$, and by the product rule
\[
J(\phi)v=s\big(\dot BA+B\dot A\big)=s\big(BXA-BXA\big)=0 ,
\]
so $v\in\ker J(\phi)$. The map $X\mapsto(-XA,BX)$ is linear with kernel
$\{X:XA=0,\ BX=0\}$, so the gauge orbit contributes
$r^2-\dim\{X:XA=0,BX=0\}$ dimensions; if $A$ has full row rank then $XA=0$ forces
$X=0$, and symmetrically for $B$ of full column rank, so in that case the gauge
orbit alone already has dimension $r^2$.

\emph{Gauge tangent dimension.} The map $X\mapsto(-XA,BX)$ is linear, with
kernel $\{X:XA=0,\ BX=0\}$. The condition $XA=0$ says every row $x^\top$ of $X$
satisfies $x^\top A=0$, i.e.\ $x\in(\operatorname{col}A)^\perp\subseteq\R^r$, a
space of dimension $r-r_A$; the condition $BX=0$ says every column of $X$ lies in
$\ker B\subseteq\R^r$, of dimension $r-r_B$. Together they say
$X\in\ker B\otimes(\operatorname{col}A)^\perp$, i.e.\ $X=\sum_i u_iw_i^\top$
with $u_i\in\ker B$ and $w_i\in(\operatorname{col}A)^\perp$, a space of dimension
$(r-r_A)(r-r_B)$. The gauge tangent therefore has dimension
$r^2-(r-r_A)(r-r_B)$.

\emph{Kernel dimension.} Since $J(\dot A,\dot B)=s(\dot BA+B\dot A)$ and $s\ne0$,
\[
\operatorname{range}J
=\{\dot BA\}+\{B\dot A\}
=\big(\R^{d_{\mathrm{out}}}\!\otimes\operatorname{row}A\big)
+\big(\operatorname{col}B\otimes\R^{d_{\mathrm{in}}}\big),
\]
with summands of dimensions $d_{\mathrm{out}}r_A$ and $d_{\mathrm{in}}r_B$ and
intersection exactly $\operatorname{col}B\otimes\operatorname{row}A$, of
dimension $r_Ar_B$ (for subspaces $U\subseteq\R^m$, $V\subseteq\R^n$,
$(\R^m\otimes V)\cap(U\otimes\R^n)=U\otimes V$). Hence
$\operatorname{rank}J=d_{\mathrm{out}}r_A+d_{\mathrm{in}}r_B-r_Ar_B$ exactly and,
with $p=r(d_{\mathrm{in}}+d_{\mathrm{out}})$,
\[
\dim\ker J=p-\operatorname{rank}J
=d_{\mathrm{out}}(r-r_A)+d_{\mathrm{in}}(r-r_B)+r_Ar_B .
\]
Using the assumption $d_{\mathrm{out}},d_{\mathrm{in}}\ge r$ this is at least
$r(r-r_A)+r(r-r_B)+r_Ar_B=r^2+(r-r_A)(r-r_B)\ge r^2$, giving
\eqref{eq:kerbound}. Subtracting the gauge dimension,
\[
\dim\ker J-\big(r^2-(r-r_A)(r-r_B)\big)
=(r-r_A)(d_{\mathrm{out}}-r_B)+(r-r_B)(d_{\mathrm{in}}-r_A),
\]
which is zero at full rank $r_A=r_B=r$, and zero also for $r=d_{\mathrm{in}}=1$,
$d_{\mathrm{out}}=2$, $A=[1]$, $B=0$ (first term $0$, second term
$(1-0)(1-1)=0$), but equals $1$ at the point of Remark~\ref{rem:not-a-quotient}
($d_{\mathrm{in}}=2$). A rank-deficient point therefore \emph{may} have a kernel
strictly larger than the gauge tangent; it need not. Without
$r\le\min(d_{\mathrm{in}},d_{\mathrm{out}})$ the final inequality fails, e.g.\
$d_{\mathrm{in}}=d_{\mathrm{out}}=1$, $r=2$, $r_A=r_B=1$ gives $\dim\ker J=3<4$.

For the consequences, pick a nonzero $v\in\ker J(\phi)$ lying on a curve
$\phi(t)$ along which $T$ is constant: away from the origin
$r^2-(r-r_A)(r-r_B)\ge1$, so some $X$ gives a nonzero gauge tangent on the curve
$(e^{-tX}A,Be^{tX})$; at the origin $A=B=0$ take the straight line
$\phi(t)=(tU,0)$ with $U\ne0$, along which $BA\equiv0$ and $\phi''(0)=0$. In
either case $\Fr_a v=J^\top F_aJv=0$ for every group, hence
$(\Fr_a-\Fr_b)v=0$ with $v\ne0$: zero is an eigenvalue of the symmetric matrix
$\Fr_a-\Fr_b$, so $\lambda_{\min}(\Fr_a-\Fr_b)\le0$. Finally, $L_a$ and $L_b$ are
constant along $\phi(t)$, so $\Delta_{ab}(\phi(t))$ is constant and
\[
0=\frac{d^2}{dt^2}\Delta_{ab}(\phi(t))\Big|_{t=0}
=v^\top\Dab v+(g_a-g_b)^\top\phi''(0).
\]
With $g_a=g_b$ the second term vanishes (at the origin it vanishes regardless,
since $\phi''(0)=0$), leaving $v^\top\Dab v=0$ for a nonzero $v$, so $\Dab$ is not
positive definite.
\end{proof}

\begin{restatable}[The symmetry survives any map through $BA$]{corollary}{corGaugeG}
\label{cor:gauge-G}
Let $\phi=(A,B,m)$, where $m$ collects any further trainable parameters (for
DoRA, the magnitude vector). If $T(\phi)=G(BA,m)$ for any differentiable $G$,
then $T$ is constant along the curve of Proposition~\ref{prop:gauge}(i) with $m$
held fixed, and $v=(-XA,BX,0)\in\ker J(\phi)$, and likewise along the pure-$A$
line at the origin. This covers LoRA, rsLoRA and PiSSA ($G(W,m)=\theta_0+sW$,
no $m$) and also DoRA, whose adapted weight $m(W_0+BA)/\lVert W_0+BA\rVert_c$
depends on $(A,B)$ only through $BA$ wherever the normalization is defined. Since
the $(A,B)$ block of $J$ is $D_1G\cdot J_{BA}$, where $J_{BA}$ is the Jacobian
of $(A,B)\mapsto BA$, $\ker J\supseteq\ker J_{BA}\times\{0\}$:
under the same rank assumption, part~(i), the lower bound $\dim\ker J\ge r^2$ of
part~(ii), and part~(iii) hold verbatim; the exact kernel formula becomes a lower
bound.
\end{restatable}

\begin{proof}
If $T(\phi)=G(BA,m)$ then, with $m$ fixed along the curve,
$T(\phi(t))=G(B(t)A(t),m)=G(BA,m)=T(\phi)$ for all $t$, so
$J(\phi)v=\frac{d}{dt}T(\phi(t))|_{t=0}=0$; the same holds along the pure-$A$
line $(tU,0)$ at the origin. By the chain rule the $(A,B)$ block of $J$ is
$D_1G(BA,m)\,J_{BA}$ with $J_{BA}$ the Jacobian of $(A,B)\mapsto BA$, so
$\ker J\supseteq\ker J_{BA}\times\{0\}$, and
$\dim\ker J_{BA}$ is given exactly by \eqref{eq:kerbound} (the case $G=$ identity,
$s=1$); hence $\dim\ker J\ge r^2$ under the rank assumption. Part~(iii) used only
the existence of a nonzero $v$ on a curve of constant $T$, which we have. For DoRA,
$W=m\,(W_0+BA)/\lVert W_0+BA\rVert_c$ is a function of $BA$ alone at fixed $m$
and $W_0$, and $m$ is not moved by the curve, so the hypothesis holds wherever
the column norm is nonzero.
\end{proof}

\begin{remark}[The horizontal space is first-order only]
\label{rem:not-a-quotient}
Take
$T(a_1,a_2,b)=b\,(a_1,a_2)$ and $\Delta(w_1,w_2)=\tfrac12w_1^2-w_2$, so
$\Delta_{ab}(a_1,a_2,b)=\tfrac12b^2a_1^2-ba_2$. At $(1,0,0)$ the Jacobian is
$J=\big(\begin{smallmatrix}0&0&1\\0&0&0\end{smallmatrix}\big)$, so
$\ker J=\operatorname{span}\{e_{a_1},e_{a_2}\}$ and
$\mathcal{H}_\phi=\operatorname{span}\{e_b\}$; the gradient vanishes and
$\Dab=\big(\begin{smallmatrix}0&0&0\\0&0&-1\\0&-1&1\end{smallmatrix}\big)$, so
$\lambda^{\mathcal{H}}_{\min}(\Dab)=e_b^\top\Dab e_b=1>0$ and
Theorem~\ref{thm:impossible} applies. Indeed
$\Delta_{ab}(\phi+(0,0,t))-\Delta_{ab}(\phi)=\tfrac12t^2>0$. However
$\Delta_{ab}(\phi+(0,t,t))-\Delta_{ab}(\phi)=\tfrac12t^2-t^2=-\tfrac12t^2<0$,
and $(0,t,t)-(0,0,t)=(0,t,0)\in\ker J$. Two updates agreeing to first order in
predictor space therefore move the gap in opposite directions at second order,
so $\mathcal{H}_\phi$ cannot be read as a second-order quotient and no
conclusion transfers from a horizontal representative to an arbitrary one. The
point is the standard $B=0$ initialization, not an exotic configuration. Here $\dim\ker J=2$ while the gauge tangent is one-dimensional ($r=1$,
$r_A=1$, $r_B=0$, $d_{\mathrm{in}}=2$, $d_{\mathrm{out}}=1$, so the excess of
Proposition~\ref{prop:gauge}(ii) is $0+1=1$): the kernel is strictly larger than
the gauge tangent at this point.
\end{remark}

\begin{remark}[Why the restriction is not optional]
\label{rem:gauge-necessary}
Proposition~\ref{prop:gauge} shows the unrestricted hypotheses
``$\Fr_a-\Fr_b\succ0$'' and ``$\Dab\succ0$ with $g_a=g_b$'' are \emph{never}
satisfied for LoRA, rsLoRA or PiSSA at any $\phi$ with $A$ or $B$ of full rank.
Restricting to $\mathcal{H}_\phi=(\ker J)^\perp$ excises exactly the directions
responsible and leaves the conditions satisfiable: in the supplement we sample
random $(F_a,F_b)$ at a random $(A,B)$ and find
$\lambda_{\min}(\Fr_a-\Fr_b)\le0$ in every draw, while
$\lambda^{\mathcal{H}}_{\min}(\Fr_a-\Fr_b)>0$ occurs in roughly half of them.
Since $\ker J$ is precisely the set of directions that leave the predictor
unchanged to first order, $\mathcal{H}_\phi$ is a Euclidean choice of
representative for that first-order quotient, and quantifying over
$\delta\in\mathcal{H}_\phi$ is the weakest restriction that makes the hypotheses
attainable. It is not more than that: Remark~\ref{rem:not-a-quotient} shows the
identification fails at second order.
\end{remark}

\subsection{Necessity and sufficiency}

\thmImpossible*

\begin{proof}
By Proposition~\ref{prop:expansion} with $g_a=g_b$ the slope term vanishes, so
for $\delta\in\mathcal{H}_\phi$,
\[
\Delta_{ab}(\phi+\delta)-\Delta_{ab}(\phi)
=\tfrac12\delta^\top\Dab\delta+R_{ab}(\delta)
\;\ge\;\tfrac12\rho_{\min}\|\delta\|^2-\tfrac{M}{3}\|\delta\|^3 ,
\]
using $\delta^\top\Dab\delta\ge\lambda^{\mathcal{H}}_{\min}(\Dab)\|\delta\|^2$
for $\delta\in\mathcal{H}_\phi$ and
$|R_{ab}|\le\frac{M}{3}\|\delta\|^3$. Since
$\|\delta\|\le\varepsilon\le3\rho_{\min}/4M$,
\[
\tfrac{M}{3}\|\delta\|^3\le\tfrac{M}{3}\varepsilon\|\delta\|^2
\le\tfrac{M}{3}\cdot\tfrac{3\rho_{\min}}{4M}\|\delta\|^2
=\tfrac{\rho_{\min}}{4}\|\delta\|^2 ,
\]
so the increment is at least
$(\tfrac12-\tfrac14)\rho_{\min}\|\delta\|^2=\tfrac14\rho_{\min}\|\delta\|^2>0$
for $\delta\ne0$. (When $M=0$ the remainder vanishes identically, the constraint
$\|\delta\|\le3\rho_{\min}/4M$ is read as vacuous, and the increment is at least
$\tfrac12\rho_{\min}\|\delta\|^2$.) This is an increase of the \emph{signed} gap
$\Delta_{ab}=L_a-L_b$. If $\Delta_{ab}(\phi)\ge0$ then
$|\Delta_{ab}(\phi+\delta)|=\Delta_{ab}(\phi+\delta)>\Delta_{ab}(\phi)=|\Delta_{ab}(\phi)|$,
so the absolute gap widens; without that orientation an increase in
$\Delta_{ab}$ may reduce $|\Delta_{ab}|$.
\end{proof}

\begin{remark}[Both hypotheses are necessary]
\label{rem:thm1-hyps}
\emph{Slope matching.} Let the exact gap along a ray be
$\Delta_{ab}(\phi+t u)=t+\tfrac12t^2$ for a unit $u$, so $g_a-g_b$ has component
$1$ along $u$, $\Dab=1\succ0$ and $M=0$. Taking $t=-0.1$ gives
$\Delta_{ab}(\phi+tu)-\Delta_{ab}(\phi)=-0.095<0$: a nonzero update that
strictly \emph{decreases} the gap. Hence the conclusion, and with it
Theorem~\ref{thm:feasibility}(ii), fails without $g_a=g_b$.
\emph{Orientation.} If $\Delta_{ab}(\phi)=-1$ and the signed gap increases to
$-0.9$, the absolute gap has narrowed. ``Widening'' therefore requires
$\Delta_{ab}(\phi)\ge0$, which is a labelling convention on the pair and not an
additional assumption on the model.
\end{remark}

\corWeyl*

\begin{proof}
By Proposition~\ref{prop:defect}, $\Dab=(\Fr_a-\Fr_b)+(\Xi_a-\Xi_b)$ with both
terms symmetric. Compressing to $\mathcal{H}_\phi$ and applying Weyl's
inequality to the restrictions,
\[
\lambda^{\mathcal{H}}_{\min}(\Dab)
\;\ge\;\lambda^{\mathcal{H}}_{\min}(\Fr_a-\Fr_b)+\lambda^{\mathcal{H}}_{\min}(\Xi_a-\Xi_b)
\;\ge\;\lambda^{\mathcal{H}}_{\min}(\Fr_a-\Fr_b)-\eta^{\mathcal{H}}_{ab},
\]
positive under the stated hypothesis. Theorem~\ref{thm:impossible} additionally
requires $g_a=g_b$, which is why it is stated as a separate condition here. That
the restriction to $\mathcal{H}_\phi$ cannot be removed is
Proposition~\ref{prop:gauge}.
\end{proof}

\begin{remark}[$\Fr_a-\Fr_b$ is not the only lever]
\label{rem:not-only-lever}
Since $\Dab=(\Fr_a-\Fr_b)+(\Xi_a-\Xi_b)$, a procedure that altered the defect
would alter $\Dab$ without touching the reachable Fishers, and
Proposition~\ref{prop:defect} explicitly retains that term. We therefore do not
claim $\|\Fr_a-\Fr_b\|$ is the only available intervention, only that it is the
one our estimator acts on and the one measurable at foundation-model scale.
\end{remark}

\thmWorstcase*

\begin{proof}
Let $\Pi$ be the orthogonal projector onto $\mathcal{H}_\phi$ and
$S=\Pi\Dab\Pi$ the compression, which is symmetric. Since
$\delta\mapsto\delta^\top\Dab\delta$ is homogeneous of degree two and agrees with
$\delta^\top S\delta$ on $\mathcal{H}_\phi$, its maximum over
$\{\delta\in\mathcal{H}_\phi:\|\delta\|\le\varepsilon\}$ is
$\lambda_{\max}(S|_{\mathcal{H}})\varepsilon^2$, attained at $\varepsilon$ times a
corresponding eigenvector. Applying this to $\pm\Dab$ and using
$\|\Dab\|^{\mathcal{H}}_2
=\max\{\lambda_{\max}(S|_{\mathcal{H}}),-\lambda_{\min}(S|_{\mathcal{H}})\}$ gives
$\mathcal{Q}_{ab}(\varepsilon)=\frac12\|\Dab\|^{\mathcal{H}}_2\varepsilon^2$ with
attainment.
For $\mathcal{G}_{ab}$, apply the triangle inequality to
\eqref{eq:gap-expansion}: $|(g_a-g_b)^\top\delta|\le\|g_a-g_b\|\varepsilon$ by
Cauchy--Schwarz, the quadratic term is at most
$\frac12\|\Dab\|^{\mathcal{H}}_2\varepsilon^2$ on $\mathcal{H}_\phi$, and
$|R_{ab}|\le\frac{M}{3}\varepsilon^3$. When $g_a=g_b$ the slope term is absent
and the same two-sided argument gives
$|\mathcal{G}_{ab}(\varepsilon)-\frac12\|\Dab\|^{\mathcal{H}}_2\varepsilon^2|
\le\frac{M}{3}\varepsilon^3$.
\end{proof}

\begin{restatable}[Fisher matching controls the second-order channel up to the defect]{theorem}{thmSufficiency}
\label{thm:sufficiency}
If $\Fr_a=\Fr_b$ then for all $\|\delta\|\le\varepsilon\le\varepsilon_0$,
\[
\big|\Delta_{ab}(\phi+\delta)-\Delta_{ab}(\phi)-(g_a-g_b)^\top\delta\big|
\;\le\;\tfrac12\eta_{ab}\varepsilon^2+\tfrac{M}{3}\varepsilon^3 .
\]
In particular any $\delta\perp(g_a-g_b)$ holds the gap fixed to
$O(\eta_{ab}\varepsilon^2+\varepsilon^3)$.
\end{restatable}

\begin{proof}
By Proposition~\ref{prop:defect}, $\Fr_a=\Fr_b$ gives $\Dab=\Xi_a-\Xi_b$, so
$|\frac12\delta^\top\Dab\delta|\le\frac12\eta_{ab}\|\delta\|^2$. Substituting
into \eqref{eq:gap-expansion} and bounding $|R_{ab}|\le\frac{M}{3}\varepsilon^3$
gives the display; taking $\delta\perp(g_a-g_b)$ removes the slope term.
\end{proof}

\begin{restatable}[Two feasibility results]{theorem}{thmFeasibility}
\label{thm:feasibility}
Define the second-order model
$\widehat\Delta_{ab}(\delta)=(g_a-g_b)^\top\delta+\tfrac12\delta^\top\Dab\delta$.
Since $\widehat\Delta_{ba}=-\widehat\Delta_{ab}$, fix an \emph{orientation}: let
$\mathcal{O}$ contain, for each unordered pair $\{a,b\}$, exactly one ordered pair
$(a,b)$ with $\Delta_{ab}(\phi)\ge0$ (ties broken arbitrarily), so that $a$ is
the currently worse-off group. Let
\[
\mathcal{S}_\varepsilon:=\big\{\delta\in\mathcal{H}_\phi:\ \|\delta\|\le\varepsilon,\
\widehat\Delta_{ab}(\delta)\le0\ \ \forall(a,b)\in\mathcal{O}\big\},
\]
the horizontal updates under which no oriented gap increases \emph{in the truncated quadratic model}. Fix pooling weights $\pi_a\ge0$ with $\sum_a\pi_a=1$, let $\bar g=\sum_a\pi_ag_a$ be the gradient of the pooled loss $\sum_a\pi_aL_a$, and call $\delta\in\mathcal{H}_\phi$ a \emph{pooled descent direction} if $\bar g^\top\delta<0$. (Quantifying over both
orders of every pair would instead force $\widehat\Delta_{ab}(\delta)=0$.)
\emph{(i)} If $\Dab=0$ for all pairs, $\mathcal{S}_\varepsilon$ is the
intersection of a ball with a closed convex polyhedral cone in
$\mathcal{H}_\phi$, and it contains a pooled descent direction iff
$-P_{\mathcal{H}}\bar g\notin\operatorname{cone}\{P_{\mathcal{H}}(g_a-g_b):(a,b)\in\mathcal{O}\}$,
which is a single LP.
\emph{(ii)} If $g_a=g_b$ for all pairs and
$\lambda^{\mathcal{H}}_{\min}(\Dab)>0$ for some $(a,b)\in\mathcal{O}$, then
$\mathcal{S}_\varepsilon=\{0\}$ for every $\varepsilon>0$; the same holds for the
exact gap for $\varepsilon$ as in Theorem~\ref{thm:impossible}.
\end{restatable}

Stating $\mathcal{S}_\varepsilon$ through the truncated model is what makes (i)
exact; phrased on the true loss it would inherit the cubic remainder. The slope
hypothesis in (ii) cannot be dropped: the example after
Theorem~\ref{thm:impossible} has a definite $\Dab$ and a nonzero gap-decreasing
update. And by Proposition~\ref{prop:gauge}, ``$\mathcal{S}_\varepsilon=\{0\}$''
must be read on $\mathcal{H}_\phi$: raw parameters can always move along a gauge
orbit without changing the predictor or any gap. What \emph{exact Hessian} matching $\Dab=0$ buys is
therefore conditional but real: it converts the truncated quadratic
feasibility question into a linear one. Reachable-Fisher matching
$\Fr_a=\Fr_b$ does not by itself do this: it leaves $\Dab=\Xi_a-\Xi_b$, so the
quadratic term is bounded by $\tfrac12\eta_{ab}\varepsilon^2$
(Theorem~\ref{thm:sufficiency}) but not removed, and the exact constraint is
not convexified. Our method matches empirical reachable Fishers, not Hessians.

\begin{proof}
(i) With $\Dab=0$ for all pairs the quadratic term of the model
$\widehat\Delta_{ab}$ vanishes identically (no remainder is involved, since
$\mathcal{S}_\varepsilon$ is defined through the truncated model), and each
constraint reduces to $(g_a-g_b)^\top\delta\le0$, $(a,b)\in\mathcal{O}$, for $\delta\in\mathcal{H}_\phi$;
a finite intersection of homogeneous halfspaces with a subspace is a closed
convex polyhedral cone. Let $A$ have rows $(P_{\mathcal{H}}(g_a-g_b))^\top$, $(a,b)\in\mathcal{O}$, and
$c=-P_{\mathcal{H}}\bar g$. By Farkas' lemma exactly one
holds: there is $\delta$ with $A\delta\le0$ and $c^\top\delta>0$; or there is
$y\ge0$ with $A^\top y=c$, i.e.\ $c\in\operatorname{cone}\{P_{\mathcal{H}}(g_a-g_b):(a,b)\in\mathcal{O}\}$. Because
the constraint set is a cone, feasibility of the strict inequality is unaffected
by the norm bound after rescaling.

(ii) With $g_a=g_b$ the model reduces to
$\widehat\Delta_{ab}(\delta)=\frac12\delta^\top\Dab\delta
\ge\frac12\lambda^{\mathcal{H}}_{\min}(\Dab)\|\delta\|^2>0$ for every nonzero
$\delta\in\mathcal{H}_\phi$, so no such $\delta$ satisfies the constraint and
$\mathcal{S}_\varepsilon=\{0\}$ for every $\varepsilon>0$; no smallness condition
is needed because the model carries no remainder. For the exact gap,
Theorem~\ref{thm:impossible} gives the same conclusion on the stated ball. The
slope hypothesis is necessary: Remark~\ref{rem:thm1-hyps} exhibits a definite
$\Dab$ together with a nonzero gap-decreasing update. The restriction to
$\mathcal{H}_\phi$ is likewise necessary by Proposition~\ref{prop:gauge}, and
$\mathcal{S}_\varepsilon=\{0\}$ must be read as ``no nonzero \emph{horizontal
additive} update'': a step along a gauge orbit changes neither predictor nor
gap, and by Remark~\ref{rem:not-a-quotient} a non-horizontal representative of
the same first-order class may behave differently at second order.
\end{proof}

\subsection{Reachable versus full-model geometry}

Throughout, $\Rch_\phi:=\operatorname{range}J(\phi)\subseteq\R^P$ is the
reachable subspace and $P_\phi$ the orthogonal projector onto it.

\begin{restatable}[Projection and invariance]{lemma}{lemProjection}
\label{lem:projection}
$\Fr_a-\Fr_b=J^\top P_\phi(F_a-F_b)P_\phi J$: the subgroup geometry enters only
through the compression of $F_a-F_b$ onto $\Rch_\phi$. Under a
reparameterization $\phi=S(\psi)$ with invertible Jacobian $K$,
$\Fr_a\mapsto K^\top\Fr_aK$; hence $\tr\Fr_a$ and $\lambda_{\max}(\Fr_a)$ are
not invariant, the across-group coefficient of variation $\cv\{\tr\Fr_a\}_a$ is
invariant under isotropic $K=cI$ but \emph{not} under general $K$, and the
inertia of $\Fr_a-\Fr_b$ (its numbers of positive, negative and zero
eigenvalues) is invariant, as are the generalized eigenvalues of
$(\Fr_a,\Fr_b)$ whenever that pencil is regular (i.e.\
$\det(\Fr_a-\lambda\Fr_b)\not\equiv0$); by Proposition~\ref{prop:gauge} the
pencil is singular on $\ker J$, so regularity should be read on
$\mathcal{H}_\phi$.
\end{restatable}

\begin{proof}
Every column of $J$ lies in $\Rch_\phi$, hence $P_\phi J=J$ and
$J^\top(F_a-F_b)J=(P_\phi J)^\top(F_a-F_b)(P_\phi J)
=J^\top P_\phi(F_a-F_b)P_\phi J$.

For the reparameterization claims, the chain rule applied to
\eqref{eq:reachable-fisher} gives $F^{\mathrm{R},(\psi)}_a=K^\top F^{\mathrm{R},(\phi)}_aK$.
Taking $K=cI$ scales every $\tr\Fr_a$ by $c^2$ without changing the predictor,
so the trace and $\lambda_{\max}$ are not invariant; note that this particular
$K$ leaves the across-group coefficient of variation
$\cv\{\tr\Fr_a\}_a$ unchanged, since it rescales all group traces by the same
factor. A non-isotropic $K$ is required to move the CV, and suffices: with
$p=2$, $\Fr_a=\diag(1,0)$ and $\Fr_b=\diag(0,1)$ the traces are $(1,1)$ and
$\cv=0$, whereas under $K=\diag(1,2)$ they become $(1,4)$ with $\cv=0.6$. Hence
the CV is invariant under isotropic rescaling but not under general
reparameterization.

The difference transforms as
$F^{\mathrm{R},(\psi)}_a-F^{\mathrm{R},(\psi)}_b=K^\top(F^{\mathrm{R},(\phi)}_a-F^{\mathrm{R},(\phi)}_b)K$, a
congruence by an invertible matrix, so its inertia is preserved by Sylvester's
law of inertia. Finally, if $F^{\mathrm{R},(\phi)}_av=\lambda F^{\mathrm{R},(\phi)}_bv$ then
$w=K^{-1}v$ satisfies $F^{\mathrm{R},(\psi)}_aw=\lambda F^{\mathrm{R},(\psi)}_bw$, so the
generalized eigenvalues of the pencil are invariant \emph{provided the pencil is
regular}. By Proposition~\ref{prop:gauge} both $\Fr_a$ and $\Fr_b$ annihilate
$\ker J$, so $\det(\Fr_a-\lambda\Fr_b)\equiv0$ on $\R^p$ and the pencil is
singular there; the statement should be read for the compressions to
$\mathcal{H}_\phi$, where regularity is generic.
\end{proof}

\thmOnlyreachable*

\begin{proof}
(i) Take $P=2$, $p=1$, $J=e_1$ and $F_a=\diag(1,c)$, $F_b=\diag(1,0)$. Both are
positive semidefinite, $\|F_a-F_b\|_2=\|\diag(0,c)\|_2=c$, and
$\Fr_a=e_1^\top F_ae_1=1=e_1^\top F_be_1=\Fr_b$; the disparity lies entirely in
$\Rch^\perp$, consistent with Lemma~\ref{lem:projection}. The consequence for
the gap is exactly Theorem~\ref{thm:sufficiency}, i.e.\ the second-order channel
is bounded by $\frac12\eta_{ab}\varepsilon^2$, not that the gap is fixed,
since the slope term is unconstrained by this construction.

(ii) Take $P=2$, $p=1$, $J=e_1$ and
\[
F_a=\diag(2,1),\qquad F_b=\diag(1,2)=QF_aQ^\top,\qquad
Q=\begin{pmatrix}0&1\\1&0\end{pmatrix}.
\]
$Q$ is orthogonal, so $F_a$ and $F_b$ are orthogonally similar and share the
spectrum $\{1,2\}$; every function of the eigenvalues alone therefore agrees.
Nevertheless $\Fr_a=2$, $\Fr_b=1$ and $\Fr_a-\Fr_b=1>0$. Here $p=1$ and $J$ is
injective, so $\ker J=\{0\}$, $\mathcal{H}_\phi=\R$, and no gauge restriction is
active; under the remaining hypotheses of Theorem~\ref{thm:impossible} (slope
matching and, via Corollary~\ref{cor:weyl}, $\eta^{\mathcal{H}}_{ab}<1$), the
forced-increase conclusion applies, while no summary depending only on the
individual spectra of $F_a$ and $F_b$ distinguishes the two groups (the joint
quantity $\|F_a-F_b\|_2=1$ does). For a factorized adapter the same
construction is read on $\mathcal{H}_\phi$. Padding both matrices with a common block embeds the construction in
any $P>p$.
\end{proof}

\subsection{Trace insufficiency and a matrix-discrepancy alternative}

\begin{restatable}[Trace decomposition]{lemma}{lemTraceDecomp}
\label{lem:trace-decomp}
With $\widehat F^{\mathrm{R},\mathrm{obs}}_a=\frac1{n_a}\sum_{i\in a}
\nabla_\phi\ell_i\nabla_\phi\ell_i^\top$, mean $\hat g_a$ and within-group
covariance $\widehat\Sigma_a$ (both with the $1/n_a$ convention),
$\tr\widehat F^{\mathrm{R},\mathrm{obs}}_a
=\frac1{n_a}\sum_{i\in a}\|\nabla_\phi\ell_i\|_2^2
=\|\hat g_a\|_2^2+\tr\widehat\Sigma_a$.
\end{restatable}

\begin{proof}
Write $v_i=\nabla_\phi\ell_i$. Since $\tr(v_iv_i^\top)=\|v_i\|_2^2$, linearity of
the trace gives the first equality. The parallel-axis identity
$\frac1{n_a}\sum_{i\in a}v_iv_i^\top=\hat g_a\hat g_a^\top+\widehat\Sigma_a$,
with $\widehat\Sigma_a$ using the $1/n_a$ convention, followed by taking traces
gives the second.
\end{proof}

\thmTrace*

\begin{proof}
(i) The implication is immediate, and the converse fails by the construction
in (ii), whose two matrices have equal traces but differ. Quantitatively,
$|\tr(\Fr_a-\Fr_b)|=|\langle I,\Fr_a-\Fr_b\rangle_F|
\le\|I\|_F\|\Fr_a-\Fr_b\|_F=\sqrt p\,\|\Fr_a-\Fr_b\|_F$ by Cauchy--Schwarz.

(ii) \emph{Upper bound.} Let $A,B\succeq0$ with $\tr A=\tr B=\tau$. Weyl's
inequality gives $\lambda_{\max}(A-B)\le\lambda_{\max}(A)-\lambda_{\min}(B)
\le\lambda_{\max}(A)\le\tr A=\tau$, using $\lambda_{\min}(B)\ge0$ and that the
eigenvalues of $A\succeq0$ are nonnegative so the largest is at most their sum.
Symmetrically $\lambda_{\max}(B-A)\le\tau$, and since $A-B$ is symmetric,
$\|A-B\|_2=\max\{\lambda_{\max}(A-B),\lambda_{\max}(B-A)\}\le\tau$.

\emph{Attainment.} For $\Fr_a=\diag(\tau,0,\dots,0)$ and
$\Fr_b=\diag(0,\tau,0,\dots,0)$ both traces equal $\tau$ and
$\Fr_a-\Fr_b=\diag(\tau,-\tau,0,\dots,0)$ has spectral norm $\tau$.

(iii) If $C:=\Fr_a-\Fr_b\succeq0$, its eigenvalues are nonnegative, so
$\|C\|_2=\lambda_{\max}(C)\le\sum_i\lambda_i(C)=\tr C$. If
$\tr\Fr_a=\tr\Fr_b$ as well, then $\tr C=0$ forces every eigenvalue to vanish,
so $C=0$.
\end{proof}

\corTraceExclude*

\begin{proof}
Write $C=\Fr_a-\Fr_b$. Since $\Fr_c v=J^\top F_cJv=0$ for
$v\in\ker J$, and $\Fr_c$ is symmetric, we have $\Fr_c=P_{\mathcal{H}}\Fr_cP_{\mathcal{H}}$
and $C$ vanishes on $\ker J$. If $\lambda^{\mathcal{H}}_{\min}(C)>0$, then $C\succeq0$ on
all of $\R^p$, and taking the trace in an orthonormal basis of
$\mathcal{H}_\phi$ extended by one of $\ker J$ gives
$\tr C=\sum_{i\le d_{\mathcal{H}}}u_i^\top Cu_i\ge d_{\mathcal{H}}\,
\lambda^{\mathcal{H}}_{\min}(C)$, i.e.\ $\lambda^{\mathcal{H}}_{\min}(C)\le\tr C/d_{\mathcal{H}}$
(this holds for any symmetric $C$ vanishing on $\ker J$, since the minimum of the
Rayleigh quotient on $\mathcal{H}_\phi$ is at most its average over an orthonormal
basis). Hence if $|\tr C|/d_{\mathcal{H}}\le\eta^{\mathcal{H}}_{ab}$ then
$\lambda^{\mathcal{H}}_{\min}(C)\le\eta^{\mathcal{H}}_{ab}$, violating the
hypothesis of Corollary~\ref{cor:weyl};
the same argument applies to $-C$ for the reverse orientation. Note that this
excludes only the \emph{sufficient} condition: $\Dab$ may still be positive
definite on $\mathcal{H}_\phi$ through the defect.
\end{proof}

\begin{remark}[Trace equality and the magnitude of the gap change]
\label{rem:trace-magnitude}
Write $C=\Fr_a-\Fr_b$. Theorem~\ref{thm:trace} concerns $C$, whereas Theorem~\ref{thm:worstcase} gives
$\mathcal{Q}_{ab}(\varepsilon)=\frac12\|\Dab\|^{\mathcal{H}}_2\varepsilon^2$,
\emph{not} a formula in $\|\Fr_a-\Fr_b\|_2$; the two differ by the defect. For
reachable Fishers the restricted and full spectral norms of $C$ coincide: since
$C=P_{\mathcal{H}}CP_{\mathcal{H}}$, for any unit $w$,
$|w^\top Cw|=|(P_{\mathcal{H}}w)^\top C(P_{\mathcal{H}}w)|\le
\|C\|^{\mathcal{H}}_2\|P_{\mathcal{H}}w\|^2\le\|C\|^{\mathcal{H}}_2$, so
$\tau_{\mathcal{H}}=\|C\|^{\mathcal{H}}_2=\|C\|_2$, which is at most the common
trace $\tau$ of Theorem~\ref{thm:trace}(ii). The horizontal restriction matters for $\Dab$ and the defect,
which do not vanish on $\ker J$, not for the reachable Fishers themselves. Since
$\big|\,\|\Dab\|^{\mathcal{H}}_2-\tau_{\mathcal{H}}\,\big|\le\eta^{\mathcal{H}}_{ab}$,
under slope matching Theorem~\ref{thm:worstcase} gives
$|\mathcal{G}_{ab}(\varepsilon)-\frac12\tau_{\mathcal{H}}\varepsilon^2|
\le\frac12\eta^{\mathcal{H}}_{ab}\varepsilon^2+\frac{M}{3}\varepsilon^3$: the
\emph{center} of this bound is $\frac12\tau_{\mathcal{H}}\varepsilon^2$, and it
is an exact value only if in addition $\eta^{\mathcal{H}}_{ab}=0$ and $M=0$
(an exactly quadratic gap). Equating the two is invalid: with
$\Fr_a-\Fr_b=\diag(\tau,-\tau)$ and $\Xi_a-\Xi_b=\diag(0,3\tau)$ one gets
$\Dab=\diag(\tau,2\tau)\succ0$, so the definiteness premise holds, yet
$\mathcal{Q}_{ab}(\varepsilon)=\tau\varepsilon^2\ne\frac12\tau\varepsilon^2$.
Note also that the extremal difference $\diag(\tau,-\tau,0,\dots)$ is itself
indefinite, so Theorem~\ref{thm:impossible} does not apply to it and no
forced-increase reading attaches to the extremal construction.
\end{remark}

\propGram*

\begin{proof}
Expand $\|\widehat S_a-\widehat S_b\|_F^2=\tr(\widehat S_a^2)-2\tr(\widehat S_a\widehat S_b)+\tr(\widehat S_b^2)$.
By cyclicity of the trace,
\[
\tr(\widehat S_a\widehat S_b)=\frac{1}{n_an_b}\tr\!\big(G_a^\top G_aG_b^\top G_b\big)
=\frac{1}{n_an_b}\tr\!\big(G_aG_b^\top(G_aG_b^\top)^\top\big)
=\frac{\|G_aG_b^\top\|_F^2}{n_an_b},
\]
and the same computation with $b=a$ gives
$\tr(\widehat S_a^2)=\|G_aG_a^\top\|_F^2/n_a^2$. Nothing uses $m_a=n_a$: the
normalization $1/n_a$ and the row count $m_a$ enter separately.

Two factors are used in this paper. \emph{Observed}: rows $\nabla_\phi\ell_i$ at
the observed labels ($m_a=n_a$) give the observed-label empirical reachable Fisher
$\widehat F^{\mathrm{R},\mathrm{obs}}_a$. \emph{Expected}: with $q_\theta(c\mid x)$
the predicted probability of class $c$ among $C$ classes, rows
$\sqrt{q_\theta(c\mid x_i)}\,\nabla_\phi\log q_\theta(c\mid x_i)$ over all
classes ($m_a=n_aC$), or $\sigma^{-1}\nabla_\phi\mu_\theta(x_i)$ for Gaussian
regression with predicted mean $\mu_\theta$ and variance $\sigma^2$ ($m_a=n_a$),
give the sample estimate $\widehat F^{\mathrm{R},\mathrm{exp}}_a$ of the
model-expected reachable Fisher~\eqref{eq:reachable-fisher}. For the expected
factor, $\frac1{n_a}G_a^\top G_a=\frac1{n_a}\sum_i\sum_cq_\theta(c\mid x_i)
\nabla_\phi\log q_\theta(c\mid x_i)\nabla_\phi\log q_\theta(c\mid x_i)^\top$, which
is the sample average over inputs of the per-input expected score outer product,
i.e.\ $\widehat F^{\mathrm{R},\mathrm{exp}}_a$; for Gaussian regression the score is
$\sigma^{-2}(y-\mu)\nabla_\phi\mu$ and its expectation over $y$ gives
$\sigma^{-2}\nabla_\phi\mu\nabla_\phi\mu^\top$. All three Gram blocks are
required. Forming $G_aG_a^\top$, $G_aG_b^\top$ and $G_bG_b^\top$ costs
$O\big((m_a^2+m_am_b+m_b^2)p\big)$ and memory
$O\big(m_a^2+m_am_b+m_b^2+(m_a{+}m_b)p\big)$; no $p\times p$ object appears. The
trace $\tr\widehat S_a=\frac1{n_a}\|G_a\|_F^2$ costs $O(m_ap)$. For
$m_a\asymp m_b\asymp m$ the discrepancy is a factor $\Theta(m)$ more expensive;
without balance the self-Gram terms dominate when one group is much larger. Both
are linear in $p$.
\end{proof}

\begin{corollary}[The Frobenius discrepancy bounds the second-order channel]
\label{cor:closure}
Let $\widehat{\mathcal{P}}(\phi)=\sum_{a<b}\|\widehat F^{\mathrm{R},\mathrm{obs}}_a
-\widehat F^{\mathrm{R},\mathrm{obs}}_b\|_F^2$ be computed by \eqref{eq:gram} from the
\emph{observed-label} per-example gradients actually used in training, and let
$\zeta_c=\|\widehat F^{\mathrm{R},\mathrm{obs}}_c-\Fr_c\|_2$ measure its departure from
the model-expected reachable Fisher of \eqref{eq:reachable-fisher}. Then for
$\delta\in\mathcal{H}_\phi$,
\[
\mathcal{G}_{ab}(\varepsilon)\;\le\;\|g_a-g_b\|\,\varepsilon
\;+\;\tfrac12\big(\sqrt{\widehat{\mathcal{P}}(\phi)}+\zeta_a+\zeta_b
+\eta^{\mathcal{H}}_{ab}\big)\varepsilon^2
\;+\;\tfrac{M}{3}\varepsilon^3 .
\]
Driving $\widehat{\mathcal{P}}\to0$ therefore bounds the second-order channel
only up to $\zeta_a+\zeta_b+\eta^{\mathcal{H}}_{ab}$, and does \emph{not} bound
the gap, because the $O(\varepsilon)$ slope term remains and generally dominates.
\end{corollary}

\begin{proof}
Each summand of $\widehat{\mathcal{P}}$ is nonnegative, so
$\|\widehat F^{\mathrm{R},\mathrm{obs}}_a-\widehat F^{\mathrm{R},\mathrm{obs}}_b\|_2
\le\|\widehat F^{\mathrm{R},\mathrm{obs}}_a-\widehat F^{\mathrm{R},\mathrm{obs}}_b\|_F
\le\sqrt{\widehat{\mathcal{P}}(\phi)}$. The training statistic is the
observed-label empirical matrix, whereas Proposition~\ref{prop:defect} is stated
for the model-expected $\Fr_c$; two applications of the triangle inequality give
\[
\|\Fr_a-\Fr_b\|_2\;\le\;\sqrt{\widehat{\mathcal{P}}(\phi)}+\zeta_a+\zeta_b .
\]
Combining with $\|\Dab\|^{\mathcal{H}}_2\le\|\Fr_a-\Fr_b\|_2+\eta^{\mathcal{H}}_{ab}$
and substituting into the $\mathcal{G}_{ab}$ bound of
Theorem~\ref{thm:worstcase}, which retains the slope term, gives the display.
The $\zeta_c$ are not controlled by the penalty and do not vanish with batch
size in general: the observed-label empirical Fisher is neither the expected
Fisher nor an unbiased estimator of it \citep{kunstner2019limitations}, so
$\zeta_c$ contains a systematic component as well as sampling error.
\end{proof}

\begin{remark}[The slope term cannot be omitted]
\label{rem:slope-needed}
Take the fixed-variance Gaussian model $p_\phi(y)=\mathcal N(\phi,1)$ covered by
\textbf{A2}, evaluated at $\phi=0$, with the group \emph{distributions}
$Y_a\equiv-1$ and $Y_b\in\{-1,+1\}$ equiprobably. This specifies the laws, not
only their means, so that every term but the slope is exactly zero. The
observed-label gradient is $\nabla_\phi\ell=\phi-y$, so at $\phi=0$ the
observed-label empirical Fishers are $\E[y^2]=1$ for both groups, matching the
model Fisher $\Fr_a=\Fr_b=1$; hence $\widehat{\mathcal{P}}=\zeta_a=\zeta_b=0$.
Also $H_a=H_b=1$ so $\eta_{ab}=0$, and $M=0$. But the mean gradients are
$g_a=1$ and $g_b=0$, so $g_a-g_b=1$, the gap change is exactly $\delta$, and
$\mathcal{G}_{ab}(\varepsilon)=\varepsilon$. A bound omitting the slope term
would assert $\mathcal{G}_{ab}(\varepsilon)\le0$. This is why
Corollary~\ref{cor:closure} is a bound on the second-order channel only, and why
the Frobenius discrepancy of Section~\ref{sec:frobenius} is not presented as a gap guarantee.
\end{remark}

\subsection{Performance gaps}

\thmBayesfloor*

\begin{proof}
Decompose $R_a-R_b=(R_a^\star-R_b^\star)+(R_a-R_a^\star)-(R_b-R_b^\star)$. By
classification calibration, $0\le R_c-R_c^\star\le\psi^{-1}(\epsilon_c)$
\citep{bartlett2006convexity}, so the difference of the excess terms lies in
$[-\psi^{-1}(\epsilon_b),\psi^{-1}(\epsilon_a)]$ and has absolute value at most
$\max_c\psi^{-1}(\epsilon_c)$. The reverse triangle inequality applied to
$|R_a-R_b|$ and $\gamma_{ab}=|R_a^\star-R_b^\star|$ gives the two-sided display,
and rearranging with $|R_a-R_b|\ge0$ gives the lower bound. If
$\epsilon_a,\epsilon_b\to0$ along a sequence then $\psi^{-1}(\epsilon_c)\to0$,
since $\psi$ is continuous, convex and vanishes only at $0$, so
$|R_a-R_b|\to\gamma_{ab}$.
\end{proof}

The lower bound is $\max\{0,\gamma_{ab}-\max_c\psi^{-1}(\epsilon_c)\}$ and is
vacuous whenever some group carries excess risk at least $\gamma_{ab}$;
Remark~\ref{rem:not-a-floor} shows this is not an artifact of the bound, since
gaps strictly below $\gamma_{ab}$ are attainable by degrading the better-served
group. We therefore make no claim that $\gamma_{ab}$ lower-bounds the achievable
gap.

\begin{remark}[$\gamma_{ab}$ is not a lower bound]
\label{rem:not-a-floor}
A predictor with $R_a^\star=0.2$, $R_b^\star=0.1$ and $R_a=R_b=0.2$ achieves gap
$0<\gamma_{ab}=0.1$, consistent with the inequality because group $b$ then
carries excess risk $0.1$. Gaps below $\gamma_{ab}$ are therefore attainable,
but only by degrading the better-served group, that is, by levelling down. What
Theorem~\ref{thm:bayesfloor} rules out is closing the gap \emph{while} both
groups approach their Bayes risk.
\end{remark}

\begin{remark}[Balanced error and several groups]
\label{rem:balanced}
For binary labels, balanced error $\tfrac12(\mathrm{FNR}_a+\mathrm{FPR}_a)$ is the
$0$--$1$ risk of group $a$ under the reweighted distribution $\widetilde{\mathcal
D}_a$ that keeps the class-conditional laws of $\mathcal D_a$ and sets
$\Pr(Y{=}1)=\tfrac12$. Applying Theorem~\ref{thm:bayesfloor} to
$\widetilde{\mathcal D}_a$ and $\widetilde{\mathcal D}_b$, with $L_c$ the
class-balanced surrogate risk and $R_c^\star$ the balanced Bayes risk, gives the
same conclusion for balanced error; the proof is unchanged. For a maximum gap over
several groups, apply the theorem to each pair: $\max_{a,b}|R_a-R_b|$ converges
to $\max_{a,b}\gamma_{ab}$ when every excess risk vanishes. Multiclass balanced
accuracy would require a multiclass calibration inequality, and MAE a separate
result for conditional-median prediction; we prove neither.
\end{remark}

\begin{remark}[Consistency with the measurements]
\label{rem:predicts-null}
For binary classification Theorem~\ref{thm:bayesfloor} and
Remark~\ref{rem:balanced} suggest one mechanism for Table~\ref{tab:peft}: where
adaptation leaves the model close to its per-group Bayes risk, the gap is close
to the Bayes-risk difference regardless of curvature dispersion. The reported
gap ratios, however, use multiclass balanced accuracy (Fitzpatrick17k) and MAE (UTKFace, FHIBE face), which the theorem does not cover, so for those
entries the mechanism is an analogy, not an explanation, and for UTKFace we
offer none. We did not measure excess risks or Bayes risks, and alternative
explanations (estimation noise, the single seed, coordinate dependence of
the CV ratios, and the 15 non-independent method--task pairs behind
$\rho=-0.17$) are not excluded. The theory does identify where matching could
show an effect: where excess risk is large relative to the Bayes-risk
difference, i.e.\ under-fitted or shifted subgroups.
\end{remark}

%% file: methodology.tex
\section{Training objective}
\label{sec:method}
\label{sec:mitigation}

Section~\ref{sec:theory} gives sufficient conditions and diagnostic quantities;
it does not show that a successful adapter must equalize any trace or matrix
discrepancy. Theorem~\ref{thm:worstcase} shows that, for additive updates in the horizontal
subspace $\mathcal{H}_\phi$ of Section~\ref{sec:theory-setup}, the worst-case
\emph{second-order} term of the subgroup gap over a trust region of radius
$\varepsilon$ is exactly $\tfrac12\|\Dab\|^{\mathcal{H}}_2\varepsilon^2$, which
agrees with $\tfrac12\|\Fr_a-\Fr_b\|^{\mathcal{H}}_2\varepsilon^2$ to within the
defect $\eta^{\mathcal{H}}_{ab}$. A
natural target is therefore a \emph{matrix} discrepancy between reachable
Fishers, for that channel only, since the slope term $\|g_a-g_b\|\varepsilon$ dominates it for
small updates and is not addressed by any curvature statistic. This section
builds a trainable objective around that target. We first describe the
per-example quantity the implementation computes (the squared norm of the
loss gradient with respect to the parameters actually being trained, whose
subgroup mean is exactly a reachable-Fisher trace), then, in
Section~\ref{sec:frobenius}, the Frobenius discrepancy, a computable upper bound
on the operator norm that appears in the magnitude bound, which
Proposition~\ref{prop:gram} makes computable without forming any $p\times p$
matrix and which we use to audit the trained models.
Figure~\ref{fig:method-overview} summarizes the pipeline.

\begin{figure}[H]
\centering
\resizebox{\linewidth}{!}{%
\begin{tikzpicture}[
  node distance=5mm and 7mm,
  box/.style={draw,rounded corners=1.5pt,align=center,minimum height=8mm,
              inner xsep=6pt,fill=blue!4},
  stat/.style={draw,rounded corners=1.5pt,align=center,minimum height=8mm,
               inner xsep=6pt,fill=orange!9},
  arr/.style={-{Latex[length=2mm]},semithick}
]
\node[box] (batch) {group-balanced\\training batch};
\node[box,right=of batch] (grad) {per-example adapter gradient\\$u_i=\nabla_\phi\ell_i$};
\node[stat,right=of grad] (trace) {subgroup sensitivity\\$c_a=\operatorname{mean}_{i:a_i=a}\|u_i\|_2^2$};
\node[stat,right=of trace] (pen) {normalized Huber penalty\\with adaptive group weights};
\node[box,right=of pen] (step) {balanced gradient\\AdamW update};
\draw[arr] (batch) -- (grad);
\draw[arr] (grad) -- (trace);
\draw[arr] (trace) -- (pen);
\draw[arr] (pen) -- (step);
\end{tikzpicture}
}
\caption{Training pipeline. After training, the pairwise Gram discrepancy of Equation~\ref{eq:mit-gram} audits both arms (Section~\ref{sec:frobenius}). The task loss uses every row in the batch. The curvature penalty uses observed-label gradients with respect to the trainable adapter parameters and task head, aggregates their squared norms by subgroup, and combines the resulting penalty gradient with the task gradient.}
\label{fig:method-overview}
\end{figure}

\subsection{Model, groups, and trainable parameters}

Let $\mathcal D=\{(x_i,y_i,a_i)\}_{i=1}^N$ contain an input $x_i$, target $y_i$, and recorded subgroup label $a_i\in\{1,\ldots,K\}$. For each experiment, these labels define the $K$ subgroups whose sensitivities are matched during training.

The paired runs of Section~\ref{sec:mitigation-results} adapt the backbone with
LoRA, DoRA or PiSSA \citep{hu2021lora, liu2024dora, meng2024pissa}. LoRA and PiSSA
use $W=W_0+s\,BA$, with PiSSA initializing $(A,B)$ from the leading singular
components of $W_0$. For DoRA the adapted weight is
\begin{equation}
W=m\,\frac{W_0+BA}{\lVert W_0+BA\rVert_c},
\qquad
B\in\R^{d_{\mathrm{out}}\times r},\quad
A\in\R^{r\times d_{\mathrm{in}}},
\label{eq:mit-dora}
\end{equation}
where $W_0$ is frozen, $A$ and $B$ form the low-rank update, and $m$ is a trainable column-wise magnitude. The paired runs of Section~\ref{sec:mitigation-results} use rank $r=8$, scale $\alpha=16$ (so $s=\alpha/r$), and zero adapter dropout for all three adapters, with the task head trained in full; the diagnostic sweep uses the same LoRA-family settings, with VeRA at rank 256 (Appendix~\ref{app:repro}). We use $\phi$ to denote the complete set of trainable parameters, consisting of the adapter parameters and the task head. All gradients used by the proposed method are calculated with respect to $\phi$.

\subsection{Subgroup sensitivity from per-example gradients}

Let $\ell_i(\phi)=\ell(f_{\theta,\phi}(x_i),y_i)$ be the unreduced observed-label loss, where $\theta$ denotes the frozen parameters. For each example, the implementation computes
\begin{equation}
u_i=\nabla_\phi\ell_i(\phi),
\qquad
t_i=\lVert u_i\rVert_2^2.
\label{eq:mit-ti}
\end{equation}
For subgroup $a$, the corresponding observed-label empirical reachable Fisher
(in the notation of Section~\ref{sec:theory}; it is $p\times p$, not the
$P\times P$ full-model $F_a$) would be
\begin{equation}
\widehat F^{\mathrm{R},\mathrm{obs}}_a=\frac{1}{n_a}\sum_{i\in\mathcal B:a_i=a}u_i u_i^\top .
\label{eq:mit-full-fisher}
\end{equation}
Its trace measures total gradient sensitivity across all trainable
directions and gives one scalar comparable across subgroups, computable exactly
without constructing the full matrix. It is a \emph{necessary} matching
condition (equal matrices have equal traces), but by
Theorem~\ref{thm:trace} it is not sufficient, and its insufficiency is extremal:
among positive semidefinite matrices of equal trace $\tau$, the spectral
discrepancy can be as large as $\tau$. Lemma~\ref{lem:trace-decomp} adds that
the empirical trace decomposes as
$\tr\widehat F^{\mathrm{R},\mathrm{obs}}_a=\|\hat g_a\|_2^2+\tr\widehat\Sigma_a$, so a
trace penalty equalizes a \emph{sum} of a mean-gradient term and a within-group
dispersion term, and the two can trade off while the penalty reports success.
The mean-gradient term is related to, but does not determine, the slope channel
$g_a-g_b$ of Proposition~\ref{prop:expansion}: equal norms $\|\hat g_a\|=\|\hat g_b\|$
do not imply $\hat g_a=\hat g_b$. We therefore treat the trace statistic below as the base
estimator and the training signal, and
Section~\ref{sec:frobenius} complements it with a matrix discrepancy that controls the magnitude of the curvature channel and audits what the trace penalty achieved. By linearity of the trace, $\operatorname{tr}(u_i u_i^\top)=\lVert u_i\rVert_2^2$, so the average of the scalars $t_i$ is exactly $\tr\widehat F^{\mathrm{R},\mathrm{obs}}_a$. In contrast, explicitly storing $\widehat F^{\mathrm{R},\mathrm{obs}}_a$ costs $O(p^2)$ for $p=\dim(\phi)$. The ConvNeXt--DoRA runs contain approximately $1.58$ million trainable parameters; one dense FP32 subgroup matrix would therefore contain about $2.5\times10^{12}$ entries and require roughly $10$ TB. Our matrix-free computation follows the same principle as per-example-gradient and matrix--vector-product curvature methods, but it does not require a randomized trace estimator because the trace of each outer product has the closed form above \citep{pearlmutter1994fast, dangel2019backpack, hutchinson1989stochastic}. For the examples from subgroup $a$ in the current batch $\mathcal B$, we compute
\begin{equation}
c_a(\phi)=\frac{1}{n_a}\sum_{i\in\mathcal B:a_i=a}t_i,
\qquad
n_a=|\{i\in\mathcal B:a_i=a\}|.
\label{eq:mit-ca}
\end{equation}
Thus, $c_a$ measures how strongly the current examples from subgroup $a$ act on the parameters being adapted. A larger value indicates a stronger local response in the trainable parameter space. The proposed objective uses this signal to bring subgroup sensitivities closer together during training.

\subsection{Normalizing and penalizing subgroup differences}

The objective is motivated by Theorem~\ref{thm:impossible} rather than by the
diagnostic sweep: when slopes match and the reachable-Fisher difference is
positive definite on the horizontal subspace with restricted minimum eigenvalue
above the Hessian--Fisher defect, every small nonzero horizontal update increases
the signed loss gap, so reducing the discrepancy is one available intervention,
and the one that is measurable at scale. A trace penalty reaches that condition
only when the two groups' reachable Fishers are Loewner ordered
(Theorem~\ref{thm:trace}(iii)), which we do not measure; applied to every pair,
as here, it is a heuristic suggested by that special case rather than an
intervention the theory guarantees. It is not the only one: the defect is another lever
(Remark~\ref{rem:not-only-lever}), and none of the three conditions (slope
matching, a small defect, the horizontal-step restriction) is enforced by the procedure
below. 
The sweep of Section~\ref{sec:results} supplies a consistency check on the
ranking premise, but the motivation does not rest on it.
During training, we use the observed-label statistic $c_a$ because it is computed directly from the loss gradients used to update the model; Section~\ref{sec:method} keeps it distinct
from the expected-Fisher geometry used in evaluation.

Raw trace values vary by orders of magnitude across datasets, parameterizations, and training epochs, and can be much larger than the task loss. We therefore compare each subgroup with a detached within-batch reference. In the default method, the reference is the unweighted mean over the constrained subgroups present in the batch:
\begin{equation}
c_{\mathrm{ref}}=\operatorname{sg}\!\left[\frac{1}{|\mathcal{A}_B|}\sum_{a\in\mathcal{A}_B}c_a\right],
\qquad
e_a=\frac{c_a-c_{\mathrm{ref}}}{\max\{c_{\mathrm{ref}},\epsilon_{\mathrm{den}}\}},
\label{eq:mit-dev}
\end{equation}
where $\mathcal{A}_B$ is the set of constrained groups present in the batch, $\operatorname{sg}$ denotes stop-gradient, and $\epsilon_{\mathrm{den}}=10^{-8}$ is a fixed floor that keeps the ratio defined when every constrained group has zero statistic (the same floor is used in Equations~\ref{eq:mit-frob-reg} and~\ref{eq:mit-dual}). The quantity $e_a$ is the normalized version of $c_a$. It measures the difference between subgroup $a$ and the batch reference as a proportion of that reference. For example, $e_a=0.2$ means that $c_a$ is 20\% above the subgroup mean, while $e_a=-0.2$ means that it is 20\% below the mean. This normalization places all subgroup differences on a common scale, even when the raw trace values vary across datasets or training stages. The reference is detached so that it provides a fixed target during the current update.

We apply a Huber penalty to the normalized deviations,
\begin{equation}
\rho_\kappa(e)=
\begin{cases}
\tfrac12 e^2, & |e|\le\kappa,\\[2pt]
\kappa(|e|-\tfrac12\kappa), & |e|>\kappa,
\end{cases}
\qquad \kappa=1.
\label{eq:mit-huber}
\end{equation}
Here $\rho_\kappa$ denotes the Huber loss and $\kappa$ is its transition threshold \citep{huber1992robust} (we avoid $\delta$, which denotes an adapter update in Section~\ref{sec:theory}). The implementation uses $\kappa=1$. Since $e_a$ is a relative deviation, this threshold corresponds to an absolute difference equal to the reference value. The loss is quadratic when $|e_a|\le 1$, so small subgroup differences are corrected smoothly. It becomes linear when $|e_a|>1$, which limits the effect of an unusually large batch estimate. In this way, the penalty remains sensitive near the matching point without allowing one extreme subgroup value to dominate the update. With subgroup weights $\mu_a$, the regularizer is
\begin{equation}
R(\phi,\mu)=\sum_{a\in\mathcal{A}_B}\mu_a\rho_\kappa(e_a).
\label{eq:mit-regularizer}
\end{equation}

The default reference is the subgroup mean. The ablations replace it with the minimum subgroup trace or replace Huber with linear, absolute, Euclidean-norm, or one-sided hinge penalties. These alternatives are specified as variants of the same pipeline rather than as separate methods.

\subsection{A matrix discrepancy for auditing trained models}
\label{sec:frobenius}

Theorem~\ref{thm:trace} shows that trace equality cannot control the operator
norm that sets the worst-case curvature channel; Corollary~\ref{cor:closure}
shows that the Frobenius discrepancy does bound it, up to the estimator and
defect terms. It is one sufficient, computable matrix norm, not the only one, and
like any scalar norm it does not certify the definiteness condition of
Corollary~\ref{cor:weyl}. Let $G_a\in\R^{n_a\times p}$ collect the observed-label
per-example gradients $u_i$ of the constrained rows of subgroup $a$ in the
current batch, so $\widehat F^{\mathrm{R},\mathrm{obs}}_a=\frac1{n_a}G_a^\top G_a$. By
Proposition~\ref{prop:gram} the pairwise Frobenius discrepancy is available
without forming a $p\times p$ matrix:
\begin{equation}
d_{ab}^2
=\big\|\widehat F^{\mathrm{R},\mathrm{obs}}_a-\widehat F^{\mathrm{R},\mathrm{obs}}_b\big\|_F^2
=\frac{\|G_aG_a^\top\|_F^2}{n_a^2}
-\frac{2\|G_aG_b^\top\|_F^2}{n_an_b}
+\frac{\|G_bG_b^\top\|_F^2}{n_b^2}.
\label{eq:mit-gram}
\end{equation}
Only the $n\times n$ Gram matrices $G_aG_b^\top$ appear. For a balanced batch
with $n$ rows per group and $|\mathcal{A}_B|$ constrained groups, all pairs cost
$O(|\mathcal{A}_B|^2n^2p)$ time and $O(|\mathcal{A}_B|^2n^2+|\mathcal{A}_B|np)$ memory, against $O(|\mathcal{A}_B|np)$ for the
trace. As an operation count on the statistic alone this is a factor $\Theta(n)$
more work; with at most $64$ rows per group the Gram blocks are
at most $64\times64$, so the overhead is bounded by the per-group batch size on
an operation that already materializes $G_a$. No step of
\eqref{eq:mit-gram} depends on $p$ beyond the single matrix product, which is
why the discrepancy is affordable at the $1.58$M trainable parameters of the
ConvNeXt--DoRA runs, where a dense $\widehat F^{\mathrm{R},\mathrm{obs}}_a$ would need
roughly $10$ TB.

That count is not the wall-clock cost, because it omits what the two penalties
share. Both gradients are a sum of one Hessian-vector product per example and
differ only in the vector: $\sum_a(\partial R/\partial c_a)(2/n_a)\sum_{i\in a}
H_iu_{ai}$ for the trace, and $\sum_i H_iv_i$ with $v=(W+W^\top)G$,
$W=\partial R/\partial K$, for the Frobenius form. Each product can be formed and
discarded one example at a time, so neither keeps more than one second-order
graph alive, and the Frobenius penalty pays extra only for the detached $G$ and
$v$ and for the small Gram matrix. Table~\ref{tab:timing} measures single
training steps from a stored checkpoint at the training batch composition. The
shared per-example term dominates: the Frobenius penalty costs
$1.00\times$ the trace penalty on ConvNeXt--DoRA and $0.93\times$ on ViT--LoRA,
with identical peak memory, while both cost $23$--$372\times$ a step with no
penalty. We therefore use the trace during training for scalability and
simplicity rather than for speed, and reserve the Frobenius discrepancy for the
post-hoc audit, whose statistic alone costs $1$--$2\%$ of a penalized step.

\begin{table}[t]
\centering\small
\caption{Measured cost of one training step, median of 20 timed steps after 3
warm-up steps on UTKFace at the training batch
composition ($25$ examples per subgroup $\times\,5$ groups). The optimizer update is
excluded; it is identical in every row. NVIDIA RTX~5090, PyTorch
2.14.0+cu130, CUDA~13.0.}
\label{tab:timing}
\begin{tabular}{lrrr@{\hskip 1.1cm}rrr}
\toprule
& \multicolumn{3}{c}{ConvNeXt--DoRA ($p=1.58$M)} & \multicolumn{3}{c}{ViT--LoRA ($p=1.34$M)} \\
\cmidrule(lr){2-4}\cmidrule(lr){5-7}
Step & s/step & $\times$trace & peak MB & s/step & $\times$trace & peak MB \\
\midrule
no penalty            & 0.478   & ---  & 14{,}656 & 0.289 & ---  & 6{,}539 \\
trace penalty         & 177.63  & 1.00 & 14{,}668 & 6.562 & 1.00 & 6{,}549 \\
Frobenius penalty     & 176.98  & 1.00 & 14{,}668 & 6.093 & 0.93 & 6{,}549 \\

trace statistic only      & 2.953 & 0.02 & 2{,}568 & 1.556 & 0.24 & 2{,}195 \\
Frobenius statistic only  & 2.692 & 0.02 & 5{,}832 & 1.436 & 0.22 & 5{,}248 \\
\bottomrule
\end{tabular}
\end{table}


Normalization and robustification carry over unchanged from the trace penalty.
Writing $s_{ab}=\operatorname{sg}[\tfrac12(\tr\Fr_a+\tr\Fr_b)]$ for a detached
scale reference (which by Lemma~\ref{lem:projection} is the same
reparameterization-sensitive quantity the trace penalty already used, so no new
coordinate dependence is introduced), the normalized discrepancy and the
regularizer are
\begin{equation}
\tilde e_{ab}=\frac{d_{ab}}{\max\{s_{ab},\epsilon_{\mathrm{den}}\}},
\qquad
R_{\mathrm F}(\phi,\mu)=\sum_{a<b}\mu_{ab}\,\rho_\kappa(\tilde e_{ab}),
\label{eq:mit-frob-reg}
\end{equation}
with $\rho_\kappa$ the Huber loss of Equation~\ref{eq:mit-huber} and $\mu_{ab}$
pair weights updated by the same bounded exponentiated rule as
Equation~\ref{eq:mit-dual}, with $\nu_{ab}$ formed from the moving average of
$\tilde e_{ab}$ instead of the subgroup gap. The discrepancy $d_{ab}\ge0$ is
zero iff the two empirical reachable Fishers agree, and unlike a trace difference
it cannot hide directional differences through cancellation of eigenvalues of
opposite sign. It is \emph{not} one-sided: minimizing it may lower one group's
matrix, raise the other's, or move both, so matching can still be achieved by
increasing a subgroup's curvature; the task loss is what discourages the
degenerate versions of this. Corollary~\ref{cor:closure}
then gives the property the trace objective lacks, with all three gaps between
penalty and gap made explicit: writing
$\widehat{\mathcal{P}}(\phi)=\sum_{a<b}d_{ab}^2$ for the quantity actually
computed from observed-label gradients,
\[
\mathcal{G}_{ab}(\varepsilon)\le\|g_a-g_b\|\varepsilon
+\tfrac12\big(\sqrt{\widehat{\mathcal{P}}}+\zeta_a+\zeta_b+\eta^{\mathcal{H}}_{ab}\big)\varepsilon^2
+\tfrac{M}{3}\varepsilon^3
\]
for every pair and every horizontal update direction. Driving the penalty down
bounds the second-order channel \emph{up to} $\zeta_a+\zeta_b+\eta^{\mathcal{H}}_{ab}$;
it does \emph{not} bound the gap, because the $O(\varepsilon)$ slope term
survives and generally dominates. The $\zeta_c$ arise because this section
computes an observed-label empirical matrix while the theory is stated for the
model-expected reachable Fisher, and they carry a systematic component that
larger batches do not remove \citep{kunstner2019limitations}. This is why the
penalty is combined with the task loss and the balanced step below rather than
offered as a standalone guarantee.

Three practical caveats follow from the theory rather than from tuning. Zeroth,
by Proposition~\ref{prop:gauge} and Corollary~\ref{cor:gauge-G} the factorized
adapter, including DoRA, carries a $\mathrm{GL}(r)$ gauge symmetry. At a fixed
point, each reachable Fisher annihilates the at least $r^2$ gauge directions of
$\ker J$ per adapted weight (for $r\le\min(d_{\mathrm{in}},d_{\mathrm{out}})$, as
in all our runs). The scalar statistics are nevertheless \emph{not} constant
along a gauge orbit, because $J$ itself changes along it. In the scalar case
$T(a,b)=ab$ with full-model Fisher $f_c$,
$\Fr_c=f_c\big(\begin{smallmatrix}b^2&ab\\ab&a^2\end{smallmatrix}\big)$ and
$\tr\Fr_c=f_c(a^2+b^2)$: the gauge change $(a,b)\mapsto(ta,b/t)$ preserves the
predictor $ab$ but changes the trace. Traces, Frobenius discrepancies and the
penalty can therefore move along function-preserving directions, and a
comparison of two arms of the same adapter family can partly reflect different
factor scalings. The theorems are unaffected (they are stated in Euclidean
adapter coordinates), and the definiteness conditions of
Section~\ref{sec:theory} hold only for horizontal additive updates, whereas the
optimizer below is not projected onto $\mathcal{H}_\phi$; by
Remark~\ref{rem:not-a-quotient} a non-horizontal step can move the gap the other
way at second order, particularly at the rank-deficient $B=0$ initialization. The
audit below removes the $\mathrm{GL}(r)$ freedom by evaluating each checkpoint at
its \emph{balanced} representative. First,
$d_{ab}$ is a $p$-dimensional matrix discrepancy estimated from $n_a+n_b$
gradients, so it is rank-deficient whenever $n_a+n_b<p$; the estimate is then of
the discrepancy restricted to the span of the sampled gradients, and balanced
batches matter more for \eqref{eq:mit-gram} than for the trace. Second,
Lemma~\ref{lem:projection} implies $d_{ab}$ is not invariant to general
reparameterization, nor even to isotropic rescaling: under $K=cI$ it scales as
$d_{ab}\mapsto c^2d_{ab}$, and only the normalized $\tilde e_{ab}=d_{ab}/s_{ab}$
of Equation~\ref{eq:mit-frob-reg} is invariant there. It may therefore be
compared across subgroups within a fixed adapter but not across adapter
families; the inertia of $\Fr_a-\Fr_b$ is a coordinate-invariant alternative
that we do not report here. A training arm built on $R_{\mathrm F}$ should use the
identical sampler, initialization, optimizer, schedule and update budget as the
trace arm (Section~\ref{sec:method-arms}), so that their difference isolates the
statistic.

\paragraph{Audit protocol.} The audit measures what the trace penalty did to the
matrix it cannot control. For each of the 30 primary cells we load the final
\texttt{plain} and trace-matching checkpoints and, in evaluation mode, draw the
\emph{same} 20 group-balanced held-out batches of $n=64$ examples per subgroup for
both.

\emph{Gauge canonicalization.} Before measuring, every factorized layer is replaced
by its balanced representative: with the rank-$r$ SVD $BA=U\Sigma V^\top$, set
$B\leftarrow U\Sigma^{1/2}$ and $A\leftarrow\Sigma^{1/2}V^\top$, so that
$B^\top B=AA^\top$; any DoRA magnitude is left unchanged. This preserves the
predictor, and balanced factorizations of a full-rank product differ only by an
orthogonal $G\in\mathrm O(r)$, which acts on the adapter coordinates as an
orthogonal map; traces and Frobenius norms of the reachable Fishers are invariant
under it. The audit statistics are therefore invariant to the $\mathrm{GL}(r)$
gauge, and \texttt{validate\_exact.py} checks this on a small network
(Appendix~\ref{app:exact}). We also report the unbalanced values.

\emph{Statistics.} Let $u_{ai}$ be the observed-label per-example gradient of
example $i$ of group $a$ with respect to the trained parameters, $K$ their Gram
matrix, and $s_{ab}=\tfrac12(\tr\widehat F^{\mathrm{R},\mathrm{obs}}_a
+\tr\widehat F^{\mathrm{R},\mathrm{obs}}_b)$. For every subgroup pair we compute
the trace gap, the plug-in Frobenius discrepancy $d_{ab}^2$ of
\eqref{eq:mit-gram}, and the unbiased estimator of the squared discrepancy
\[
\widehat d^{\,2}_{ab,\mathrm U}
=\frac{\sum_{i\ne j}(u_{ai}^\top u_{aj})^2}{n_a(n_a-1)}
+\frac{\sum_{i\ne j}(u_{bi}^\top u_{bj})^2}{n_b(n_b-1)}
-\frac{2\sum_{i,j}(u_{ai}^\top u_{bj})^2}{n_an_b},
\]
whose expectation, for independent draws, is the squared Frobenius distance
between the two groups' population observed-label reachable Fishers (the plug-in value adds a positive per-example noise term that
can differ between arms). It can be negative in finite samples; we report the
normalized \emph{squared} value $\widehat d^{\,2}_{ab,\mathrm U}/s_{ab}^2$ directly,
without clipping or taking a square root. The nonzero eigenvalues of
$\widehat F^{\mathrm{R},\mathrm{obs}}_a-\widehat F^{\mathrm{R},\mathrm{obs}}_b=G^\top DG$,
with $D=\operatorname{diag}(n_a^{-1}\mathbf 1,-n_b^{-1}\mathbf 1)$, equal those of
the small matrix $K^{1/2}DK^{1/2}$, so the operator norm on the span of the
sampled gradients costs one $(n_a{+}n_b)$-dimensional eigendecomposition. From
the same quantities we obtain the surrogate-loss gap from the per-example losses.

\emph{Analysis.} All quantities are compared between the two arms of a cell only.
The primary audit statistic, fixed before measurement, is
$\widehat d^{\,2}_{ab,\mathrm U}/s_{ab}^2$ averaged over pairs; the normalized
trace gap is the manipulation check, the operator-to-Frobenius ratio measures how tight the Frobenius bound is, and the surrogate-loss gap connects the audit to the loss channel of Proposition~\ref{prop:expansion}.
Cells are compared pairwise by win counts, median paired changes with a cell bootstrap,
and sign tests; across cells we report Spearman correlations of the change in
trace gap with the change in Frobenius discrepancy, and of the change in
Frobenius discrepancy with the changes in surrogate-loss and performance gaps. All
are read descriptively, like the rest of Section~\ref{sec:mitigation-results}.
\texttt{measure\_frobenius.py} implements the audit and checks every identity
above against dense matrices.

\subsection{Adaptive subgroup weights}

A uniform penalty would apply the same strength to every subgroup throughout training. We instead adjust the subgroup weights so that groups with consistently high sensitivity receive greater emphasis. This follows the general idea of group-specific weighting used in constrained fairness and group-robust optimization \citep{agarwal2018reductions, sagawa2019distributionally}.

Batch estimates of $c_a$ can vary from one step to the next. We therefore maintain an exponential moving average $\bar c_a$ for each subgroup. The current batch is used to update the model parameters, while the moving averages determine how the subgroup weights change. This allows the weights to respond to persistent differences rather than to a single batch.

Let $\bar c_{\mathrm{ref}}$ be the mean reference computed from the moving averages, matching the reference used by the penalty. The relative gap and exponentiated update are
\begin{equation}
\nu_a=\operatorname{clip}\!\left(
\frac{\bar c_a-\bar c_{\mathrm{ref}}}{\max\{\bar c_{\mathrm{ref}},\epsilon_{\mathrm{den}}\}},-\nu_{\max},\nu_{\max}\right),
\qquad
\log\mu\leftarrow\Pi_\tau(\log\mu+\eta_\mu\nu).
\label{eq:mit-dual}
\end{equation}
The relative gap $\nu_a$ (we avoid $\gamma$, which denotes a Bayes-risk difference in Section~\ref{sec:theory}) compares the smoothed sensitivity of subgroup $a$ with the smoothed subgroup mean. If a subgroup remains above the mean, $\nu_a$ is positive and its weight increases. If it remains below the mean, its weight decreases. Therefore, the penalty does not treat all subgroup deviations at a fixed rate; it places more emphasis on groups whose high sensitivity persists over training.

The update is performed in log-weight space, following exponentiated or mirror-ascent weighting \citep{beck2003mirror}. The operator $\Pi_\tau$ centers the log-weights and clamps them to $[-\tau,\tau]$, and $\mu=\operatorname{softmax}(\log\mu)$ converts them into non-negative weights that sum to one. These two operations prevent the weighting from becoming concentrated on only one subgroup.

We use moving-average coefficient $0.9$, an epoch-level log-weight rate of $0.3$, gap bound $\nu_{\max}=1$, and trust-region bound $\tau=3$. The per-step value of $\eta_\mu$ is derived from the number of batches in an epoch, making the rate comparable across datasets and batch sizes.

\subsection{Why the task and mitigation terms are combined}

For one training step, the intended update combines the pooled task loss with the curvature regularizer:
\begin{equation}
\mathcal L_{\mathrm{step}}(\phi)
=\mathcal L_{\mathrm{task}}(\phi)+\omega R(\phi,\mu).
\label{eq:mit-objective}
\end{equation}
Neither term is sufficient alone. The pooled task loss is required to learn the predictor, but it constrains only average fit: two parameter settings can have similar pooled loss while having very different subgroup gradient norms, and groups that generate larger gradients can dominate its update even in a count-balanced batch. This is precisely the degree of freedom targeted by $R$. Conversely, $R$ alone says nothing about fitting $y$. It is minimized whenever the subgroup statistics match, including degenerate
solutions that are uniformly insensitive and inaccurate; with a symmetric mean
reference the trace form can also match by raising below-mean traces. The
Frobenius form of Equation~\ref{eq:mit-gram} shares this property (it can be
reduced by raising either group's matrix) and differs from the trace form
only in that equal traces no longer imply a zero penalty. The task loss therefore anchors predictive utility, while the mitigation term selects among useful descent directions by discouraging unequal subgroup sensitivity.

The two gradients are computed separately. Differentiating $t_i=\lVert\nabla_\phi\ell_i\rVert^2$ introduces Hessian--vector products, and its gradient can be much larger than the ordinary task gradient even when the penalty value is small. In an uncapped development run, the penalty gradient was about 15 times the task-gradient norm; clipping only the summed gradient then effectively removed the task direction and balanced accuracy remained at chance ($0.500$). We therefore set
\begin{equation}
\omega=\begin{cases}\min\!\left(1,
\beta\dfrac{\lVert g_{\mathrm{task}}\rVert_2}
{\lVert g_{\mathrm{pen}}\rVert_2}\right), & g_{\mathrm{pen}}\ne0,\\[4pt]
1, & g_{\mathrm{pen}}=0,\end{cases}
\qquad
g=g_{\mathrm{task}}+\omega g_{\mathrm{pen}},
\qquad \beta=0.5.
\label{eq:mit-balance}
\end{equation}
(We write $\omega$ rather than $s$ to avoid a clash with the adapter scale of
Proposition~\ref{prop:defect}.)
This caps the curvature-gradient norm at one half of the task-gradient norm while preserving the relative subgroup weights. It also gives a direct first-order safeguard. For $\beta<1$,
\begin{equation}
g_{\mathrm{task}}^\top g
\ge (1-\beta)\lVert g_{\mathrm{task}}\rVert_2^2>0,
\label{eq:mit-descent-safeguard}
\end{equation}
unless the task gradient is zero. Hence the negative combined gradient cannot reverse the task descent direction before AdamW preconditioning; the mitigation term can rotate that direction, but not overwhelm it. The combined gradient is subsequently clipped to norm $1$ before the AdamW update. This mechanism explains why the two terms can operate together; whether they improve held-out utility or fairness remains an empirical question rather than a consequence of the bound.

The image models require chunked computation. The implementation first computes the trace values without retaining a second-order graph, uses them to obtain $c_a$ and the derivative of $R$ with respect to each $c_a$, and then recomputes the weighted trace in small chunks. Because $c_a$ is a mean, each example receives the detached weight $(\partial R/\partial c_a)/n_a$. Summing the chunk gradients therefore gives the same penalty gradient as a full-batch computation; chunking changes peak memory, not the objective. In the reported chunked runs, every constrained row in the batch contributes to the statistic. The task loss also uses every row, including rows outside the constraint set.

\subsection{Balanced batches and comparison arms}
\label{sec:method-arms}

Ordinary shuffled batches can contain very few examples from a rare group, making $c_a$ far noisier for that group than for common groups. The group-balanced sampler draws the same number of examples from every constrained subgroup, with replacement when necessary. If unconstrained categories exist, an additional slice from those categories is included for the task loss. Because this sampler changes how often examples are seen, an epoch is a fixed number of balanced draws rather than one pass through the dataset.

The protocol specifies training arms sharing a sampler, schedule, seed, and model
initialization. The \texttt{plain} arm uses only the pooled task loss. The
\texttt{curvature} arm uses the trace statistic of Equation~\ref{eq:mit-ca}.
Holding the sampler fixed across arms isolates the effect of the subgroup
statistic from the effect of oversampling. The pairwise discrepancy of
Equation~\ref{eq:mit-gram} is used as the matrix audit of these arms; it could
also replace the trace statistic in a \texttt{frobenius} training arm, which
would isolate the effect of penalizing a matrix discrepancy rather than a trace
(Theorem~\ref{thm:trace}, Corollary~\ref{cor:closure}).

\paragraph{What this submission reports.} Section~\ref{sec:mitigation-results}
compares the \texttt{plain} and \texttt{curvature} (Huber trace) arms on a
30-cell grid and audits both with the Frobenius discrepancy.
Theorem~\ref{thm:bayesfloor} suggests that any matching
intervention should move measured gaps most where excess risk is large relative
to the Bayes-risk difference (under-fitted or shifted subgroups) rather
than in the in-distribution settings evaluated here.

Paired runs use AdamW with learning rate $5\times10^{-4}$, no weight decay,
gradient clipping at $1$, and seed 12345. They are scheduled for at most 60
epochs with patience 10, and the retained checkpoint is selected by the pooled
validation metric rather than by a fairness gap. Batch size is tuned per cell as
a hyperparameter, over values between 64 and 1024; the selected value is 64--512
for 29 of the 30 cells and 1024 for UTKFace ViT--DoRA, and every selected value
is reported in Table~\ref{tab:primary-cells} of Appendix~\ref{app:mitigation}.

\subsection{Post-training evaluation}

Post-training evaluation is necessary because the training objective establishes only an optimization-space intervention. It serves four separate checks. First, the matrix audit of Section~\ref{sec:frobenius} tests on held-out batches whether the quantity acted on during training (the trace) and the matrix discrepancy it cannot control actually became more similar across groups. Second, task performance, both its mean and its worst subgroup, tests the failure mode in which traces match because useful learning collapsed. Third, direct outcome metrics such as subgroup performance range, equal-opportunity difference, and average-odds difference test fairness in prediction space; these cannot be inferred from $R$. Fourth, comparison with the \texttt{plain}
arm, which shares the balanced sampler, determines whether a change is
attributable to curvature matching rather than to oversampling alone.
Section~\ref{sec:mitigation-results}
reports these checks for the \texttt{plain} and \texttt{curvature}
arms (the outcome metrics partly in Appendix~\ref{app:mitigation}); a
\texttt{frobenius} training arm is a natural extension.

The training penalty uses the observed-label trace because that is the quantity
implemented by the differentiable loss-gradient path. Evaluation also computes
the expected-Fisher trace of Equation~\ref{eq:trace-estimator}, as an
out-of-objective check: agreement would show that the intervention transfers from
the empirical proxy to the model-distribution geometry studied in the diagnostic
sweep, whereas disagreement reveals proxy mismatch rather than hiding it behind a
shared estimator. A smaller curvature disparity is treated as evidence that the
optimization-space target changed, not as proof that prediction-level
performance gaps improved.

%% file: app_mitigation.tex
\section{Trace-matching experiment: details}
\label{app:mitigation}

\paragraph{Design and analysis.}
\begin{itemize}[leftmargin=*,itemsep=1pt,topsep=2pt]
    \item \emph{Cells.} Each cell is one task, encoder and adapter, trained once
    with group-balanced task-loss training (P) and once with the Huber-penalized
    trace-matching objective (H).
    \item \emph{Primary analysis.} All \mvPn{} cells, each at the batch size
    selected for it as a hyperparameter (Section~\ref{sec:method-arms}).  
    Both
    arms of a pair always share that batch size, which is what makes the pair a
    controlled comparison; the arms are never allowed to differ in it.

    \item \emph{Batch size} Batch size is tuned per cell as a hyperparameter, over values between 64 and 1024; the selected value is 64--512 for 29 of the 30 cells and 1024 for UTKFace ViT--DoRA, and every selected value is reported in Table~\ref{tab:primary-cells}.
    
    \item \emph{Sign convention.} Throughout,
    $R_{\mathrm{gap}}=100\,(G_{\mathrm P}-G_{\mathrm H})/G_{\mathrm P}$, where $G$ is the
    test-split best--worst gap; positive favours trace matching.
    \item \emph{Uncertainty.} The \mvPn{} cells share tasks, encoders and adapters,
    so every interval and test below is descriptive.
    \begin{itemize}[leftmargin=*,itemsep=0pt,topsep=0pt]
        \item Cell bootstrap of the median (10{,}000 replicates, seed 0):
        $[\mvPlo,\,\mvPhi]\%$.
        \item Task-blocked bootstrap (resample tasks, then cells within tasks):
        $[\mvBlo,\,\mvBhi]\%$.
        \item Two-sided sign test $p=\mvPsignp$; Wilcoxon test on
        $\log(G_{\mathrm H}/G_{\mathrm P})$ $p=\mvPwilp$.
    \end{itemize}
    \item \emph{Utility quadrants.}
    \begin{itemize}[leftmargin=*,itemsep=0pt,topsep=0pt]
        \item Gap narrows, utility not worse: \mvQboth.
        \item Gap narrows, utility slightly worse: \mvQgaponly.
        \item Gap not narrower, utility not worse: \mvQutilonly.
        \item Worst utility changes: $\mvUclsworst$~pp balanced accuracy
        (\mvUclsworstcell) and $+\mvUregworst$~years MAE (\mvUregworstcell).
    \end{itemize}
    \item \emph{Equalized rates (secondary, classification only).} EOD falls in
    \mvPeodw{} of the \mvPeodn{} classification cells and AOD in \mvPaodw{} of
    \mvPaodn{}. Appendix~\ref{app:eod} explains why the theory makes no
    prediction for these criteria.
    \texttt{mitigation/primary\_analysis/summary.csv}, the same file that
    produces Section~\ref{sec:mitigation-results}
    .
    \item \emph{Mechanism correlation.} The Fisher-CV correlation is reported
        descriptively, without a $p$-value: $\rho=+0.23$ between the relative changes in CV and gap, and the CV decreased in
    25 of 30 cells. The per-cell values, plain and trace matching, are
    in Table~\ref{tab:mit-cv} and in the supplementary file
    \texttt{mechanism\_cv.csv}, recomputed from the stored checkpoints by
    \texttt{mitigation/mechanism\_cv.py}.
    \item \emph{Matrix audit.} Table~\ref{tab:mit-audit} gives the per-cell
    change in the normalized trace gap, unbiased squared Frobenius discrepancy,
    operator norm and surrogate-loss gap, from
    \texttt{mitigation/audit/fisher\_audit.csv} and
    \texttt{loss\_gaps.csv}.
\end{itemize}

\paragraph{Selection of the penalty form.}
Before fixing the primary intervention, we conducted a fixed-batch development
sweep at batch size 256 across all datasets. We compared the mean-referenced
Huber penalty in Equation~\eqref{eq:mit-regularizer} with an $\ell_1$ penalty,
a squared $\ell_2$ penalty, a one-sided squared-hinge penalty, a variant using
the minimum subgroup trace rather than the mean as the reference, and a
raw-deviation variant. The best individual alternative was
dataset dependent: $\ell_1$ produced the smallest best--worst task-performance
gap on some datasets, whereas squared $\ell_2$ performed best on others. Huber
was the only penalty form that consistently reduced the gap relative to the
plain arm across all datasets in this sweep. We therefore fixed Huber for the
primary paired experiment. This choice provides a robust compromise between
the two behaviors: with $\kappa=1$, the Huber loss is quadratic for
$|e_a|\leq 1$, like a squared $\ell_2$ penalty near the matching point, and
linear for $|e_a|>1$, like an $\ell_1$ penalty for large deviations. The
fixed-batch sweep was used to select the penalty form; the subsequently
reported primary comparison fixes Huber and selects batch size separately for
each dataset--encoder--adapter cell as described in
Section~\ref{sec:method-arms}.

\begin{table}[ht]
\caption{Primary analysis by task and adapter, each cell at its selected batch
size. Image rows pool the three encoders; each diabetes row is its single
FT-Transformer cell, so its counts are out of one. Gap wins count cells with
$R_{\mathrm{gap}}>0$. Plain and trace gaps are means over the cells in the row.
``Gap and utility'' counts cells that narrow the gap without a loss of pooled
utility. $\Delta$ utility is the mean change in balanced accuracy (percentage
points, higher is better) or MAE (years, lower is better).}

\label{tab:mit-settings}
\centering
\small
\resizebox{\linewidth}{!}{%
\input{tables/mitigation_settings_rgap_e1}
}
\end{table}

\begin{figure}[ht]
\centering
\includegraphics[width=\linewidth]{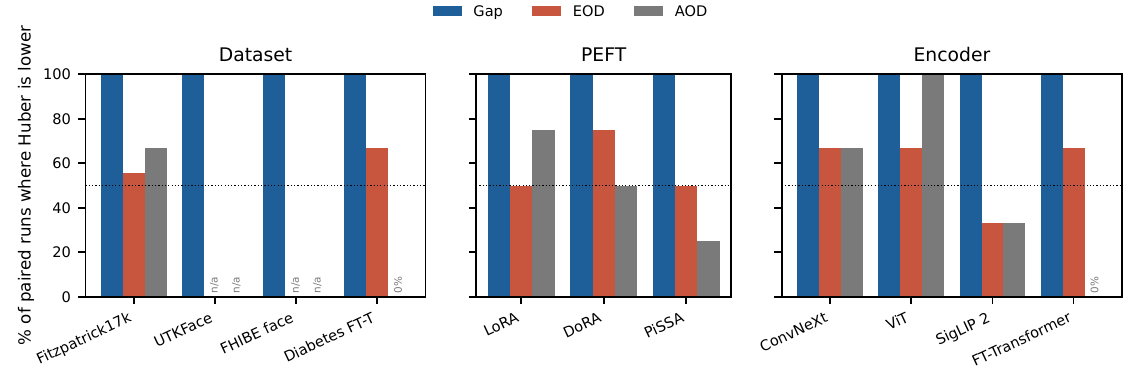}
\caption{Share of the 30 cells, each at its selected batch size, in which trace
matching lowers the gap, EOD or AOD relative to the plain arm, by dataset,
adapter and encoder. EOD and AOD are undefined for the regression tasks (n/a);
0\% marks a measured value that never favoured trace matching. EOD and AOD count
the 12 classification cells only. The dotted line marks 50\%.}
\label{fig:mit-winrates}
\end{figure}

\begin{figure}[ht]
\includegraphics[width=0.9\linewidth]{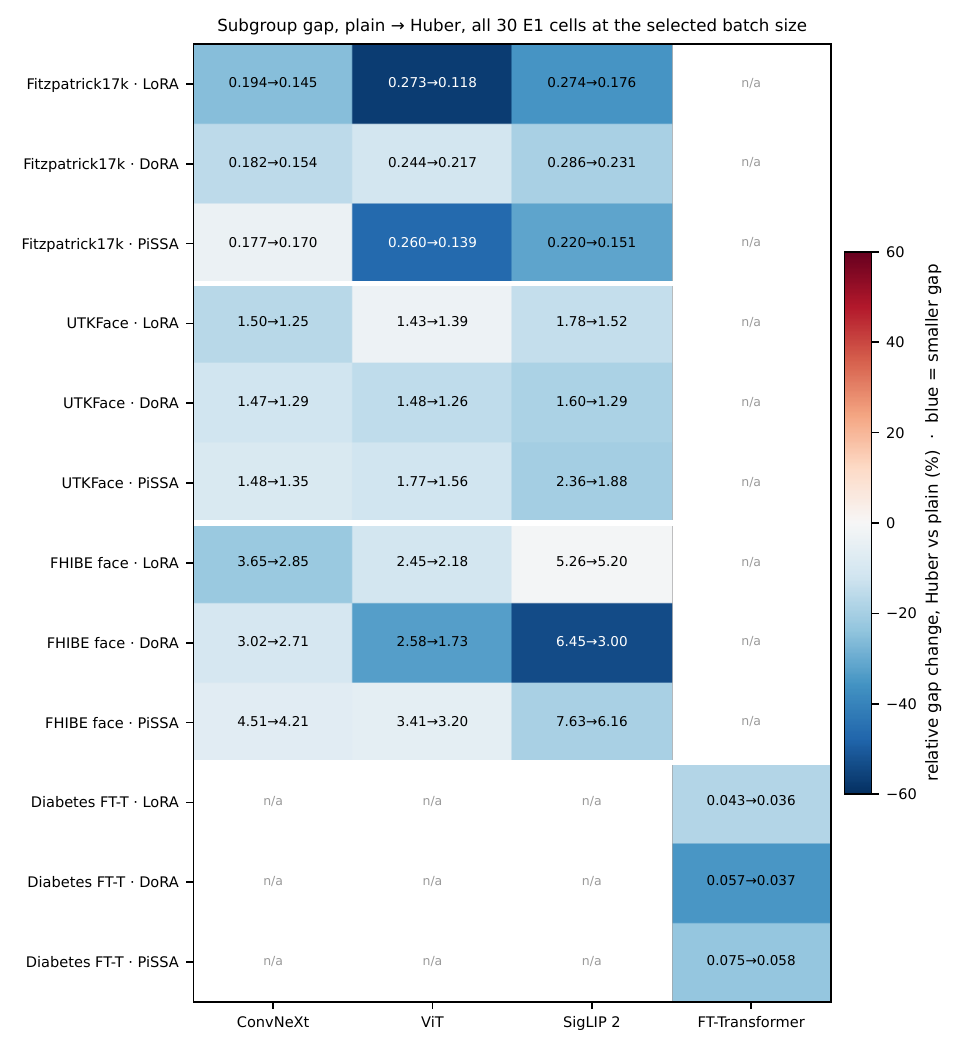}
\caption{Best--worst gap, plain$\to$trace matching, for all 30 cells, each at
its selected batch size. Colour is $R_{\mathrm{gap}}$; blue means trace matching
narrows the gap. The image encoders do not apply to diabetes and the
FT-Transformer does not apply to the image tasks (n/a).}

\label{fig:mit-heatmap}
\end{figure}

\begin{figure}[h]
\centering
\includegraphics[width=\linewidth,trim=0 0 0 19.5,clip]{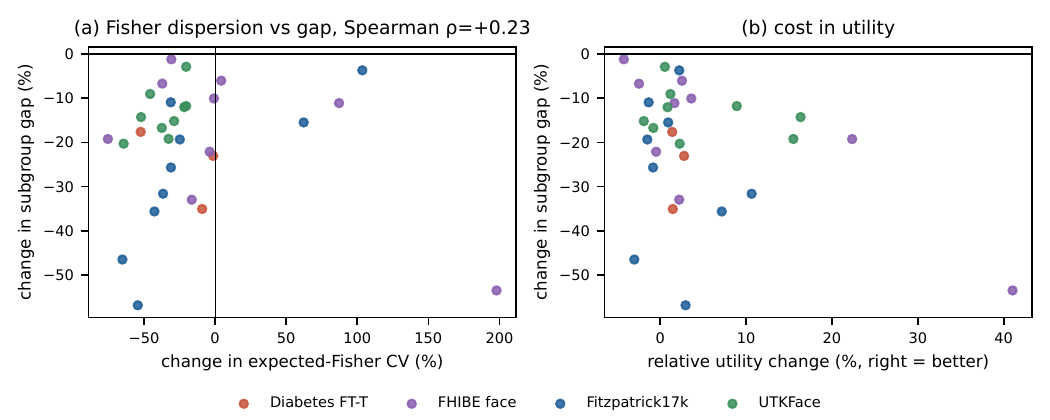}
\caption{Primary cells, from the experiment pipeline. The axes use a
\emph{change} convention, $\Delta G=100(G_{\mathrm H}-G_{\mathrm P})/G_{\mathrm P}=-R_{\mathrm{gap}}$
and likewise for the CV, so negative values are reductions. Left: $\Delta$CV
against $\Delta G$ over all 30 cells (Spearman $\rho=+0.23$; no $p$-value is
reported because the cells share tasks, encoders and adapters). Right: relative
change in pooled utility (right is better) against $\Delta G$. The per-cell CV
values are in Table~\ref{tab:mit-cv}.}
\label{fig:mit-mechanism}
\end{figure}

Batch size is a per-cell hyperparameter rather than a
fixed constant, so it is reported in Table~\ref{tab:primary-cells}; both arms
of a pair always share it, which is what makes the pair a controlled
comparison. Figures~\ref{fig:primary-overview-fitzpatrick17k}--\ref{fig:primary-overview-diabetes}
draw the same 30 cells, one figure per dataset, each panel one encoder $\times$ PEFT cell.

\clearpage
\input{tables/primary_cells}
\begin{figure}[p]
\centering
\includegraphics[width=\textwidth]{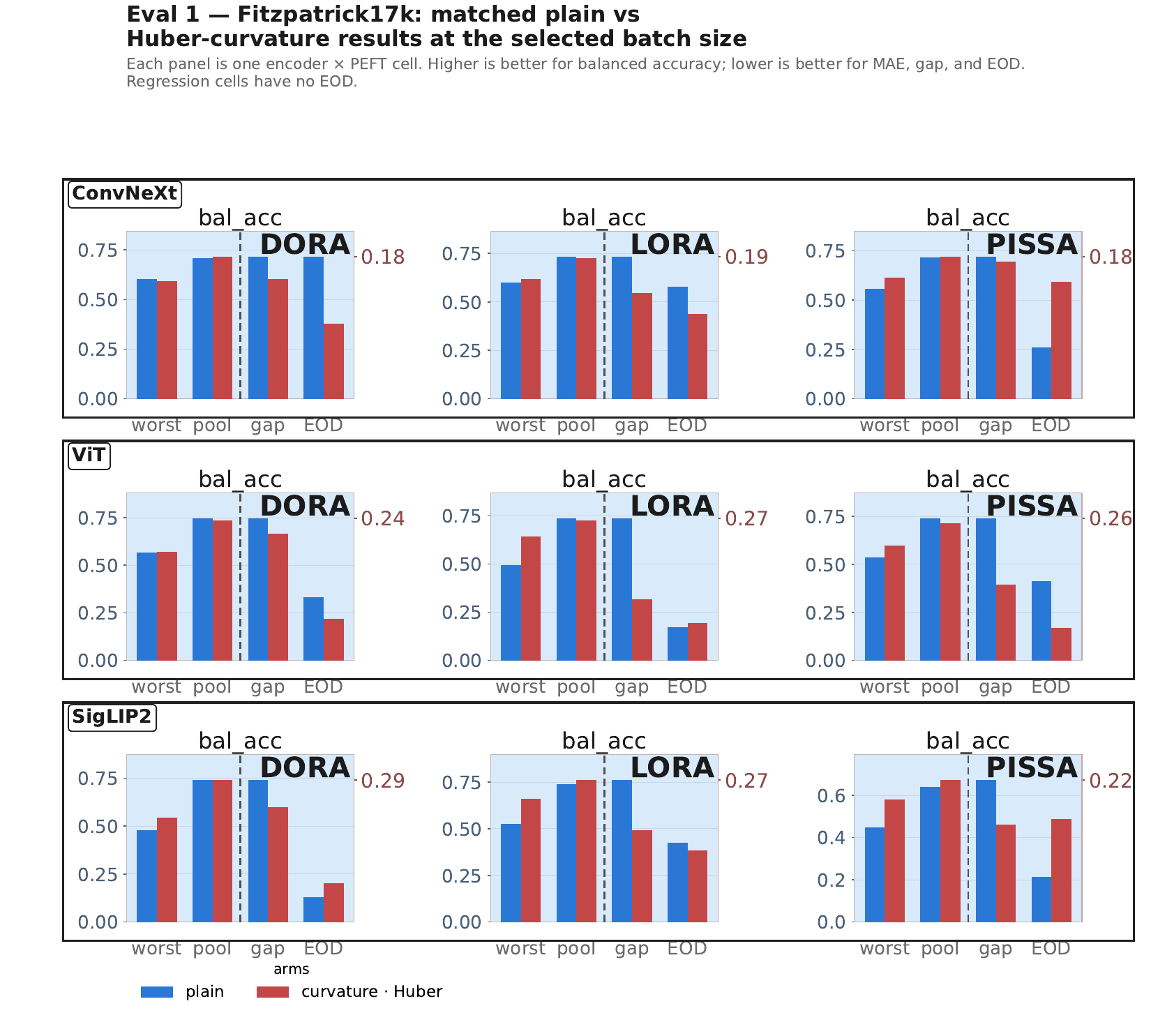}
\caption{Fitzpatrick17k primary cells: matched plain versus Huber-curvature
results at the selected batch size. Rows are the three encoders, columns the
three PEFT methods. Higher is better for balanced accuracy; lower is better
for gap and EOD.}
\label{fig:primary-overview-fitzpatrick17k}
\end{figure}

\begin{figure}[p]
\centering
\includegraphics[width=\textwidth]{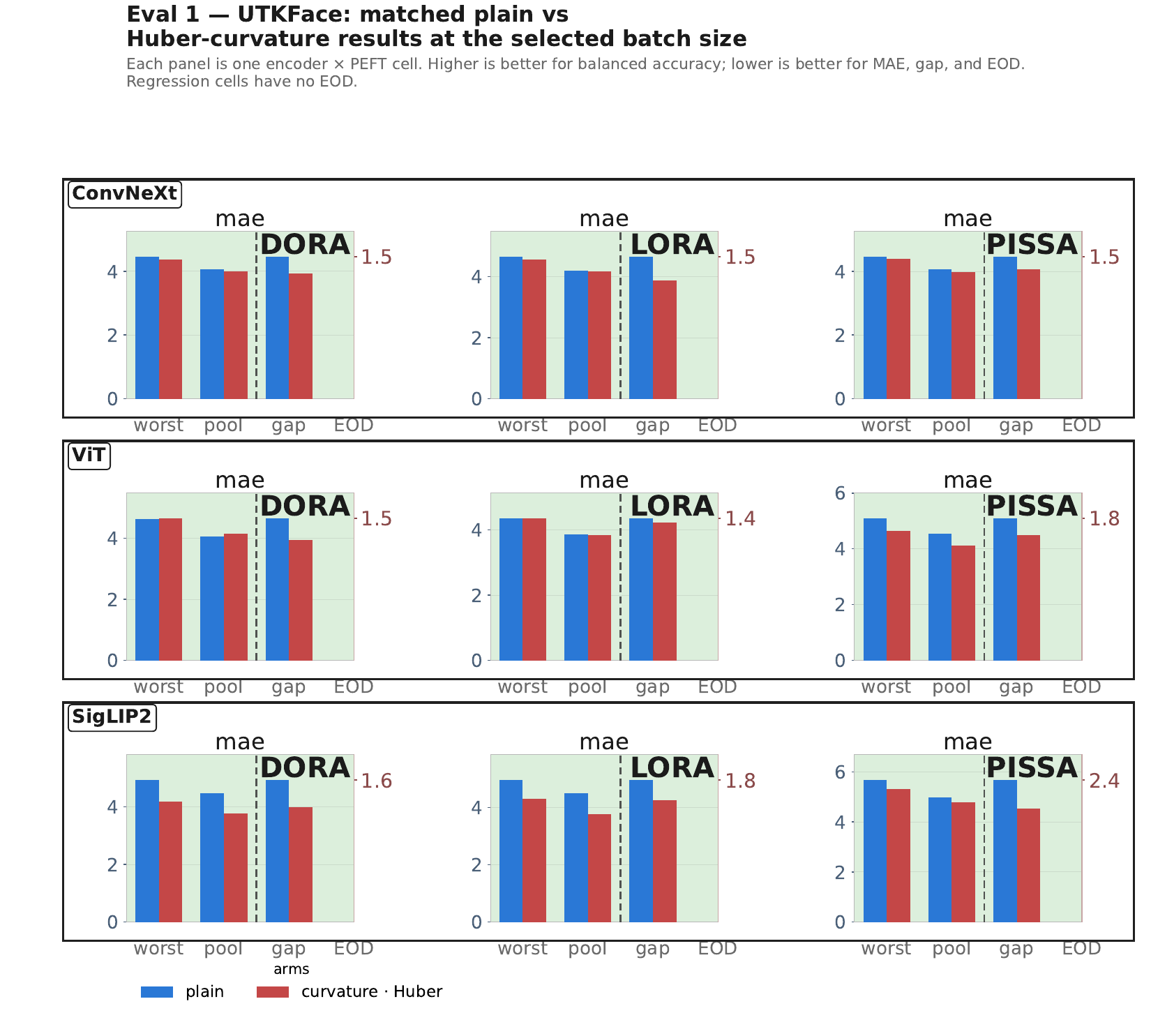}
\caption{UTKFace primary cells: matched plain versus Huber-curvature results
at the selected batch size. Rows are the three encoders, columns the three
PEFT methods. Higher is better for balanced accuracy; lower is better for
MAE and gap. UTKFace is a regression task, so EOD is undefined.}
\label{fig:primary-overview-utkface}
\end{figure}

\begin{figure}[p]
\centering
\includegraphics[width=\textwidth]{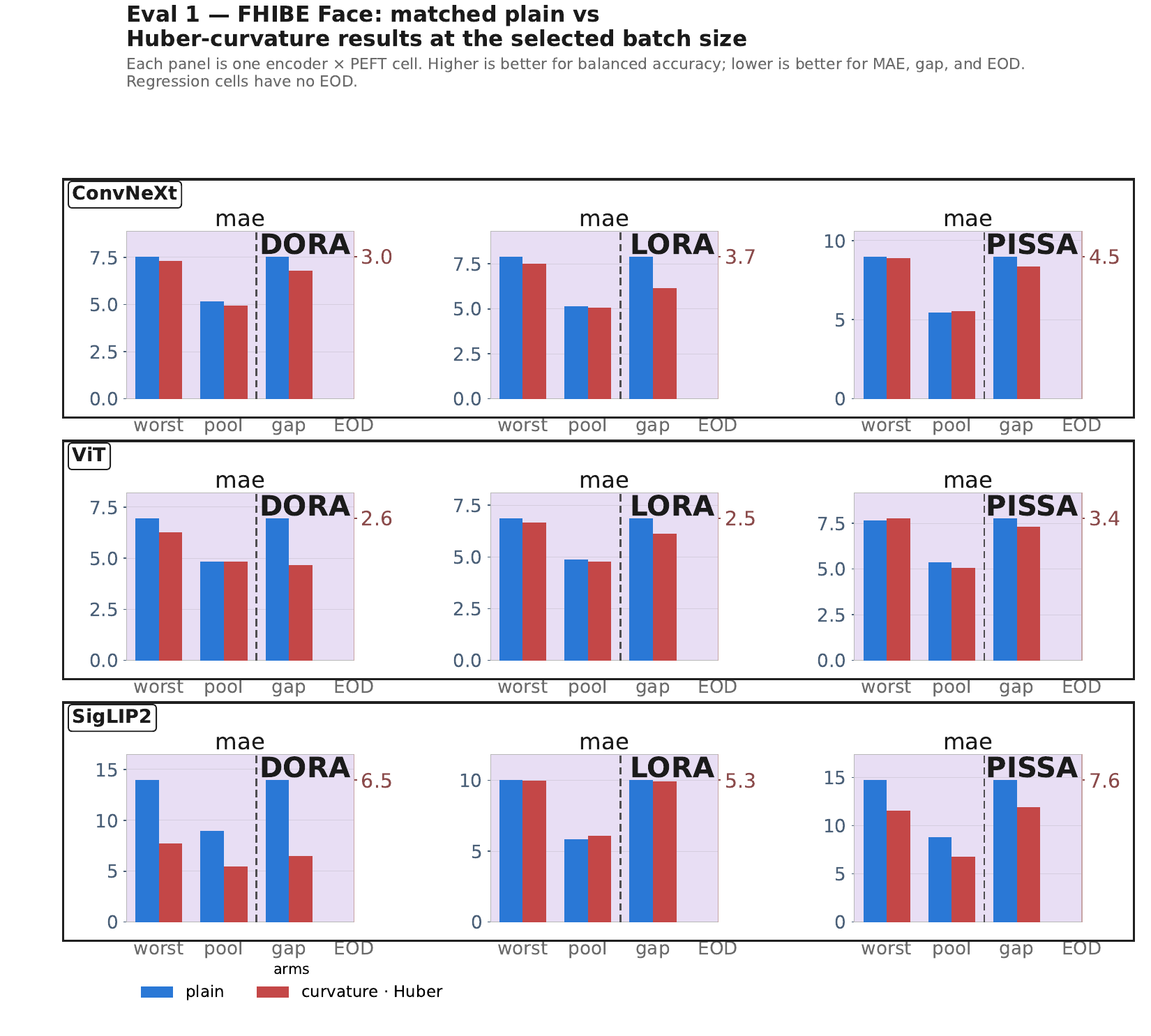}
\caption{FHIBE Face primary cells: matched plain versus Huber-curvature
results at the selected batch size. Rows are the three encoders, columns the
three PEFT methods. Higher is better for balanced accuracy; lower is better
for MAE and gap. FHIBE Face is a regression task, so EOD is undefined.}
\label{fig:primary-overview-fhibe-face}
\end{figure}

\begin{figure}
\centering
\includegraphics[width=\textwidth]{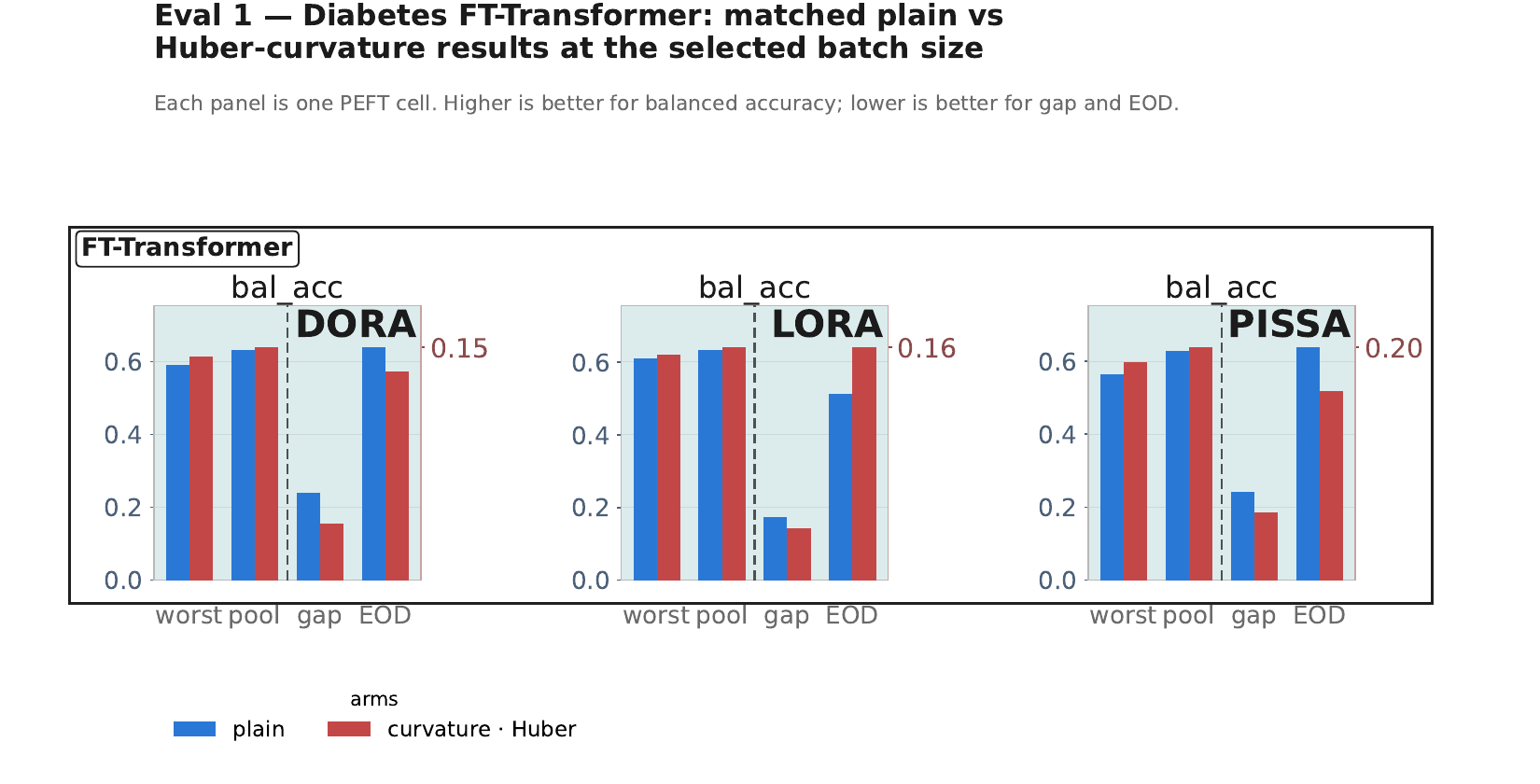}
\caption{Diabetes FT-Transformer primary cells: matched plain versus
Huber-curvature results at the selected batch size, one panel per PEFT
method. Higher is better for balanced accuracy; lower is better for gap and
EOD.}
\label{fig:primary-overview-diabetes}
\end{figure}

\subsection{Cost of adapter versus full-model penalty scope}
\label{app:scope}

The penalty differentiates each per-example loss with respect to the adapter and
head only (\emph{adapter} scope). The alternative differentiates with respect to
every parameter, including the frozen backbone (\emph{full} scope).
\begin{itemize}[leftmargin=*,itemsep=1pt,topsep=2pt]
    \item \emph{Setup.} LoRA, DoRA and PiSSA on SigLIP~2 (375M parameters) and
    ViT (87M parameters), on FHIBE face and UTKFace, with the Huber penalty and
    identical hyperparameters, so scope is the only difference within a row.
    \item \emph{Time.} Full scope is slower in all 12 pairs, by 41--62\%. The
    increase is smallest for DoRA (41--43\%) and largest for PiSSA (59--62\%),
    and it is stable across the two datasets, which differ in size by roughly
    a factor of two.
    \item \emph{Memory.} Full scope raises peak allocation from
    $6.4$--$6.9$\,GB to $13.5$--$14.2$\,GB for LoRA and PiSSA, roughly doubling
    it, and from $12.7$--$13.0$\,GB to $20.6$--$21.1$\,GB for DoRA. The penalty
    must retain activations for the frozen backbone under full scope, so the
    increase is in the backbone's activation memory rather than in parameter
    state; it is therefore independent of dataset size.
    \item \emph{Utilization.} Full scope also runs at higher GPU utilization
    ($94.5$--$97.7\%$ against $89.9$--$95.6\%$), so the extra time is additional
    work rather than a stall.
    \item \emph{Scope of this comparison.} These runs measure cost only. They do
    not test whether full scope changes subgroup gaps or task quality, so the
    choice of adapter scope in the main experiments rests on cost together with
    the identifiability argument of Lemma~\ref{lem:projection}, not on a
    measured quality comparison.
\end{itemize}

\begin{table}[ht]
\caption{Measured cost of adapter versus full penalty scope on a primary evaluation device NVIDIA RTX 5090. Each entry is the
mean $\pm$ standard deviation over repeated timed runs. Train time is seconds
per run; $\Delta t$ is the increase from adapter to full scope. Peak memory is
the maximum allocated during the run, and is independent of the dataset because
the batch composition is fixed, so the two datasets repeat the same value.
GPU utilization is the mean over the run. This table reports cost only; it makes
no quality comparison between the two scopes.}
\label{tab:scope-time}
\centering
\small
\resizebox{\linewidth}{!}{%
\begin{tabular}{llrrrrrrr}
\toprule
 &  &  & \multicolumn{3}{c}{train time (s)} & \multicolumn{2}{c}{peak mem (GB)} & GPU util (\%) \\
\cmidrule(lr){4-6}\cmidrule(lr){7-8}\cmidrule(lr){9-9}
Encoder & Dataset & PEFT & adapter & full & $\Delta t$ & adapter & full & adapter / full \\
\midrule
ViT      & FHIBE face & LoRA  & $63.30\pm0.21$  & $101.30\pm0.34$ & $+60\%$ & 6.42  & 13.55 & $92.2$ / $95.1$ \\
         &            & DoRA  & $87.41\pm0.32$  & $123.74\pm0.58$ & $+42\%$ & 12.69 & 20.55 & $90.7$ / $94.8$ \\
         &            & PiSSA & $62.55\pm0.44$  & $100.67\pm0.32$ & $+61\%$ & 6.39  & 13.53 & $92.2$ / $95.0$ \\
\cmidrule(l){2-9}
         & UTKFace    & LoRA  & $155.35\pm0.49$ & $248.97\pm0.37$ & $+60\%$ & 6.42  & 13.55 & $95.6$ / $97.3$ \\
         &            & DoRA  & $212.89\pm0.52$ & $304.80\pm0.57$ & $+43\%$ & 12.69 & 20.55 & $92.8$ / $96.2$ \\
         &            & PiSSA & $152.04\pm0.50$ & $246.57\pm0.32$ & $+62\%$ & 6.39  & 13.53 & $95.4$ / $97.7$ \\
\midrule
SigLIP 2 & FHIBE face & LoRA  & $68.78\pm0.49$  & $108.37\pm0.66$ & $+58\%$ & 6.88  & 14.16 & $92.6$ / $95.0$ \\
         &            & DoRA  & $93.30\pm0.60$  & $131.24\pm0.49$ & $+41\%$ & 12.97 & 21.13 & $89.9$ / $94.5$ \\
         &            & PiSSA & $67.41\pm0.43$  & $107.33\pm0.50$ & $+59\%$ & 6.84  & 14.14 & $92.7$ / $95.2$ \\
\cmidrule(l){2-9}
         & UTKFace    & LoRA  & $168.15\pm0.64$ & $266.36\pm0.28$ & $+58\%$ & 6.88  & 14.16 & $95.4$ / $97.6$ \\
         &            & DoRA  & $227.19\pm0.26$ & $323.24\pm0.33$ & $+42\%$ & 12.97 & 21.13 & $93.3$ / $96.4$ \\
         &            & PiSSA & $164.25\pm0.44$ & $263.99\pm0.27$ & $+61\%$ & 6.84  & 14.14 & $95.1$ / $97.6$ \\
\bottomrule
\end{tabular}
}
\end{table}
\begin{table}[ht]
\caption{Measured cost of adapter versus full penalty scope, on a second
evaluation device (NVIDIA GB10). Each entry is the mean $\pm$ standard
deviation over the three training epochs. Train time is
seconds per epoch; $\Delta t$ is the increase from adapter to full scope. Peak
memory is the maximum allocated during the run, and is independent of the
dataset because the batch composition is fixed, so the two datasets repeat the
same value. GPU utilization is the mean over the run. This table reports cost
only; it makes no quality comparison between the two scopes.}
\label{tab:scope-time-device2}
\centering
\small
\resizebox{\linewidth}{!}{%
\begin{tabular}{llrrrrrrr}
\toprule
 &  &  & \multicolumn{3}{c}{train time (s)} & \multicolumn{2}{c}{peak mem (GB)} & GPU util (\%) \\
\cmidrule(lr){4-6}\cmidrule(lr){7-8}\cmidrule(lr){9-9}
Encoder & Dataset & PEFT & adapter & full & $\Delta t$ & adapter & full & adapter / full \\
\midrule
ViT      & FHIBE face & LoRA  & $377.83\pm0.30$  & $594.81\pm0.68$  & $+57\%$ & 6.42  & 13.55 & $93.9$ / $94.3$ \\
         &            & DoRA  & $542.34\pm0.87$  & $757.80\pm0.76$  & $+40\%$ & 12.69 & 20.55 & $94.1$ / $94.3$ \\
         &            & PiSSA & $344.20\pm0.72$  & $590.89\pm0.75$  & $+72\%$ & 6.39  & 13.53 & $93.6$ / $94.5$ \\
\cmidrule(l){2-9}
         & UTKFace    & LoRA  & $940.16\pm0.64$  & $1477.05\pm0.49$ & $+57\%$ & 6.42  & 13.55 & $94.6$ / $94.9$ \\
         &            & DoRA  & $1350.09\pm0.69$ & $1883.82\pm1.25$ & $+40\%$ & 12.69 & 20.55 & $94.5$ / $94.6$ \\
         &            & PiSSA & $853.01\pm0.52$  & $1462.13\pm0.70$ & $+71\%$ & 6.39  & 13.53 & $94.3$ / $94.8$ \\
\midrule
SigLIP 2 & FHIBE face & LoRA  & $421.83\pm0.46$  & $657.00\pm0.88$  & $+56\%$ & 6.88  & 14.16 & $93.8$ / $94.6$ \\
         &            & DoRA  & $589.28\pm0.72$  & $818.11\pm0.93$  & $+39\%$ & 12.97 & 21.13 & $94.0$ / $94.4$ \\
         &            & PiSSA & $387.44\pm0.77$  & $652.53\pm0.93$  & $+68\%$ & 6.84  & 14.14 & $93.7$ / $94.4$ \\
\cmidrule(l){2-9}
         & UTKFace    & LoRA  & $1048.41\pm0.57$ & $1656.39\pm0.29$ & $+58\%$ & 6.88  & 14.16 & $94.5$ / $94.9$ \\
         &            & DoRA  & $1456.91\pm0.37$ & $2038.55\pm0.60$ & $+40\%$ & 12.97 & 21.13 & $94.5$ / $94.7$ \\
         &            & PiSSA & $961.78\pm0.64$  & $1630.53\pm0.69$ & $+70\%$ & 6.84  & 14.14 & $94.3$ / $94.8$ \\
\bottomrule
\end{tabular}
}
\end{table}

\clearpage

\input{tables/mit_audit_cells}

\input{tables/mit_cv_cells}

%% file: tables/mitigation_settings_rgap_e1.tex
\begin{tabular}{llrrrrrr}
\toprule
Task & PEFT & gap wins & plain gap & trace gap & median $R_{\mathrm{gap}}$ (\%) & gap and utility & $\Delta$ utility \\
\midrule
Fitzpatrick17k & LoRA & \textbf{3/3} & 0.247 & 0.146 & $+36$ & 1/3 & $+0.1$ \\
Fitzpatrick17k & DoRA & \textbf{3/3} & 0.238 & 0.201 & $+15$ & 2/3 & $-0.0$ \\
Fitzpatrick17k & PiSSA & \textbf{3/3} & 0.219 & 0.153 & $+32$ & 2/3 & $+0.5$ \\
\midrule
UTKFace (MAE, yr) & LoRA & \textbf{3/3} & 1.57 & 1.39 & $+14$ & 3/3 & $-0.26$ \\
UTKFace (MAE, yr) & DoRA & \textbf{3/3} & 1.52 & 1.28 & $+15$ & 2/3 & $-0.22$ \\
UTKFace (MAE, yr) & PiSSA & \textbf{3/3} & 1.87 & 1.59 & $+12$ & 3/3 & $-0.23$ \\
\midrule
FHIBE face (MAE, yr) & LoRA & \textbf{3/3} & 3.79 & 3.41 & $+11$ & 2/3 & $+0.03$ \\
FHIBE face (MAE, yr) & DoRA & \textbf{3/3} & 4.02 & 2.48 & $+33$ & 3/3 & $-1.24$ \\
FHIBE face (MAE, yr) & PiSSA & \textbf{3/3} & 5.19 & 4.53 & $+7$ & 2/3 & $-0.73$ \\
\midrule
Diabetes FT-T & LoRA & \textbf{1/1} & 0.043 & 0.036 & $+18$ & 1/1 & $+0.7$ \\
Diabetes FT-T & DoRA & \textbf{1/1} & 0.057 & 0.037 & $+35$ & 1/1 & $+0.6$ \\
Diabetes FT-T & PiSSA & \textbf{1/1} & 0.075 & 0.058 & $+23$ & 1/1 & $+0.9$ \\
\bottomrule
\end{tabular}

%% file: tables/primary_cells.tex
\begingroup
\scriptsize
\setlength{\tabcolsep}{2.2pt}
\begin{longtable}{@{}lllcllll@{}}
\caption{Primary Eval~1 results for all 30 cells. P and H denote plain and
Huber-curvature training. BS is the selected batch size shared by the two arms
of each pair. Bracketed values are percentage reductions, so positive values
favor Huber. Balanced accuracy is reported for classification and MAE for
regression; regression EOD is undefined (\textemdash).}
\label{tab:primary-cells}\\
\toprule
Dataset & Encoder & PEFT & BS & Worst P$\to$H & Pooled P$\to$H &
Gap P$\to$H [drop] & EOD P$\to$H [drop] \\
\midrule
\endfirsthead

\multicolumn{8}{c}{\tablename\ \thetable\ (continued)}\\
\toprule
Dataset & Encoder & PEFT & BS & Worst P$\to$H & Pooled P$\to$H &
Gap P$\to$H [drop] & EOD P$\to$H [drop] \\
\midrule
\endhead

\midrule
\multicolumn{8}{r}{Continued on next page}\\
\endfoot

\bottomrule
\endlastfoot

Fitzpatrick17k & ConvNeXt & DoRA & 256 &
\ensuremath{0.603\!\to\!0.594} &
\ensuremath{0.708\!\to\!0.716} &
\ensuremath{0.182\!\to\!0.154}\;[\ensuremath{15.5\%}] &
\ensuremath{0.182\!\to\!0.096}\;[\ensuremath{47.4\%}] \\

Fitzpatrick17k & ConvNeXt & LoRA & 256 &
\ensuremath{0.598\!\to\!0.617} &
\ensuremath{0.733\!\to\!0.725} &
\ensuremath{0.194\!\to\!0.145}\;[\ensuremath{25.7\%}] &
\ensuremath{0.153\!\to\!0.116}\;[\ensuremath{24.1\%}] \\

Fitzpatrick17k & ConvNeXt & PiSSA & 256 &
\ensuremath{0.557\!\to\!0.613} &
\ensuremath{0.716\!\to\!0.721} &
\ensuremath{0.177\!\to\!0.170}\;[\ensuremath{3.7\%}] &
\ensuremath{0.064\!\to\!0.145}\;[\ensuremath{-127.7\%}] \\

Fitzpatrick17k & ViT & DoRA & 128 &
\ensuremath{0.566\!\to\!0.570} &
\ensuremath{0.747\!\to\!0.737} &
\ensuremath{0.244\!\to\!0.217}\;[\ensuremath{10.9\%}] &
\ensuremath{0.108\!\to\!0.072}\;[\ensuremath{34.0\%}] \\

Fitzpatrick17k & ViT & LoRA & 128 &
\ensuremath{0.496\!\to\!0.644} &
\ensuremath{0.738\!\to\!0.727} &
\ensuremath{0.273\!\to\!0.118}\;[\ensuremath{56.9\%}] &
\ensuremath{0.063\!\to\!0.072}\;[\ensuremath{-13.8\%}] \\

Fitzpatrick17k & ViT & PiSSA & 128 &
\ensuremath{0.534\!\to\!0.598} &
\ensuremath{0.740\!\to\!0.715} &
\ensuremath{0.260\!\to\!0.139}\;[\ensuremath{46.5\%}] &
\ensuremath{0.145\!\to\!0.059}\;[\ensuremath{59.6\%}] \\

Fitzpatrick17k & SigLIP~2 & DoRA & 256 &
\ensuremath{0.479\!\to\!0.543} &
\ensuremath{0.739\!\to\!0.741} &
\ensuremath{0.286\!\to\!0.231}\;[\ensuremath{19.3\%}] &
\ensuremath{0.050\!\to\!0.078}\;[\ensuremath{-56.8\%}] \\

Fitzpatrick17k & SigLIP~2 & LoRA & 256 &
\ensuremath{0.524\!\to\!0.660} &
\ensuremath{0.739\!\to\!0.762} &
\ensuremath{0.274\!\to\!0.176}\;[\ensuremath{35.6\%}] &
\ensuremath{0.153\!\to\!0.138}\;[\ensuremath{9.6\%}] \\

Fitzpatrick17k & SigLIP~2 & PiSSA & 512 &
\ensuremath{0.448\!\to\!0.580} &
\ensuremath{0.639\!\to\!0.673} &
\ensuremath{0.220\!\to\!0.151}\;[\ensuremath{31.6\%}] &
\ensuremath{0.070\!\to\!0.160}\;[\ensuremath{-127.2\%}] \\

\addlinespace[2pt]

UTKFace & ConvNeXt & DoRA & 64 &
\ensuremath{4.46\!\to\!4.37} &
\ensuremath{4.06\!\to\!4.01} &
\ensuremath{1.47\!\to\!1.29}\;[\ensuremath{12.0\%}] &
\textemdash \\

UTKFace & ConvNeXt & LoRA & 128 &
\ensuremath{4.66\!\to\!4.56} &
\ensuremath{4.20\!\to\!4.17} &
\ensuremath{1.50\!\to\!1.25}\;[\ensuremath{16.7\%}] &
\textemdash \\

UTKFace & ConvNeXt & PiSSA & 64 &
\ensuremath{4.46\!\to\!4.40} &
\ensuremath{4.06\!\to\!3.98} &
\ensuremath{1.48\!\to\!1.35}\;[\ensuremath{9.0\%}] &
\textemdash \\

UTKFace & ViT & DoRA & 1024 &
\ensuremath{4.63\!\to\!4.65} &
\ensuremath{4.06\!\to\!4.15} &
\ensuremath{1.48\!\to\!1.26}\;[\ensuremath{15.2\%}] &
\textemdash \\

UTKFace & ViT & LoRA & 64 &
\ensuremath{4.36\!\to\!4.34} &
\ensuremath{3.86\!\to\!3.84} &
\ensuremath{1.43\!\to\!1.39}\;[\ensuremath{2.9\%}] &
\textemdash \\

UTKFace & ViT & PiSSA & 256 &
\ensuremath{5.09\!\to\!4.65} &
\ensuremath{4.55\!\to\!4.13} &
\ensuremath{1.77\!\to\!1.56}\;[\ensuremath{11.8\%}] &
\textemdash \\

UTKFace & SigLIP~2 & DoRA & 128 &
\ensuremath{4.94\!\to\!4.19} &
\ensuremath{4.47\!\to\!3.77} &
\ensuremath{1.60\!\to\!1.29}\;[\ensuremath{19.2\%}] &
\textemdash \\

UTKFace & SigLIP~2 & LoRA & 128 &
\ensuremath{4.97\!\to\!4.29} &
\ensuremath{4.49\!\to\!3.77} &
\ensuremath{1.78\!\to\!1.52}\;[\ensuremath{14.3\%}] &
\textemdash \\

UTKFace & SigLIP~2 & PiSSA & 128 &
\ensuremath{5.68\!\to\!5.32} &
\ensuremath{4.97\!\to\!4.78} &
\ensuremath{2.36\!\to\!1.88}\;[\ensuremath{20.3\%}] &
\textemdash \\

\addlinespace[2pt]

FHIBE Face & ConvNeXt & DoRA & 128 &
\ensuremath{7.54\!\to\!7.32} &
\ensuremath{5.16\!\to\!4.94} &
\ensuremath{3.02\!\to\!2.71}\;[\ensuremath{10.1\%}] &
\textemdash \\

FHIBE Face & ConvNeXt & LoRA & 128 &
\ensuremath{7.90\!\to\!7.50} &
\ensuremath{5.15\!\to\!5.06} &
\ensuremath{3.65\!\to\!2.85}\;[\ensuremath{22.1\%}] &
\textemdash \\

FHIBE Face & ConvNeXt & PiSSA & 256 &
\ensuremath{8.98\!\to\!8.88} &
\ensuremath{5.47\!\to\!5.55} &
\ensuremath{4.51\!\to\!4.21}\;[\ensuremath{6.7\%}] &
\textemdash \\

FHIBE Face & ViT & DoRA & 256 &
\ensuremath{6.95\!\to\!6.27} &
\ensuremath{4.84\!\to\!4.83} &
\ensuremath{2.58\!\to\!1.73}\;[\ensuremath{32.9\%}] &
\textemdash \\

FHIBE Face & ViT & LoRA & 128 &
\ensuremath{6.87\!\to\!6.68} &
\ensuremath{4.86\!\to\!4.78} &
\ensuremath{2.45\!\to\!2.18}\;[\ensuremath{11.1\%}] &
\textemdash \\

FHIBE Face & ViT & PiSSA & 128 &
\ensuremath{7.64\!\to\!7.77} &
\ensuremath{5.36\!\to\!5.05} &
\ensuremath{3.41\!\to\!3.20}\;[\ensuremath{6.0\%}] &
\textemdash \\

FHIBE Face & SigLIP~2 & DoRA & 64 &
\ensuremath{13.96\!\to\!7.74} &
\ensuremath{8.94\!\to\!5.46} &
\ensuremath{6.45\!\to\!3.00}\;[\ensuremath{53.5\%}] &
\textemdash \\

FHIBE Face & SigLIP~2 & LoRA & 256 &
\ensuremath{10.03\!\to\!10.00} &
\ensuremath{5.82\!\to\!6.08} &
\ensuremath{5.26\!\to\!5.20}\;[\ensuremath{1.2\%}] &
\textemdash \\

FHIBE Face & SigLIP~2 & PiSSA & 128 &
\ensuremath{14.74\!\to\!11.56} &
\ensuremath{8.78\!\to\!6.81} &
\ensuremath{7.63\!\to\!6.16}\;[\ensuremath{19.2\%}] &
\textemdash \\

\addlinespace[2pt]

Diabetes FT-T & FT-T & DoRA & 128 &
\ensuremath{0.591\!\to\!0.615} &
\ensuremath{0.633\!\to\!0.639} &
\ensuremath{0.057\!\to\!0.037}\;[\ensuremath{35.1\%}] &
\ensuremath{0.151\!\to\!0.135}\;[\ensuremath{10.3\%}] \\

Diabetes FT-T & FT-T & LoRA & 256 &
\ensuremath{0.610\!\to\!0.621} &
\ensuremath{0.633\!\to\!0.640} &
\ensuremath{0.043\!\to\!0.036}\;[\ensuremath{17.6\%}] &
\ensuremath{0.127\!\to\!0.159}\;[\ensuremath{-25.1\%}] \\

Diabetes FT-T & FT-T & PiSSA & 128 &
\ensuremath{0.565\!\to\!0.597} &
\ensuremath{0.629\!\to\!0.638} &
\ensuremath{0.075\!\to\!0.058}\;[\ensuremath{23.0\%}] &
\ensuremath{0.198\!\to\!0.161}\;[\ensuremath{18.8\%}] \\

\end{longtable}
\endgroup

%% file: tables/mit_audit_cells.tex
\begin{longtable}{lllrrrr}
\caption{Post-training audit, per cell. Each entry is the change from the plain arm to trace matching, in percent, so negative favours trace matching; bold marks a reduction. Trace gap, Frobenius discrepancy and operator norm are normalised by $s_{ab}$ and averaged over subgroup pairs and held-out batches; the loss gap is the best--worst subgroup gap in the training loss.}\label{tab:mit-audit}\\
\toprule
Task & Encoder & Adapter & trace gap & Frobenius & operator norm & loss gap \\
\midrule
\endfirsthead
\toprule
Task & Encoder & Adapter & trace gap & Frobenius & operator norm & loss gap \\
\midrule
\endhead
\bottomrule
\endfoot
Fitzpatrick17k & ConvNeXt & DoRA & \textbf{-25} & +98 & +1 & \textbf{-49} \\
Fitzpatrick17k & ConvNeXt & LoRA & \textbf{-31} & \textbf{-84} & \textbf{-4} & \textbf{-40} \\
Fitzpatrick17k & ConvNeXt & PiSSA & +20 & \textbf{-41} & +10 & +19 \\
Fitzpatrick17k & SigLIP~2 & DoRA & +4 & \textbf{-18} & \textbf{-3} & \textbf{-33} \\
Fitzpatrick17k & SigLIP~2 & LoRA & +20 & +66 & \textbf{-4} & \textbf{-16} \\
Fitzpatrick17k & SigLIP~2 & PiSSA & \textbf{-17} & \textbf{-25} & +19 & \textbf{-32} \\
Fitzpatrick17k & ViT & DoRA & \textbf{-27} & +36 & \textbf{-14} & \textbf{-51} \\
Fitzpatrick17k & ViT & LoRA & \textbf{-9} & +102 & \textbf{-18} & \textbf{-33} \\
Fitzpatrick17k & ViT & PiSSA & +3 & +75 & \textbf{-2} & \textbf{-51} \\
UTKFace & ConvNeXt & DoRA & \textbf{-9} & \textbf{-25} & \textbf{-8} & \textbf{-4} \\
UTKFace & ConvNeXt & LoRA & \textbf{-12} & \textbf{-48} & \textbf{-6} & \textbf{-22} \\
UTKFace & ConvNeXt & PiSSA & \textbf{-24} & \textbf{-41} & \textbf{-14} & \textbf{-28} \\
UTKFace & SigLIP~2 & DoRA & \textbf{-20} & +204 & \textbf{-12} & \textbf{-53} \\
UTKFace & SigLIP~2 & LoRA & \textbf{-9} & +103 & \textbf{-7} & \textbf{-37} \\
UTKFace & SigLIP~2 & PiSSA & \textbf{-20} & \textbf{-4} & \textbf{-23} & \textbf{-24} \\
UTKFace & ViT & DoRA & +3 & \textbf{-20} & +2 & +8 \\
UTKFace & ViT & LoRA & \textbf{-20} & \textbf{-8} & \textbf{-19} & \textbf{-3} \\
UTKFace & ViT & PiSSA & +3 & +26 & +7 & \textbf{-20} \\
FHIBE face & ConvNeXt & DoRA & \textbf{-13} & +17 & \textbf{-27} & \textbf{-13} \\
FHIBE face & ConvNeXt & LoRA & \textbf{-14} & +7 & \textbf{-14} & \textbf{-31} \\
FHIBE face & ConvNeXt & PiSSA & +0 & \textbf{-14} & \textbf{-21} & \textbf{-8} \\
FHIBE face & SigLIP~2 & DoRA & +42 & \textbf{-38} & \textbf{-2} & \textbf{-63} \\
FHIBE face & SigLIP~2 & LoRA & +16 & +26 & +23 & +8 \\
FHIBE face & SigLIP~2 & PiSSA & \textbf{-12} & \textbf{-25} & \textbf{-55} & \textbf{-41} \\
FHIBE face & ViT & DoRA & \textbf{-21} & \textbf{-10} & \textbf{-20} & \textbf{-30} \\
FHIBE face & ViT & LoRA & +17 & +19 & +11 & \textbf{-16} \\
FHIBE face & ViT & PiSSA & \textbf{-8} & +3 & \textbf{-23} & \textbf{-15} \\
Diabetes FT-T & FT-T & DoRA & \textbf{-38} & \textbf{-5} & \textbf{-26} & \textbf{-18} \\
Diabetes FT-T & FT-T & LoRA & \textbf{-46} & \textbf{-5} & \textbf{-28} & \textbf{-20} \\
Diabetes FT-T & FT-T & PiSSA & \textbf{-33} & +40 & \textbf{-31} & \textbf{-16} \\
\end{longtable}

%% file: tables/mit_cv_cells.tex
\begin{longtable}{lllrrr}
\caption{Expected-Fisher trace coefficient of variation across the constrained subgroups, per cell, recomputed from the stored checkpoints on the test split by \texttt{mitigation/mechanism\_cv.py}. $\Delta$CV is the change from the plain arm to trace matching in percent, so negative means trace matching made the subgroup traces more even; bold marks a reduction. These are the per-cell values behind Figure~\ref{fig:mit-mechanism}.}\label{tab:mit-cv}\\
\toprule
Task & Encoder & Adapter & CV plain & CV matched & $\Delta$CV \\
\midrule
\endfirsthead
\toprule
Task & Encoder & Adapter & CV plain & CV matched & $\Delta$CV \\
\midrule
\endhead
\bottomrule
\endfoot
Fitzpatrick17k & ConvNeXt & DoRA & 0.095 & 0.154 & +62 \\
Fitzpatrick17k & ConvNeXt & LoRA & 0.206 & 0.142 & \textbf{-31} \\
Fitzpatrick17k & ConvNeXt & PiSSA & 0.161 & 0.327 & +104 \\
Fitzpatrick17k & SigLIP~2 & DoRA & 0.164 & 0.123 & \textbf{-25} \\
Fitzpatrick17k & SigLIP~2 & LoRA & 0.165 & 0.095 & \textbf{-43} \\
Fitzpatrick17k & SigLIP~2 & PiSSA & 0.139 & 0.088 & \textbf{-37} \\
Fitzpatrick17k & ViT & DoRA & 0.315 & 0.217 & \textbf{-31} \\
Fitzpatrick17k & ViT & LoRA & 0.414 & 0.189 & \textbf{-54} \\
Fitzpatrick17k & ViT & PiSSA & 0.358 & 0.125 & \textbf{-65} \\
UTKFace & ConvNeXt & DoRA & 0.157 & 0.123 & \textbf{-22} \\
UTKFace & ConvNeXt & LoRA & 0.213 & 0.133 & \textbf{-37} \\
UTKFace & ConvNeXt & PiSSA & 0.197 & 0.107 & \textbf{-46} \\
UTKFace & SigLIP~2 & DoRA & 0.232 & 0.156 & \textbf{-33} \\
UTKFace & SigLIP~2 & LoRA & 0.189 & 0.091 & \textbf{-52} \\
UTKFace & SigLIP~2 & PiSSA & 0.352 & 0.125 & \textbf{-64} \\
UTKFace & ViT & DoRA & 0.138 & 0.098 & \textbf{-29} \\
UTKFace & ViT & LoRA & 0.110 & 0.088 & \textbf{-20} \\
UTKFace & ViT & PiSSA & 0.168 & 0.134 & \textbf{-20} \\
FHIBE face & ConvNeXt & DoRA & 0.224 & 0.222 & \textbf{-1} \\
FHIBE face & ConvNeXt & LoRA & 0.258 & 0.248 & \textbf{-4} \\
FHIBE face & ConvNeXt & PiSSA & 0.247 & 0.155 & \textbf{-37} \\
FHIBE face & SigLIP~2 & DoRA & 0.115 & 0.342 & +198 \\
FHIBE face & SigLIP~2 & LoRA & 0.278 & 0.192 & \textbf{-31} \\
FHIBE face & SigLIP~2 & PiSSA & 0.740 & 0.181 & \textbf{-75} \\
FHIBE face & ViT & DoRA & 0.208 & 0.174 & \textbf{-16} \\
FHIBE face & ViT & LoRA & 0.082 & 0.153 & +87 \\
FHIBE face & ViT & PiSSA & 0.155 & 0.161 & +4 \\
Diabetes FT-T & FT-T & DoRA & 0.107 & 0.097 & \textbf{-9} \\
Diabetes FT-T & FT-T & LoRA & 0.107 & 0.051 & \textbf{-52} \\
Diabetes FT-T & FT-T & PiSSA & 0.102 & 0.101 & \textbf{-1} \\
\end{longtable}

%% file: app_exact.tex
\section{Exact verification on a small adapted network}
\label{app:exact}

The deep-model experiments cannot form Hessians, so they cannot check the
theorems' quantities directly. Here we do so on a network small enough that every
matrix is formed exactly, to show that the results are operational rather than
only formal. The script \texttt{validate\_exact.py} regenerates every number.

\paragraph{Setup.} Inputs $x\in\R^4$ pass through $\tanh(Wx)$ with a rank-2 LoRA
weight $W=W_0+BA$ ($W_0\in\R^{6\times4}$ frozen), then a frozen linear head with
three classes and softmax cross-entropy; the trainable parameters are
$\phi=(A,B)$, $p=20$. Two subgroups of 400 inputs each have shifted input means,
and labels come from a fixed teacher, so the model is misspecified. The
evaluation point is reached by 300 steps of full-batch gradient descent on the
pooled loss from a full-rank initialization. All computations are in float64 with
exact Hessians, model-expected reachable Fishers and adapter Jacobians. The data
are synthetic and generated by the script.

\begin{table}[!h]
\caption{Exact checks on the small network (\texttt{validate\_exact.py}). $S$ is
the part of $\mathcal{H}_\phi$ orthogonal to $g_a-g_b$, on which the slope term
vanishes for every update (the one-step constraint), and $S_+$ the span of the
positive eigenvectors of $\Dab$ restricted to $S$.}
\label{tab:exact}
\centering
\small
\begin{tabular}{@{}p{3.6cm}p{9.6cm}@{}}
\toprule
Check & Result \\
\midrule
Gauge (Prop.~\ref{prop:gauge}) & $\dim\ker J=4=r^2$; $\|\Fr_cv\|\le10^{-15}$ for $v\in\ker J$; the unrestricted $\lambda_{\min}(\Fr_a-\Fr_b)=-4.45<0$. \\
Expansion (Prop.~\ref{prop:expansion}) & over 8 random horizontal directions, the remainder after the slope and curvature terms scales as $\varepsilon^{3.05}$ (median log--log slope). \\
Worst case (Thm.~\ref{thm:worstcase}) & along the dominant restricted eigenvector (eigenvalue $-3.13$, so $\|\Dab\|^{\mathcal{H}}_2=3.13$), the second-order change matches $\mathcal{Q}_{ab}(\varepsilon)=\tfrac12\|\Dab\|^{\mathcal{H}}_2\varepsilon^2$ in magnitude to within 34.6\%, 5.8\%, 2.2\%, 1.2\% and 0.6\% at $\varepsilon=0.2,0.1,0.05,0.025,0.0125$. \\
Sign on $S$ & $\Dab$ restricted to $S$ is indefinite (eigenvalues in $[-2.01,1.27]$); the sign of the actual gap change matched $\operatorname{sign}(\delta^\top\Dab\delta)$ in 216 of 216 random directions at $\varepsilon=0.01$. \\
Forced increase on $S_+$ (Thm.~\ref{thm:impossible}) & $\dim S_+=6$, $\rho_{\min}=0.0105$; all 1,200 sampled steps increased the gap, by at least $2.79\times\tfrac14\rho_{\min}\|\delta\|^2$. \\
Certificate (Cor.~\ref{cor:weyl}) & on $S_+$, $\lambda_{\min}(\Fr_a-\Fr_b)=-0.100$ against a defect of $0.171$: the reachable-Fisher certificate fails although $\Dab$ is positive definite there; on $\mathcal{H}_\phi$, $\eta^{\mathcal{H}}_{ab}/\|\Fr_a-\Fr_b\|^{\mathcal{H}}_2=0.66$. \\
Defect control & with logits linear in the trainable parameters, $\max|H_c-\Fr_c|=2\times10^{-16}$. \\
Audit identities & Gram-identity Frobenius, trace, and the operator norm from $K^{1/2}DK^{1/2}$ match the dense matrices to relative error $\le3\times10^{-15}$. \\
Balanced representative & a random $\mathrm{GL}(2)$ gauge change leaves predictions unchanged ($2\times10^{-15}$) but moves the trace gap from $-9.38$ to $-8.16$ and the Frobenius discrepancy from $11.42$ to $12.85$; after balancing, both copies give $-10.50$ and $12.34$. \\
\bottomrule
\end{tabular}
\end{table}

\paragraph{Reading.} Table~\ref{tab:exact} shows the local theory working as
stated on an actual nonlinear adapted network. The second-order expansion has a
cubic remainder, Theorem~\ref{thm:worstcase}'s worst-case term is attained, and
once the slope term is removed the sign of the curvature term predicts the sign
of the gap change. On a subspace where $\Dab$ is positive definite, every sampled
step increases the gap, as Theorem~\ref{thm:impossible} requires (the theorem is
stated for slope matching on all of $\mathcal{H}_\phi$; here the slope term
vanishes on the subspace itself, and the same proof applies). Two results
are cautionary and consistent with the paper's qualifications. First, the
reachable-Fisher certificate of Corollary~\ref{cor:weyl} fails on the very subspace where $\Dab$ is definite, and not marginally: there $\lambda_{\min}(\Fr_a-\Fr_b)=-0.100$ is negative, so no bound on the defect could restore the certificate, and the definiteness of $\Dab$ on $S_+$ is carried by the defect itself. On $\mathcal{H}_\phi$ the defect is of comparable size to the Fisher difference ($\eta^{\mathcal{H}}_{ab}$ is $0.66$ of it). The defect term in the guarantees is not a formality.
Second, trace and Frobenius statistics change
along a function-preserving gauge orbit and are made invariant by the balanced
representative used in the audit (Appendix~\ref{sec:frobenius}). This is one
small network and one evaluation point; it shows the statements are checkable and
behave as proved, not how large the defect is in the deep models.